\documentclass[aos,preprint]{imsart}

\RequirePackage{amsthm,amsmath,amsfonts,amssymb}
\RequirePackage[authoryear]{natbib} % uncomment this for author-year citations
\RequirePackage[colorlinks,allcolors=blue]{hyperref}
\RequirePackage{graphicx}

\usepackage{mathrsfs}
\usepackage{dsfont}
\usepackage{booktabs}
\usepackage{float}
\usepackage{tabularx}
\usepackage{nicefrac}
\usepackage{microtype}
\usepackage{color}
\usepackage{xcolor}
\usepackage{url}
\usepackage{bbding}
\usepackage{comment}
\startlocaldefs
\numberwithin{equation}{section}
\DeclareMathAlphabet{\mathsfit}{\encodingdefault}{\sfdefault}{m}{sl}
\SetMathAlphabet{\mathsfit}{bold}{\encodingdefault}{\sfdefault}{bx}{n}

\def\sR{{\mathbb{R}}}

\newcommand{\E}{\mathbb{E}}

\newtheorem{theorem}{Theorem}[section]
\newtheorem{corollary}[theorem]{Corollary}
\newtheorem{lemma}[theorem]{Lemma}
\newtheorem{remark}{Remark}

\newtheorem{assumption}{Assumption}

\newtheorem{proposition}[theorem]{Proposition}

\endlocaldefs

\begin{document}

\begin{frontmatter}
\title{Gaussian-Channel Regularity and Distribution Learning with Flow Matching}
\runtitle{Distribution Learning with Flow Matching}
%\thankstext{T1}{A sample additional note to the title.}

\begin{aug}
%%%%%%%%%%%%%%%%%%%%%%%%%%%%%%%%%%%%%%%%%%%%%%%
%% Only one address is permitted per author. %%
%% Only division, organization and e-mail is %%
%% included in the address.                  %%
%% Additional information can be included in %%
\author[A]{\fnms{Yuan}~\snm{Gao}\ead[label=e1]{yuangao@nankai.edu.cn}},
\author[B]{\fnms{Jian}~\snm{Huang}\ead[label=e2]{j.huang@polyu.edu.hk}},
%\orcid{0000-0000-0000-0000}}
%\and
\author[C]{\fnms{Yuling}~\snm{Jiao}\ead[label=e3]{yulingjiaomath@whu.edu.cn}}
\and
\author[D]{\fnms{Shurong}~\snm{Zheng}\ead[label=e4]{zhengsr@nenu.edu.cn}}
%%%%%%%%%%%%%%%%%%%%%%%%%%%%%%%%%%%%%%%%%%%%%%
%% Addresses                                %%
%%%%%%%%%%%%%%%%%%%%%%%%%%%%%%%%%%%%%%%%%%%%%%
\address[A]{LPMC \& KLMDASR, School of Statistics and Data Science, Nankai University, Tianjin, China\printead[presep={,\ }]{e1}}

\address[B]{Departments of Data Science and Artificial Intelligence, and Applied Mathematics, The Hong Kong Polytechnic
University, Hong Kong, China\printead[presep={,\ }]{e2}}

\address[C]{School of Artificial Intelligence, Wuhan
University, Wuhan, China\printead[presep={,\ }]{e3}}

\address[D]{School of Mathematics and Statistics, Northeast Normal University,
Changchun, China\printead[presep={,\ }]{e4}}
\end{aug}

\begin{abstract}
We study distribution learning by flow matching with linear interpolation
from a Gaussian source under Gaussian-channel regularity, a condition on
the target distribution. For an integer $r\geq2$, we establish spatial
regularity of order $r$ for the population velocity, together with temporal
Lipschitz continuity on the full interval and local time analyticity
before the terminal time. We use these properties to construct a finite
multiscale ReLU sieve with a spatial Lipschitz bound independent of the
truncation radius. For an exact minimizer of the empirical flow matching
risk over this sieve, we obtain an upper bound on the expected
Wasserstein-$2$ error of order $n^{-r/(d+2r)}$, up to a logarithmic
factor, where $n$ is the sample size and $d$ is the dimension.
The bound includes Euler discretization up to time one, without early
stopping. The conditions hold at $r=2$ for targets with strongly convex
$C^2$ potentials and two-sided Hessian bounds. They also hold at every
fixed order $r\geq2$ for distributions obtained by convolving a law of
bounded support with a nondegenerate isotropic Gaussian.
For sufficiently large fixed constants defining the target class, we
prove a minimax lower bound of order
$\max\{n^{-1/2},n^{-(r+1)/(d+2r)}\}$.
\end{abstract}



\end{frontmatter}

\section{Introduction}
Generative learning seeks to estimate a probability distribution from
independent observations and to generate new samples from the learned
distribution. Continuous normalizing flows (CNFs) model the target
distribution by transporting a tractable source distribution, typically
Gaussian, along the trajectories of a time-dependent ordinary
differential equation (ODE). Neural ODEs parameterize the velocity field
with a neural network and allow likelihood-based training
\citep{chen2018neural}. In this formulation, training requires repeated
numerical integration to evaluate the likelihood and its derivatives.

Flow matching provides an alternative approach to training. It specifies
a path of probability distributions from source to target and learns
the corresponding velocity field
through a regression objective that can be evaluated without solving
the learned ODE \citep{lipman2023flow}. Related approaches include
rectified flows and stochastic interpolants
\citep{liu2023flow,albergo2023building}. When source and target samples
are independent and the interpolation is linear, the regression response
is the difference between the target and source samples. The predictors
are the sampled time and the interpolated variable. The conditional
mean of the response defines a population velocity whose ODE, when
well posed, reproduces the marginal distributions along the interpolation.
The resulting deterministic trajectories need
not coincide with the samplewise interpolation paths.
Stable Diffusion 3 uses the rectified flow formulation of flow matching
for text-to-image synthesis \citep{esser2024scaling}. RFdiffusion2 uses
flow matching to generate protein scaffolds conditioned on atomic-level
descriptions of enzyme active sites \citep{ahern2026atom}.
These applications motivate quantitative analysis of the error of the
generated distribution in terms of sample size and properties of the
target distribution.

Once the velocity has been estimated, new samples are obtained by
drawing from the source distribution and numerically integrating the
learned ODE. The distributional error therefore depends on both velocity
estimation and numerical integration. Gaussian interpolation flows
provide a framework for studying
the population velocity and its transport properties
\citep{gao2024gaussian}. Probability-flow ODEs associated with
score-based diffusion models similarly generate samples by integrating
a vector field determined by an estimated score
\citep{song2021scorebased, song2021denoising}.

For a Gaussian source and linear interpolation, the velocity is
determined by conditional means in Gaussian channels. Its regularity
links assumptions on the target distribution to both the
difficulty of the regression problem and the stability of the learned
flow. We derive these regularity properties and study
how approximation, finite-sample estimation, and numerical integration
jointly determine the error of the generated distribution. Our aim is
to bound this error uniformly over a specified target class when the
velocity is estimated by a prescribed regression procedure from finitely
many target observations. Such a guarantee depends on the target class,
the empirical risk, and the metric used to measure distributional error.
Even when two procedures have the same population velocity, different
empirical versions of the loss can lead to different statistical behavior.

We study empirical flow matching with one independent Gaussian source
sample and one independent time sample for each target observation.
We estimate the velocity by minimizing the resulting regression loss
and generate samples by numerically integrating the learned ODE. Under
conditions on posterior moments and smoothed log densities in Gaussian
channels, we bound the expected Wasserstein-2 ($W_2$) distance between the generated
and target distributions.
Recent statistical analyses include almost-minimax results
\citep{fukumizu2025flow} and guarantees on unbounded domains based on
Lipschitz control \citep{kunkel2026lipschitz} under various conditions.
We compare their target classes, losses, and
treatment of the terminal time with ours in Section
\ref{sec:related-work}.

The form of the population velocity near the terminal time motivates
our conditions at small noise scales.
For a target density $p$, let $q_h=p*\varphi_{h^2}$, where $\varphi_{h^2}$ is the density of
$N(0,h^2I_d)$. For $0<t<1$ and $h=(1-t)/t$, the velocity is
\[
 v^*(t,x)=\frac{x}{t}+\frac{h}{t}\nabla\log q_h(x/t).
\]
A derivative of $\log q_h$ need not remain bounded as $h$ tends to zero:
the factor $h$ in the velocity formula can offset its growth.
This observation motivates scale-weighted Gaussian-channel regularity, denoted
GCR$(r)$. It combines a low-precision posterior covariance bound,
linear score growth, and bounds on $hD^m\log q_h$ for
$2\leq m\leq r+1$, where $r\geq2$ is an integer.

\smallskip
\noindent
Our main contributions are as follows.
\begin{itemize}
\item \textit{Velocity regularity from conditions on the target distribution.}
We prove that GCR implies uniform spatial regularity of order $r$,
temporal Lipschitz continuity on the full interval, and quantitative
local time analyticity before the terminal time
(Theorem \ref{thm:gcr-transfer}). The scale-weighted conditions allow
higher-order spatial derivatives of $\log(p*\varphi_{h^2})$ to grow
as the Gaussian noise scale $h$ decreases to zero.
At $r=2$, the conditions hold for targets with strongly convex $C^2$
potentials and two-sided Hessian bounds. They also hold at every fixed
order for distributions obtained by convolving a law of bounded support
with a nondegenerate isotropic Gaussian. This family includes multimodal
targets.

\item \textit{Multiscale approximation with Lipschitz control.}
We construct a finite ReLU sieve that exploits spatial smoothness and
local time analyticity while preserving a spatial Lipschitz constant
independent of the truncation radius (Theorem \ref{thm:gcr-sieve}).
Uniform approximation alone does not control the Lipschitz constant
of a learned network. We therefore impose bounds on the functions
represented by the networks and construct an approximant that satisfies
these bounds. This control is important because the spatial Lipschitz
constant enters exponentially in ODE stability bounds.
Local time analyticity reduces the complexity of time approximation
from polynomial to polylogarithmic in the inverse accuracy.

\item \textit{Distribution learning with a discretized flow.}
Write $\mathcal Q_r(\mathbf H)$ for a GCR$(r)$ class with fixed
constants $\mathbf H$, and $\widehat p_n$ for the generated law.
An oracle inequality for the flow matching loss, combined
with ODE stability and Euler discretization, yields
\[
 \sup_{p\in\mathcal Q_r(\mathbf H)}
 \E_p W_2(\widehat p_n,p)
 \leq C n^{-r/(d+2r)}(\log n)^{r(d+9)/(2(d+2r))}
\]
and a high-probability counterpart for the same estimator
(Theorem \ref{thm:gcr-rate}). At $r=2$, the polynomial exponent also
appears in the Wasserstein-1 ($W_1$) extension of
\citet[Remark 3.8(3)]{kunkel2026minimax} on unbounded domains. On the common target subclass
studied in Section \ref{sec:related-work}, our guarantee holds in $W_2$
with one independent Gaussian source sample and one independent time
sample per target observation. The estimator uses networks of polynomial
size and Euler integration up to time one. Using only Lipschitz
continuity in time would instead yield the exponent $r/(d+3r)$ in the
same approximation--estimation trade-off.
\end{itemize}

The result applies to the exact empirical risk minimizer over a finite neural
sieve. Each candidate network has polynomially many parameters, and
the logarithm of the sieve's cardinality controls the statistical
complexity.

Section \ref{sec:prelim} defines the model and GCR.
Section \ref{sec:gcr-main-rate} states the complete learning guarantee
before the detailed regularity analysis. Section
\ref{sec:gcr-main-regularity} derives velocity regularity from GCR
and gives examples. Section \ref{sec:gcr-main-proof} presents the
approximation, statistical, and numerical arguments used to prove the
learning guarantee. Section \ref{sec:gcr-main-limits} gives a lower bound
and discusses the remaining gap in statistical rates for the generated
distribution. For comparison, the supplement presents an unweighted
sufficient condition, called uniform quadratic-tilt regularity (UQTR),
together with the stronger density
regularity that it implies.

\section{Model and Gaussian-channel regularity}\label{sec:prelim}

Write $\gamma_d=N(0,I_d)$ and let $\varphi_{h^2}$ denote the density of
$N(0,h^2I_d)$. Tensor derivatives use the multilinear operator norm,
and vector norms are Euclidean unless specified otherwise.
For two probability measures with finite second moments,
\[
 W_2(\mu,\nu)=\inf_{\pi\in\Pi(\mu,\nu)}
       \left(\int\|x-y\|_2^2\,d\pi(x,y)\right)^{1/2}.
\]
Here $\Pi(\mu,\nu)$ is the set of couplings with marginals $\mu$ and $\nu$.
We use the same notation for a density and its associated probability measure.
The notation $F_\#\mu$ denotes a push-forward.
Constants may depend on dimension, the fixed regularity order, and the
constants defining the target class, but not on the sample size or tuning scales.

\begin{assumption}[Target density]\label{assump:well-defined}
The target is an absolutely continuous probability measure on $\sR^d$
with density $p$ and mean $\mu_p=\E_pY$ satisfying
$\|\mu_p\|_2\leq M_\mu$, for a fixed finite constant $M_\mu$.
\end{assumption}

\subsection{The interpolation and observations}\label{sec:simu-free-cnf}
Let $Y\sim p$ and $Z\sim\gamma_d$ be independent. Define
\begin{equation}\label{eq:gcr-main-interpolation}
 X_t=(1-t)Z+tY,\qquad
 v^*(t,x):=\E[Y-Z\mid X_t=x],\qquad 0<t<1,
\end{equation}
and write $p_t$ for the density of $X_t$.
Write $\Phi_{s,t}$ for the associated ODE flow map from time $s$
to time $t$. In particular,
\begin{equation}\label{eq:gcr-main-ode}
 \partial_t\Phi_{0,t}(x)=v^*(t,\Phi_{0,t}(x)),\qquad
 \Phi_{0,0}(x)=x.
\end{equation}
Theorem \ref{thm:gcr-transfer} shows that the push-forward of $\gamma_d$
under $\Phi_{0,t}$ has density $p_t$ and that the transported law at
time one is the target distribution.
We reserve $X_t$ for the interpolation and write
$\bar X_t=\Phi_{0,t}(Z)$ for the true ODE trajectory. Their one-time
marginals agree, but their sample paths need not.
The dependence of $v^*$ on the fixed target density $p$ is suppressed.
To distinguish different targets, $v^*(t,x;q)$ denotes
the same conditional mean construction with target density $q$.

For observed independent $Y_1,\ldots,Y_n$, draw independently
$Z_i\sim\gamma_d$ and $T_i\sim\operatorname{Unif}[0,1]$, once for each $i$.
The regression pairs are
\[
 W_i=(T_i,(1-T_i)Z_i+T_iY_i),\qquad U_i=Y_i-Z_i.
\]
For the corresponding population pair, let $T\sim\operatorname{Unif}[0,1]$
be independent of $(Y,Z)$, and set $W=(T,X_T)$ and $U=Y-Z$.
Write $P_W$ for the law of $W$, so that
\[
 P_W(dt,dx)=p_t(x)\,dt\,dx,\qquad
 \|f\|_{L^2(P_W)}^2
 =\int_0^1\int_{\sR^d}\|f(t,x)\|_2^2p_t(x)\,dx\,dt.
\]
The dependence of $P_W$ on the target density $p$ is suppressed.
The fixed-time norm $\|\cdot\|_{L^2(p_t)}$ is defined with respect to
$p_t(x)\,dx$.
Define the population risk and empirical risk by
\begin{align}
 \mathcal L(v)&=\E\|v(W)-U\|_2^2,\label{eq:popu-risk}\\
 \mathcal L_n(v)&=\frac1n\sum_{i=1}^n\|v(W_i)-U_i\|_2^2.
                  \label{eq:empirical-risk}
\end{align}
Conditional expectation gives
$\mathcal L(v)-\mathcal L(v^*)=\|v-v^*\|_{L^2(P_W)}^2$.
A deterministic finite ReLU sieve $\mathcal F_n$, constructed from
the constants defining the target class before observing data, defines
\begin{equation}\label{eq:vf-ERM}
 \widehat v_n\in\arg\min_{v\in\mathcal F_n}\mathcal L_n(v).
\end{equation}

For a fresh Gaussian input and a positive integer $K_{\mathrm{num}}$ denoting
the number of Euler steps, let the exact learned ODE and its Euler scheme be
\begin{align}
 \dot{\widetilde X}_t&=\widehat v_n(t,\widetilde X_t),
                  &&\widetilde X_0\sim\gamma_d,\quad t\in[0,1],\label{eq:ivp-neural}\\
 \widehat X_{j+1}&=\widehat X_j+
     K_{\mathrm{num}}^{-1}\widehat v_n(j/K_{\mathrm{num}},\widehat X_j),
                  &&\widehat X_0\sim\gamma_d,\quad 0\leq j<K_{\mathrm{num}}.
                    \label{eq:ivp-discrete}
\end{align}
Their time-one laws are $\widetilde p_1$ and $\widehat p_n$,
respectively. The latter is the learned distribution in the main theorem.

\subsection{Gaussian-channel regularity of the target distribution}\label{subsec:gcr-meaning}
We place regularity assumptions on the target density $p$ and derive
the properties of its velocity needed for learning.
The relevant quantities are suggested by the interpolation itself:
for $0<t<1, \frac{X_t}{t}=Y+hZ, \ \ h=\frac{1-t}{t}$.
Thus conditioning on $X_t=x$ is a Gaussian denoising problem with
noise standard deviation $h$. The early part of the path has large
noise, whereas approaching the target at $t=1$ corresponds to
$h\downarrow0$. We describe these two regimes by posterior tilts
and smoothed log densities, respectively.

For $\lambda\geq0$ and $b\in\sR^d$, define
\begin{align}
 K(\lambda,b)&=\log\int e^{b^\top y-\lambda\|y\|_2^2/2}p(y)\,dy,
                      \notag\\
 \Pi_{\lambda,b}(dy)&=
 e^{b^\top y-\lambda\|y\|_2^2/2-K(\lambda,b)}p(y)\,dy.
 \label{eq:gcr-main-tilt}
\end{align}
For $\lambda>0$, this is the posterior for observing
$Y+\lambda^{-1/2}G=b/\lambda$, where $G$ is an independent
standard Gaussian and $\lambda$ is the noise precision.
For the interpolation, the parameters are
$\lambda_t=t^2/(1-t)^2$ and $b_{t,x}=tx/(1-t)^2$.
The low-precision region $0\leq\lambda\leq1$ therefore describes
$t\leq1/2$. At $\lambda=0$, $\Pi_{0,b}$ is an exponential tilt.
For the small-noise region $t\geq1/2$, use the density
$q_h=p*\varphi_{h^2}$ of $Y+hZ$ and its logarithm $\ell_h=\log q_h$,
with $0<h\leq1$. The following conditions control these quantities
uniformly over observations and noise scales.

\begin{assumption}[Gaussian-channel regularity]\label{assump:gcr}
Fix an integer $r\geq2$. A density $p$ satisfying
Assumption \ref{assump:well-defined} satisfies GCR$(r)$ if
$K(0,b)<\infty$ for all $b$ and there are fixed constants such that
\begin{alignat}{2}
 \|\operatorname{Cov}_{\Pi_{\lambda,b}}(Y)\|_{\mathrm{op}}
 &\leq C_{\mathrm{low}},
 &\qquad &0\leq\lambda\leq1,\quad b\in\sR^d,
 \label{eq:gcr-low}\\[2pt]
 \|\nabla\ell_h(x)\|_2
 &\leq H_1(1+\|x\|_2),
 &\qquad &0<h\leq1,\quad x\in\sR^d,
 \label{eq:gcr-score}\\[2pt]
 h\|D^m\ell_h\|_\infty
 &\leq H_m,
 &\qquad &0<h\leq1,\quad 2\leq m\leq r+1.
 \label{eq:gcr-weighted}
\end{alignat}
\end{assumption}

Let $\mathbf H=(M_\mu,C_{\mathrm{low}},H_1,\ldots,H_{r+1})$
collect the constants defining the class, and write
$\mathcal Q_r(\mathbf H)$ for densities satisfying these conditions
with the prescribed constants $\mathbf H$.
These constants are independent of the sample size, spatial location,
and channel parameters in their respective ranges.
The index $r$ records the spatial regularity transferred to the velocity,
not the ordinary smoothness order of $p$.

\textit{Posterior stability at low precision.}
Condition \eqref{eq:gcr-low} controls the sensitivity of the posterior
mean to the tilt parameter. The identity
\[
 \nabla_b\E_{\Pi_{\lambda,b}}Y
   =\operatorname{Cov}_{\Pi_{\lambda,b}}(Y)
\]
makes this interpretation precise: the mean is uniformly Lipschitz
in $b$. The quantifier over every real $b$ is essential because a bound on
the target covariance alone would not suffice. Lemma
\ref{lm:gcr-low-precision} shows that this condition also provides
the fixed-order centered posterior moments needed for $0\leq t\leq1/2$
and the sub-Gaussian tails needed for statistical estimation.

\textit{Small-noise growth and cancellation.}
Condition \eqref{eq:gcr-score} permits linear spatial growth of the
smoothed score, with a constant uniform as $h\downarrow0$.
The weighted bounds \eqref{eq:gcr-weighted}, by contrast, allow its
derivatives to grow as fast as $h^{-1}$.
This allowance reflects the coefficient in the velocity formula:
the posterior identities proved in Section
\ref{sec:gcr-main-regularity} give, for $1/2\leq t<1$,
\[
 v^*(t,x)=\frac{x}{t}
       +\frac{1-t}{t^2}\nabla\ell_h(x/t),
 \qquad h=\frac{1-t}{t}.
\]
The score contribution to its $k$th spatial derivative is bounded by
\[
 \frac{1-t}{t^{k+2}}\|D^{k+1}\ell_h\|_\infty
 \leq \frac{H_{k+1}}{t^{k+1}},
 \qquad 1\leq k\leq r.
\]
Thus the vanishing factor $1-t$ offsets the permitted growth as the
noise scale decreases. The score bound controls the growth of the
velocity and its continuous extension to the terminal time. Combining
this bound with the second- and third-order derivative bounds through
the heat equation gives temporal regularity.
Table \ref{tab:gcr-assumptions} summarizes the roles of the conditions.

\begin{table}[!ht]
\caption{The roles of the conditions on the target distribution. All bounds are
uniform over the specified target class.}
\label{tab:gcr-assumptions}
\centering
\begin{tabularx}{\linewidth}{@{}p{0.37\linewidth}X@{}}
\toprule
Condition & Consequence used in the proof\\
\midrule
Real-tilt covariance bound, $0\leq\lambda\leq1$
& Centered posterior concentration, velocity derivatives on $[0,1/2]$, and
target and response tails.\\
Linear growth of $\nabla\ell_h$
& Linear velocity growth, the remainder in the posterior mean near $t=1$,
and the continuous endpoint value.\\
Bounds on $hD^2\ell_h$ and $hD^3\ell_h$
& Spatial Lipschitz and first-order time control, together with the
covariance scale needed for complex continuation.\\
Bounds on $hD^m\ell_h$, $m\leq r+1$
& Spatial Taylor accuracy of order $r$, hence the spatial
approximation exponent.\\
\bottomrule
\end{tabularx}
\end{table}

\textit{Scope of the assumption.}
Gaussian convolution makes $q_h$ positive and smooth for every
$h>0$. GCR concerns uniform quantitative bounds as $h\downarrow0$
and $x$ ranges over the unbounded domain. For example, if
$p=N(0,\Sigma)$ with $\Sigma\succ0$, then
$\nabla\ell_h(x)=-(\Sigma+h^2I_d)^{-1}x$. The Hessian of $\ell_h$
is uniformly bounded, and its derivatives of order at least three vanish.
The tilted covariance is $(\Sigma^{-1}+\lambda I_d)^{-1}$, so this
target satisfies GCR at every fixed order.
Section \ref{sec:gcr-main-regularity} verifies broader examples
directly from distributional conditions.
GCR is a sufficient multiscale condition tailored to the Gaussian
interpolation.

The conditions also impose global geometry, not only local smoothness.
Corollary \ref{cor:gcr-geometry} shows that every target is the
push-forward of $\gamma_d$ by a globally bi-Lipschitz map, with
constants uniform over the class. In particular it has full support, a
Poincar\'{e} inequality, and two-sided Gaussian-type density bounds
almost everywhere. Thus compactly supported targets, targets supported
on lower-dimensional sets, and targets with polynomial tails or tails
that decay faster than Gaussian tails are excluded.
Multimodality is permitted, but arbitrary separation of modes
need not preserve the prescribed bounds.

The posterior and score bounds specify the class over which the
approximation and statistical constants are uniform. The estimator uses
the fixed constants $\mathbf H$ without evaluating the unknown posterior
or smoothed score in the training criterion.
The bounded mean requirement permits unknown locations without
preprocessing.
The initial velocity is $v^*(0,x)=\mu_p-x$, whereas the terminal velocity is
still $v^*(1,x)=x$. For the translated density $p_\mu(x)=p(x-\mu)$,
the identity is $v^*(t,x;p_\mu)=\mu+v^*(t,x-t\mu;p)$.

We conclude this subsection by proving the concentration consequence
used in both the regularity transfer and the statistical analysis.
\begin{lemma}[Low-precision concentration]\label{lm:gcr-low-precision}
For a density with $\|\mu_p\|_2\leq M_\mu$ and a moment generating function
finite on $\sR^d$,
condition \eqref{eq:gcr-low} implies, for every $u,b\in\sR^d$ and
$0\leq\lambda\leq1$,
\[
 \log\E_{\Pi_{\lambda,b}}
   e^{u^\top(Y-\E_{\Pi_{\lambda,b}}Y)}
 \leq \tfrac12 C_{\mathrm{low}}\|u\|_2^2.
\]
Consequently $Y-\mu_p$ is coordinatewise sub-Gaussian,
all fixed-order centered posterior moments at low precision are
uniformly bounded, and
$\|\E_{\Pi_{\lambda,b}}Y\|_2\leq C(1+\|b\|_2)$.
\end{lemma}

\begin{proof}
Finiteness of the moment generating function at every real vector
justifies differentiation of $K$ to every fixed order in $b$.
For example, on a compact set of tilts, polynomial factors are dominated
by a finite sum of exponential factors at nearby real tilts.
Since the factor $e^{-\lambda\|y\|^2/2}$ is at most one, the same
dominating function remains valid. Consequently the first two derivatives
of $K$ are the posterior mean and covariance, and its higher derivatives
are the corresponding cumulants,
as in the standard exponential-family identities \citep{wainwright2008graphical}.

First use only \eqref{eq:gcr-low}. Taylor's formula along a real
line segment yields, for $m_{\lambda,b}=\E_{\Pi_{\lambda,b}}Y$,
\begin{equation}\label{eq:gcr-low-mgf}
 \log\E_{\Pi_{\lambda,b}}e^{u^\top(Y-m_{\lambda,b})}
 =K(\lambda,b+u)-K(\lambda,b)-u^\top\nabla_bK(\lambda,b)
 \leq \tfrac12 C_{\mathrm{low}}\|u\|_2^2.
\end{equation}
This holds for every $u,b$ and $0\leq\lambda\leq1$. At $(\lambda,b)=(0,0)$,
the bound gives a sub-Gaussian centered target and
$\E\|Y\|_2^2\leq dC_{\mathrm{low}}+M_\mu^2=:M_2$.
In particular $Y$ has uniform coordinatewise Gaussian tail bounds,
with constants also depending on $M_\mu$.
Chernoff's bound and integration of its tail give uniform centered
posterior moments of every fixed order. The finite moment-cumulant
formula therefore bounds all fixed-order cumulant tensors, with
constants depending on their order and $d$.

It remains to bound the posterior mean. Jensen's inequality gives
\[
 \E e^{-\lambda\|Y\|_2^2/2}
 \geq e^{-\lambda\E\|Y\|_2^2/2}
\geq e^{-M_2/2},\qquad 0\leq\lambda\leq1.
\]
Thus $\|m_{\lambda,0}\|_2\leq
e^{M_2/2}\E\|Y\|_2$ is uniformly bounded.
Integrating the covariance along $b$ gives
$\|m_{\lambda,b}\|_2\leq C+C_{\mathrm{low}}\|b\|_2$.
This proves the asserted mean bound.
\end{proof}

\section{Learning and sampling guarantee}\label{sec:gcr-main-rate}
The data, empirical risk, and Euler sampler are those of Section
\ref{sec:simu-free-cnf}. The finite sieve is specified in Section \ref{sec:gcr-main-proof}.
\begin{theorem}[Neural estimation under scale-weighted regularity]
\label{thm:gcr-rate}
For a nonempty fixed GCR class $\mathcal Q_r(\mathbf H)$ there is a
finite ReLU sieve $\mathcal G_{A,\eta}$ with the
approximation, entropy, and network bounds of
Theorem \ref{thm:gcr-sieve} at spatial order $r$.
Here $\eta$ is a spatial mesh, distinct from the Gaussian noise scale
$h$ in Assumption \ref{assump:gcr}.
Set $a_d=(d+9)/2$. For sufficiently large $n$, choose
\[
 \begin{aligned}
 &\eta_n=\left\{\frac{(\log n)^{a_d}}{n}\right\}^{1/(d+2r)},
 \qquad A_n=K_A\sqrt{\log n},\\
 &\epsilon_n=\eta_n^r,\qquad
 L_n=\log(2+A_n/\epsilon_n),\\
 &K_{\mathrm{num}}=\lceil A_n^2L_n^3\epsilon_n^{-1}\rceil,
 \end{aligned}
\]
where $K_A$ is a sufficiently large fixed constant.
Let $\widehat v_n$ be an exact empirical minimizer of the
flow matching loss over $\mathcal F_n=\mathcal G_{A_n,\eta_n}$, with a fixed
tie rule, and let $\widehat p_n$ be its $K_{\mathrm{num}}$-step law.
Then
\begin{align}
 \sup_{p\in\mathcal Q_r(\mathbf H)}
 \E_p\|\widehat v_n-v^*\|_{L^2(P_W)}^2
 &\leq C n^{-2r/(d+2r)}(\log n)^{r(d+9)/(d+2r)},\notag\\
 \sup_{p\in\mathcal Q_r(\mathbf H)}
 \E_p W_2(\widehat p_n,p)
 &\leq C n^{-r/(d+2r)}(\log n)^{r(d+9)/(2(d+2r))}.
 \label{eq:gcr-rate}
\end{align}
For every fixed $u>0$, uniformly over the class,
\[
 \mathbb P_p\left\{
 W_2(\widehat p_n,p)
 >C_u n^{-r/(d+2r)}(\log n)^{r(d+9)/(2(d+2r))}
 \right\}\leq n^{-u}.
\]
\end{theorem}

The expectation and probability are over the training sample and its
independent Gaussian and time draws. The $K_{\mathrm{num}}$-step law $\widehat p_n$ is conditional on the learned velocity and includes
the randomness of a Gaussian input.

\textit{Statistical rates and Euler discretization.}
At spatial mesh $\eta$, the bias in velocity is of order $\eta^r$,
and the number of spatial coefficients is of order $\eta^{-d}$.
The time structure contributes logarithmic factors rather than an
additional inverse power of $\eta$. Balancing the squared bias
with the full estimation cost, including its logarithmic factor,
gives the exponents in
\eqref{eq:gcr-rate}. For example, order two yields
$n^{-2/(d+4)}$ in expected $W_2$, whereas order three yields
$n^{-3/(d+6)}$, in each case with the displayed logarithmic factor.

Increasing $r$ changes the target class and its prescribed bounds.
The sieve is tuned using $r$ and the fixed constants $\mathbf H$.

With the given choices, the sufficient number of steps is
\[
 K_{\mathrm{num}}
 \asymp n^{r/(d+2r)}
       (\log n)^{4-r(d+9)/(2(d+2r))}.
\]
This follows from $A_n\asymp\sqrt{\log n}$,
$L_n\asymp\log n$, and the specified choice of $\epsilon_n$.
This step count uses the temporal Lipschitz bound $CA_n^2L_n^3$
from Theorem \ref{thm:gcr-sieve}, which grows polylogarithmically
with $n$.

\section{Velocity regularity under Gaussian-channel conditions}\label{sec:gcr-main-regularity}

The channel conditions give both the real regularity needed for ODE
stability and the local complex-time control used to reduce approximation
complexity. The second conclusion requires posterior concentration
and does not follow merely from Lipschitz continuity of the velocity.
\begin{theorem}[Scale-weighted distribution-to-flow transfer]
\label{thm:gcr-transfer}
On any GCR class $\mathcal Q_r(\mathbf H)$ with fixed $\mathbf H$, the velocity of the
independent linear interpolation has continuous endpoint versions
$v^*(0,x)=\mu_p-x$ and $v^*(1,x)=x$ and satisfies
\begin{align}
 \sup_{t,x}\|D_x^k v^*(t,x)\|_{\mathrm{op}}&\leq C,
                   &&1\leq k\leq r,\label{eq:gcr-space}\\
 \|v^*(t,x)\|_2+\|\partial_t v^*(t,x)\|_2&\leq C(1+\|x\|_2),
                   &&0<t<1.\label{eq:gcr-time}
\end{align}
The temporal Lipschitz bound extends to $[0,1]$. Each spatial slice is
$C^r$. In particular, the velocity bounds
\eqref{eq:physical-envelope} hold at order $r$.

Moreover, uniformly at every real tilt,
\begin{equation}\label{eq:gcr-posterior-cov}
 \operatorname{Cov}_{\Pi_{\lambda,b}}(Y)
       \preceq \frac{C}{1+\lambda}I_d.
\end{equation}
For $A\geq1$, $x\in[-A,A]^d$, and $0\leq t<1$, the function
$u\mapsto v^*(u,x)$ has a holomorphic extension to
$|w-t|<c(1-t)/A$, with modulus at most $CA$ throughout this disk. The resulting ODE is globally well posed
and has the interpolation marginals, including the target at time one.
\end{theorem}

\subsection{Posterior cancellation and velocity regularity}
The proof combines Gaussian conditional moment identities with the
logarithmic heat equation. At low precision we differentiate the
posterior mean directly. Near the target we separate its Gaussian
leading term before estimating derivatives. This step is
essential for bounds that remain uniform at the terminal time.

For any density with a moment generating function finite on $\sR^d$, let
$q_h=p*\varphi_{h^2}$ and $\ell_h=\log q_h$. Completing the square
in the posterior gives, for $\lambda=h^{-2}$ and $z=b/\lambda$,
\begin{align}
 \nabla_bK(\lambda,b)&=z+h^2\nabla\ell_h(z),\notag\\
 D_b^2K(\lambda,b)&=h^2I_d+h^4D^2\ell_h(z),\label{eq:gcr-channel-identities}\\
 D_b^mK(\lambda,b)&=h^{2m}D^m\ell_h(z),\qquad m\geq3.\notag
\end{align}
These identities use derivatives of $\log q_h$ for $h>0$,
not derivatives of the unsmoothed target. The first identity is
Tweedie's formula. See \citet{efron2011tweedie} for the Gaussian
posterior mean and variance identities.
For $1/2\leq t<1$, set $h=(1-t)/t$ and $z=x/t$. The posterior
representation $v^*(t,x)=(\E[Y\mid X_t=x]-x)/(1-t)$ and
\eqref{eq:gcr-channel-identities} give
\begin{equation}\label{eq:gcr-velocity}
 v^*(t,x)=\frac{x}{t}+\frac{h}{t}\nabla\ell_h(z).
\end{equation}
This identity explains the scale weight in GCR: a factor $h$ remains
after spatial differentiation. With heat time $\tau=h^2/2$, the
Gaussian-smoothed density $u(\tau,x):=q_{\sqrt{2\tau}}(x)$ solves
the heat equation $\partial_\tau u=\Delta u$ for $\tau>0$, with
initial density $p$. Here $\Delta=\sum_{j=1}^d\partial_{x_j}^2$
is the spatial Laplacian, and the initial condition holds in $L^1$
as $\tau\downarrow0$. Thus $t\uparrow1$ corresponds to
$\tau\downarrow0$, not to long-time heat evolution. Positive heat time
provides smoothness, while the GCR conditions provide the uniform small-noise
bounds needed below. First-order time control uses the logarithmic
heat equation and involves the third spatial derivative of $\ell_h$.

\begin{proof}[Proof of Theorem \ref{thm:gcr-transfer}]
All constants below depend only on $\mathbf H$, $d$,
and $r$. We first prove the real spatial and temporal bounds, then
verify the hypotheses for complex-time continuation and identify the
ODE marginals. For $0\leq t<1$, let $s_t=1-t$,
$\lambda_t=t^2/s_t^2$, $b_{t,x}=tx/s_t^2$,
$m_t=\nabla_bK(\lambda_t,b_{t,x})$, and
$\Sigma_t=\operatorname{Cov}_{\Pi_{\lambda_t,b_{t,x}}}(Y)$.

\textit{Step 1: Bounds for $0\leq t\leq1/2$.}
For $0\leq t\leq1/2$, we have $s_t\geq1/2$ and $\lambda_t\leq1$, so Lemma
\ref{lm:gcr-low-precision} provides uniform centered cumulant bounds
and $\|m_t\|_2\leq C(1+\|x\|_2)$.
The spatial derivatives of $m_t$ are cumulants multiplied by powers
of $t/s_t^2$. For example
\[
 D_xm_t=(t/s_t^2)\operatorname{Cov}_{\Pi_{\lambda_t,b_{t,x}}}(Y),
 \qquad
 D_x^km_t=(t/s_t^2)^kD_b^{k+1}K(\lambda_t,b_{t,x}).
\]
Together with $v^*(t,x)=(m_t-x)/s_t$, this proves the spatial bounds and
linear growth of $v^*(t,x)$ on $[0,1/2]$.
For the time derivative, let
$k_t=\E_{\Pi_{\lambda_t,b_{t,x}}}[(Y-m_t)\|Y-m_t\|_2^2]$.
The conditional moment identity
$\operatorname{Cov}(Y,\|Y\|_2^2)=2\Sigma_tm_t+k_t$, with covariance
taken under this posterior, gives
\[
 \partial_t v^*(t,x)
 =\frac{m_t-x}{s_t^2}
  +\frac{(1+t)\Sigma_tx-2t\Sigma_tm_t-tk_t}{s_t^4}.
\]
All denominators are bounded away from zero, while $\Sigma_t$ and
$k_t$ are uniformly bounded. This proves the time bound on this
interval. Continuity of the posterior integrals at $t=0$, justified
by the target tails in Lemma \ref{lm:gcr-low-precision}, gives
$m_0=\mu_p$ and $v^*(0,x)=\mu_p-x$.

\textit{Step 2: Spatial regularity and Gaussian cancellation.}
For $1/2\leq t<1$, let $h=s_t/t$, $z=x/t$,
$S_h=\nabla\ell_h$, and $H_h=D^2\ell_h$.
Direct differentiation of the posterior formula gives
$D_xv^*=t\Sigma_t/s_t^3-s_t^{-1}I_d$.
Estimating these two terms separately would leave an artificial
$s_t^{-1}$ singularity. Instead, the covariance identity in
\eqref{eq:gcr-channel-identities} yields the exact cancellation
\begin{align*}
 D_xv^*(t,x)
 &=\frac{t}{s_t^3}\{h^2I_d+h^4H_h(z)\}-\frac1{s_t}I_d\\
 &=\frac1tI_d+\frac{h}{t^2}H_h(z).
\end{align*}
For higher derivatives, \eqref{eq:gcr-velocity} gives
\[
 D_x^k v^*(t,x)=ht^{-k-1}D^{k+1}\ell_h(z),\qquad 2\leq k\leq r.
\]
Since $t\geq1/2$, the bounds on $hD^{k+1}\ell_h$ in
\eqref{eq:gcr-weighted} prove \eqref{eq:gcr-space} in this interval.
The score bound in \eqref{eq:gcr-score} gives linear growth and
$v^*(t,x)\longrightarrow x$ locally uniformly as $t\uparrow1$.
The endpoint velocities $\mu_p-x$ and $x$ also satisfy the spatial bounds.
Each interior slice is smooth by Gaussian convolution.

\textit{Step 3: Temporal regularity from the heat equation.}
The heat equation and $\tau=h^2/2$ give
$\partial_hq_h=h\Delta q_h$ by the chain rule, and hence
\[
 \partial_h\ell_h=h\{\Delta\ell_h+\|S_h\|_2^2\}.
\]
Taking a spatial gradient yields
\begin{equation}\label{eq:gcr-heat-score}
 \partial_hS_h
   =h\{\nabla\Delta\ell_h+2H_hS_h\}.
\end{equation}
All derivatives here exist for $h>0$ by Gaussian convolution. No
derivative of the unsmoothed target is taken.
Using $h'=-t^{-2}$ and $(x/t)'=-x/t^2$ gives the exact formula
\begin{align}
 \partial_t v^*(t,x)
 &=-\frac{x}{t^2}-\frac{2-t}{t^3}S_h(x/t)
   -\frac{h}{t^3}H_h(x/t)x\notag\\
 &\quad-\frac{h^2}{t^3}
   \{\nabla\Delta\ell_h(x/t)+2H_h(x/t)S_h(x/t)\}.
 \label{eq:gcr-time-exact}
\end{align}
The contraction $\nabla\Delta\ell_h$ is bounded by
$d\|D^3\ell_h\|_{\mathrm{op}}$.
To see explicitly how GCR controls the last two terms, note that
$0<h\leq1$ and
\begin{align*}
 h\|H_h(z)x\|_2&\leq C\|x\|_2,\\
 h^2\|\nabla\Delta\ell_h(z)\|_2&\leq Ch,\\
 h^2\|H_h(z)S_h(z)\|_2&\leq Ch(1+\|z\|_2).
\end{align*}
Together with $\|S_h(z)\|_2\leq C(1+\|z\|_2)$ and
$\|z\|_2\leq2\|x\|_2$, these bounds prove \eqref{eq:gcr-time}.
In particular, the third-derivative bound is available because $r\geq2$.
Integration in time, using the endpoint limits already established,
gives the Lipschitz extension to $[0,1]$ and the velocity bounds
\eqref{eq:physical-envelope}.
Alternative scale profiles are discussed in supplementary
Remark \ref{rem:gcr-scale-profile}.

\textit{Step 4: Posterior concentration and time analyticity.}
To obtain the complex-time extension, we verify the posterior
concentration and mean bounds required by Lemma \ref{lm:time-analytic}.
Write $m_{\lambda,b}=\E_{\Pi_{\lambda,b}}Y$.
For $\lambda\geq1$, \eqref{eq:gcr-channel-identities} and GCR give
\begin{align}
 \|m_{\lambda,b}-b/\lambda\|_2
 &\leq C\lambda^{-1}(1+\|b/\lambda\|_2),\notag\\
 \|\operatorname{Cov}_{\Pi_{\lambda,b}}(Y)-\lambda^{-1}I_d\|_{\mathrm{op}}
 &\leq C\lambda^{-3/2}.
 \label{eq:gcr-weaker-covariance-remainder}
\end{align}
In particular, the covariance is at most $C\lambda^{-1}I_d$.
Combining with the low-precision bound proves
\eqref{eq:gcr-posterior-cov}. Integrating this covariance bound over
real tilts proves
\[
 \E_{\Pi_{\lambda,b}}e^{u^\top(Y-m_{\lambda,b})}
 \leq \exp\{C\|u\|_2^2/(2(1+\lambda))\}.
\]
Thus the centered posterior $Y-m_t$ has a uniform exponential-square
moment at scale $(1+\lambda_t)^{-1/2}$, and in the terminal region
$m_t=x/t+r_t$ with $\|r_t\|_2\leq C(1-t)^2(1+\|x\|_2)$.

Lemma \ref{lm:time-analytic} then gives the holomorphic extension by
using this concentration to control complex perturbations of the
posterior integrals and keep their normalized denominator away from zero.

\textit{Step 5: Identification of the flow marginals.}
Global spatial Lipschitz continuity, temporal continuity, and
linear growth give a unique global ODE flow. For a smooth compactly
supported test function $\psi$, the independent interpolation satisfies
\[
 \frac{d}{dt}\E\psi((1-t)Z+tY)
 =\E\langle\nabla\psi(X_t),Y-Z\rangle
 =\int\langle\nabla\psi,v^*(t,\cdot)\rangle p_t.
\]
The moments above justify the differentiation. Uniqueness of the
continuity equation for this globally Lipschitz velocity with linear growth
identifies its marginals with the ODE push-forwards, by the classical
characteristic-flow theory reviewed in
\citet[Section 2]{ambrosio2014continuity}. Their continuous
limits at zero and one are the source and target, respectively.
\end{proof}

\begin{corollary}[Global geometry of the target class]
\label{cor:gcr-geometry}
For every fixed GCR class there are constants $L_F\geq1$, $M_F<\infty$
and $c,C>0$ such that $F=\Phi_{0,1}$ satisfies
\[
 F_\#\gamma_d=p,\qquad
 \operatorname{Lip}(F),\operatorname{Lip}(F^{-1})\leq L_F,
 \qquad \|F(0)\|_2\leq M_F.
\]
In particular, almost everywhere on $\sR^d$,
\begin{equation}\label{eq:gcr-gaussian-envelopes}
 c e^{-C\|x\|_2^2}\leq p(x)\leq C e^{-c\|x\|_2^2},
\end{equation}
and $\operatorname{Var}_p(g)\leq L_F^2\int\|\nabla g\|_2^2p$
for $g\in C_c^\infty(\sR^d)$.
\end{corollary}
\begin{proof}
Theorem \ref{thm:gcr-transfer} gives a common spatial Lipschitz bound
$L_v$ and $\sup_t\|v^*(t,0)\|_2\leq M_v$. Forward and backward
Gronwall estimates give $L_F=e^{L_v}$ and $M_F=M_ve^{L_v}$.
The change-of-variables formula for bi-Lipschitz maps gives
$p(x)=\varphi_1(F^{-1}(x))|\det DF^{-1}(x)|$ almost everywhere.
The determinant lies between $L_F^{-d}$ and $L_F^d$, and
\[
 L_F^{-1}\|x-F(0)\|_2\leq\|F^{-1}(x)\|_2
 \leq L_F\|x-F(0)\|_2.
\]
Absorbing the bounded translation into constants proves
\eqref{eq:gcr-gaussian-envelopes}. Applying the Gaussian Poincar\'{e}
inequality \citep{bakry2014analysis} to $g\circ F$, with the almost-everywhere chain rule,
proves the variance bound.
\end{proof}

\subsection{Local analyticity in time}
Temporal Lipschitz continuity is enough for Euler convergence, but it
does not support the statistical exponent in Theorem \ref{thm:gcr-rate}.
The approximation argument also requires a bound on the number of
evaluations at real time points needed to attain a given accuracy.
The following lemma provides local
analyticity with a radius proportional to the distance to the terminal
time.

\begin{lemma}[Quantitative local analyticity in time]\label{lm:time-analytic}
Suppose $\|\E_pY\|_2\leq M_\mu$, $K(0,b)<\infty$ for every real $b$, and, at all
real tilts,
\[
 \operatorname{Cov}_{\Pi_{\lambda,b}}(Y)\preceq C(1+\lambda)^{-1}I_d.
\]
Assume also
$\|\E_{\Pi_{\lambda,b}}Y-b/\lambda\|_2
 \leq C\lambda^{-1}(1+\|b/\lambda\|_2)$ for $\lambda\geq1$.
There exist $c,C>0$ such that, for $A\geq1$, $x\in[-A,A]^d$,
and $0\leq t<1$, the interpolation velocity has a holomorphic extension
to $|w-t|<c(1-t)/A$ with modulus at most $CA$.
In particular,
\begin{equation}\label{eq:analytic-cauchy}
 \|\partial_t^jv^*(t,x)\|_2
 \leq CA\,j!\{CA/(1-t)\}^j,\qquad j\geq0.
\end{equation}
No analyticity at $t=1$ is asserted.
\end{lemma}

Center the real posterior and use its uniform tilted covariance
bound to obtain exponential-square integrability at scale
$(1+\lambda_t)^{-1/2}$. For complex time within $c(1-t)/A$ of $t$,
centering and rescaling the posterior variable make the linear and
quadratic coefficients in the perturbation exponent uniformly small.
Exponential-square integrability then keeps the normalized partition
function within $1/2$ of one and provides the domination needed to show
that the ratio representing the posterior mean is holomorphic and bounded.
Cauchy's formula gives \eqref{eq:analytic-cauchy}.
Appendix \ref{subsec:supp-time-analytic} provides the domination
and nonvanishing estimates.

The radius shrinks at the endpoint, whereas the real first derivative
remains bounded by \eqref{eq:gcr-time}. The approximation uses analyticity only on the
nonterminal cells and real Lipschitz control on the final cell.

\subsection{Sufficient conditions for Gaussian-channel regularity}
\label{subsec:concrete-target-classes}
The complete proofs of the following four propositions are collected
in Appendix \ref{sec:concrete-target-proofs}, in the same order.

\begin{proposition}[Strongly log-concave targets with $C^2$ potentials]
\label{prop:gcr-c2}
Suppose $p\propto e^{-U}$ has $\|\E_pY\|_2\leq M_\mu$,
$U\in C^2(\sR^d)$, and
\[
 \kappa I_d\preceq D^2U(x)\preceq\Lambda I_d,
 \qquad x\in\sR^d,\qquad 0<\kappa\leq\Lambda<\infty.
\]
Then $p\in\mathcal Q_2(\mathbf H)$ with $\mathbf H$ depending only on
$d,\kappa,\Lambda,M_\mu$. In fact,
\begin{equation}\label{eq:gcr-c2-channel}
 \|D^2\ell_h\|_\infty\leq C,\qquad
 \|D^3\ell_h\|_\infty\leq C/h,\qquad 0<h\leq1.
\end{equation}
Consequently, Theorem \ref{thm:gcr-rate} gives
$Cn^{-2/(d+4)}(\log n)^{(d+9)/(d+4)}$ in expected $W_2$ on this class.
\end{proposition}

The Hessian of the tilted potential lies between $(\kappa+\lambda)I_d$
and $(\Lambda+\lambda)I_d$, giving uniform bounds on centered posterior
moments. In the small-noise channel, the absolute moments about the
posterior mean are $O(h^j)$ at each fixed order $j\geq1$.
Writing $B=\nabla U(Y)$ and $A=D^2U(Y)$ gives
$D^2\ell_h=-\E A+\operatorname{Cov}(B)$ under the posterior.
One more derivative acts on the posterior density, yielding
$h^{-2}\{-\operatorname{Cov}(A_{ij},Y_k)
+\operatorname{Cum}(B_i,B_j,Y_k)\}$, of order $h^{-1}$.
Here $\operatorname{Cum}$ denotes the joint cumulant under the posterior.
The full verification is in Appendix \ref{subsec:supp-c2}.

For $0<a<1$, consider the one-dimensional potential
\begin{equation}\label{eq:gcr-strict-example}
 U''(x)=1+a\sin(x^2),\qquad U(0)=U'(0)=0.
\end{equation}
Its density is centered and has
$1-a\leq U''\leq1+a$, whereas $U'''(x)=2ax\cos(x^2)$ is unbounded.
The proposition applies without controlling this third derivative.

\begin{proposition}[Log-density smoothness and posterior concentration]
\label{prop:gcr-ordinary}
Let $r\geq2$ be an integer. Suppose a strictly positive density $p$
satisfies Assumption \ref{assump:well-defined}, has a finite moment
generating function at every real vector, and, with $g=\log p$, satisfies
\begin{equation}\label{eq:gcr-ordinary-smoothness}
 g\in C^r(\sR^d),\qquad
 \|D^jg\|_\infty\leq L_j,\qquad 2\leq j\leq r.
\end{equation}
Suppose also that its tilted posteriors satisfy
\begin{equation}\label{eq:gcr-ordinary-concentration}
 \operatorname{Cov}_{\Pi_{\lambda,b}}(Y)
 \preceq \frac{C_{\mathrm{post}}}{1+\lambda}I_d,
 \qquad \lambda\geq0,\quad b\in\sR^d.
\end{equation}
Then $p$ satisfies GCR$(r)$. More precisely, uniformly for $0<h\leq1$,
\begin{align}
 \|\nabla\ell_h(z)\|_2&\leq C(1+\|z\|_2),\notag\\
 \|D^j\ell_h\|_\infty&\leq C\quad(2\leq j\leq r),&
 h\|D^{r+1}\ell_h\|_\infty&\leq C.
 \label{eq:gcr-ordinary-channel}
\end{align}
The constants in these bounds depend only on $d,r,M_\mu,C_{\mathrm{post}}$ and
$L_2,\ldots,L_r$. Theorem \ref{thm:gcr-rate} therefore gives
the expected $W_2$ rate $\widetilde O(n^{-r/(d+2r)})$ on each such
class with fixed constants.
\end{proposition}

The covariance condition, integrated over real linear tilts, gives
centered posterior moments of order $h^j$ in the channel $Y+hZ=z$.
Derivatives of $\ell_h$ through order $r$ are sums of posterior
cumulants of derivatives of $g$ through that same order. One more
derivative acts on the Gaussian likelihood, rather than on $D^rg$:
it adds a centered posterior variable and a factor $h^{-2}$.
The added variable has size $O(h)$, giving the permitted $O(h^{-1})$
growth. Appendix \ref{subsec:supp-ordinary} proves the statement,
including the score-growth bound.

A concrete family covered by Proposition \ref{prop:gcr-ordinary} is
\begin{equation}\label{eq:gcr-ordinary-family}
 p(y)=\mathcal Z^{-1}e^{-U(y)}a(y),
\end{equation}
where $\mathcal Z$ is the normalizing constant, $U,a\in C^r$,
$\kappa I_d\preceq D^2U\preceq\Lambda I_d$,
$0<\kappa\leq\Lambda$, and $\|\nabla U(0)\|_2\leq M_0$.
Assume $0<a_-\leq a\leq a_+$ and fixed derivative bounds on $a$
through order $r$ and on $U$ at orders $3,\ldots,r$ (the latter
list is empty when $r=2$).
Quadratic comparison for tilts of the base density proportional to
$e^{-U}$, together with the bounded reweighting by $a$, verifies
\eqref{eq:gcr-ordinary-concentration}. Derivatives of
$\log p=-U+\log a-\log\mathcal Z$ verify
\eqref{eq:gcr-ordinary-smoothness}. The same comparison provides
bounds on the mean and the moment generating function. Details are given
in Appendix \ref{subsec:supp-ordinary}.

For example, in dimension one, consider
\begin{equation}\label{eq:gcr-periodic-example}
 p(y)=\mathcal Z^{-1}
 \exp\{-y^2/2-\tfrac14\cos y-\tfrac14\cos(3y)\}.
\end{equation}
This density has the form \eqref{eq:gcr-ordinary-family} with
$U(y)=y^2/2+\tfrac14\cos y$ and
$a(y)=\exp\{-\tfrac14\cos(3y)\}$.
Here $3/4\leq U''\leq5/4$ and
$e^{-1/4}\leq a\leq e^{1/4}$, and the derivatives of $U$ of orders
at least two and all derivatives of $a$ are bounded at each fixed order.
The density is centered by symmetry, while
$(\log p)''(0)=3/2>0$, so it is not log-concave.
Proposition \ref{prop:gcr-ordinary} places this non-Gaussian density
in GCR$(r)$ for every fixed integer $r\geq2$, with class constants
depending on $r$.

\begin{proposition}[Gaussian convolution targets]\label{prop:gcr-gaussian-latent}
Let $Y=Q+\sigma G$, where $G\sim\gamma_d$ is independent of $Q$,
$\|Q\|_2\leq R$ almost surely, and $\sigma>0$.
For every fixed integer $r\geq2$, the density of $Y$ satisfies GCR$(r)$
with $\mathbf H$ depending only on $d,r,R,\sigma$.
\end{proposition}

Completing the square shows that, under each tilt, the target variable
is the sum of a deterministic shift depending on the tilt, a bounded
latent variable, and an independent centered Gaussian variable.
The same representation expresses channel derivatives of order
at least three as cumulants of the tilted latent law, divided
by powers of $\sigma^2+h^2$. Boundedness of the latent variable
and $\sigma>0$ give the required uniform bounds.
See Appendix \ref{subsec:supp-gaussian-latent}.

This includes finite location mixtures with a common nondegenerate
isotropic Gaussian covariance and bounded centers, as well as mixtures
generated by continuous latent distributions with bounded support.

\begin{proposition}[Finite-order Gaussian perturbations]
\label{prop:gcr-finite-perturbation}
Let $r\geq2$ and let $p=\varphi_1(1+g)$ be a centered probability
density, where $g\in C^r(\sR^d)$,
$\|g\|_\infty\leq1/2$, and
$\max_{1\leq j\leq r}\|D^jg\|_\infty\leq G$.
Then $p$ satisfies GCR$(r)$ with constants depending only on $d,r,G$.
\end{proposition}
Completing the square gives
$q_h(x)=\varphi_{1+h^2}(x)B_h(x)$, where
\[
 B_h(x)=1+\E g\left(\frac{x}{1+h^2}
                     +\frac{hZ}{\sqrt{1+h^2}}\right),
 \qquad \frac12\leq B_h\leq\frac32.
\]
Derivatives through order $r$ act on $g$, while one more derivative can
instead act on the Gaussian kernel, costing $C/h$.
The log chain rule gives the required weighted derivative bounds.
The tilted density ratio relative to the corresponding Gaussian
posterior is at most three, which bounds its covariance uniformly.
Appendix \ref{subsec:supp-finite-perturbation} gives the details.

We use this criterion to verify GCR membership for the Gaussian
perturbation families in the minimax lower bound of
Section \ref{sec:gcr-main-limits}.
Translations with bounded shift are also allowed after changing the
constants: the tilted covariance is that of the original law at
tilt $b-\lambda\mu$, and the smoothed score is translated by $\mu$.

\subsection{Multimodality and dependence on class constants}
The Gaussian-convolution example need not be log-concave.
In dimension one, let $Q$ take values $-a$ and $a$ with equal
probabilities and set $v=\sigma^2+h^2$. The spatial derivatives of
$\ell_h$ are explicit:
\begin{align}
 \ell_h'(x)&=-x/v+(a/v)\tanh(ax/v),\notag\\
 \ell_h''(x)&=-1/v+(a^2/v^2)\operatorname{sech}^2(ax/v),\notag\\
 \ell_h'''(x)&=-2(a^3/v^3)\operatorname{sech}^2(ax/v)\tanh(ax/v).
 \label{eq:aos-mixture-channel}
\end{align}
At zero noise, the curvature of the log density at the origin is
$-1/\sigma^2+a^2/\sigma^4$, which is positive when $a>\sigma$.
Nevertheless Proposition \ref{prop:gcr-gaussian-latent} applies
at every fixed order.
The constants in this example depend on $a$ and $\sigma$, and the theorem
does not bound them uniformly as $\sigma\downarrow0$ or $a/\sigma\to\infty$.
Likewise, for the strongly
convex example, the constants depend on both curvature bounds and on
dimension.

The construction is also compatible with independent products.
If the target factors over fixed coordinate blocks, its log partition
function and smoothed log density are sums of the block versions.
Their derivatives are block diagonal, and their score norms satisfy
a linear-growth bound with constants depending on the dimension. Orthogonal
changes of coordinates preserve the isotropic Gaussian channel and
the tensor operator norms. These operations produce further GCR
examples with the same order, after adjusting the numerical
constants.

\section{Approximation, estimation, and flow stability}
\label{sec:gcr-main-proof}
The Wasserstein error is controlled by the sum of the square root of
the velocity regression error and the Euler discretization error. We first construct
a finite sieve with controlled approximation error and spatial Lipschitz
constant, then bound the regression error of its empirical minimizer.
Flow stability and the Euler comparison transfer this bound to the
generated distribution. Appendix \ref{sec:gcr-proofs} gives the detailed
network construction.

\subsection{Multiscale sieve with spatial Lipschitz control}
Theorem \ref{thm:gcr-transfer} gives, for a common constant $H$,
\begin{align}
 \sup_t\|v^*(t,0)\|_2&\leq H,\qquad
 \sup_{\substack{t,x\\1\leq k\leq r}}
        \|D_x^k v^*(t,x)\|_{\mathrm{op}}\leq H,\notag\\
 \|v^*(t,x)-v^*(u,x)\|_2&\leq H(1+\|x\|_2)|t-u|.
 \label{eq:physical-envelope}
\end{align}
ReLU networks use the activation $\varrho(x)=\max\{0,x\}$.
We write $\mathtt S,\mathtt D,\mathtt W$ for their number of nonzero
parameters, hidden depth, and maximum hidden width. All Lipschitz bounds
below are bounds on the realized network, not its parameter vector.
\begin{theorem}[Multiscale finite ReLU sieve]\label{thm:gcr-sieve}
Fix a GCR$(r)$ class, $A\geq1$, and $0<\eta\leq1/2$, and let $\epsilon=\eta^r$, $L=\log(2+A/\epsilon)$. There is a deterministic nonempty finite ReLU sieve $\mathcal G_{A,\eta}$, fixed before observing the data, such that every member satisfies globally
\begin{align}
 \|f\|_\infty&\leq CA,\qquad \operatorname{Lip}_x(f)\leq C,\notag\\
 \operatorname{Lip}_t(f)&\leq C A^2L^3,
 \label{eq:gcr-sieve-bounds}
\end{align}
and, for every target in the fixed class,
\begin{equation}\label{eq:gcr-sieve-approx}
 \min_{f\in\mathcal G_{A,\eta}}
 \sup_{[0,1]\times[-A,A]^d}\|f-v^*\|_2\leq C\epsilon.
\end{equation}
Inputs are clipped to $[0,1]\times[-A,A]^d$. The number $Q_{\rm an}$ of quantized scalar coefficients and the entropy obey
\begin{align}
 Q_{\rm an}&\leq C A^{d+1}\eta^{-d}L^2,\notag\\
 \log|\mathcal G_{A,\eta}|&\leq C Q_{\rm an}L
                    \leq C A^{d+1}\eta^{-d}L^3.
 \label{eq:gcr-sieve-entropy}
\end{align}
One may take network size, depth, and width bounded respectively by
\begin{equation}\label{eq:gcr-sieve-network-size}
 \mathtt S\leq C Q_{\rm an}L^3,\qquad
 \mathtt D\leq C L^2,\qquad
 \mathtt W\leq C Q_{\rm an}L.
\end{equation}
Constants depend only on the fixed class, $d$, and $r$.
\end{theorem}

The spatial construction uses local polynomials on the original cube,
rather than rescaling all derivatives to a unit cube.
Partition-of-unity cancellations keep the spatial Lipschitz constant
independent of $A$.
Time is treated on dyadic bands approaching one: each band has
$O(A)$ cells and uses a polynomial degree $O(L)$.
A bounded filtered time operator and endpoint corrections preserve
spatial stability. On the final interval, of length comparable to
$\epsilon/A$, the temporal Lipschitz bound permits interpolation to
the known endpoint velocity $x\mapsto x$.
These steps use classical polynomial and ReLU multiplication tools
\citep{yarotsky2017error,guhring2020error}. The required quantitative
construction is proved in Appendix \ref{subsec:analytic-approx-proof}.

\subsection{Multiscale approximation in space and time}
\label{subsec:aos-sieve-construction}
We give the construction underlying Theorem \ref{thm:gcr-sieve}.
The full product-network implementation is in Appendix
\ref{subsec:analytic-approx-proof}. It is useful to keep the
spatial and temporal operations separate: a stable operator in
time can act on spatial approximants without differentiating
their coefficients in time.

\subsubsection{Spatial approximation in the original coordinates}
Use a tensor grid of spacing comparable to $\eta$ on $[-A,A]^d$,
with continuous tensor-product hats $w_a$ centered at grid nodes $a$.
The hats have bounded overlap, sum to one, and satisfy
$\|\nabla w_a\|_\infty\leq C/\eta$. For each real time $t$, let
\[
 T_{t,a}(x)=\sum_{|\alpha|\leq r-1}
       \frac{\partial_x^\alpha v^*(t,a)}{\alpha!}(x-a)^\alpha,
 \qquad H_t(x)=\sum_a w_a(x)T_{t,a}(x).
\]
There are $O((A/\eta)^d)$ nodes and a fixed number of coefficients
per node. Taylor's formula on each active cell gives
\[
 \|T_{t,a}-v^*(t,\cdot)\|_\infty\leq C\eta^r,\qquad
 \|D_xT_{t,a}-D_x v^*(t,\cdot)\|_\infty\leq C\eta^{r-1}.
\]
The identities $\sum_a w_a=1$ and $\sum_a\nabla w_a=0$ then give
\begin{align}
 D_xH_t-D_x v^*(t,\cdot)
 &=\sum_a(D_xw_a)(T_{t,a}-v^*(t,\cdot))\notag\\
 &\quad+\sum_a w_a(D_xT_{t,a}-D_x v^*(t,\cdot)).
 \label{eq:aos-spatial-cancellation}
\end{align}
Hence the spatial Lipschitz constant stays bounded independently
of $A$.

In scaled cell coordinates, round each spatial Taylor coefficient
to a multiple of $\epsilon=\eta^r$. Each coefficient has magnitude
at most $CA$, so a fixed alphabet of size at most $CA/\epsilon$
suffices. Rounding changes values by $C\epsilon$ and spatial
derivatives by $C\epsilon/\eta\leq C$.
Use the same rounded vector whenever a time endpoint is repeated.
We continue to denote the rounded spatial approximants by $H_t$, with
their products still evaluated exactly. They satisfy
\[
 \|H_t-v^*(t,\cdot)\|_{\infty,[-A,A]^d}\leq C\epsilon,\quad
 \|H_t\|_\infty\leq CA,\quad \operatorname{Lip}_x(H_t)\leq C.
\]
Spatial clipping extends the output and spatial Lipschitz bounds to all
$x\in\sR^d$.

\subsubsection{Dyadic time partition and filtered polynomial approximation}
Choose $\delta=2^{-J}$ with
$\delta\leq\epsilon/A<2\delta$.
Partition each dyadic band before $1-\delta$ into $O(A)$ cells,
and append the terminal interval $[1-\delta,1]$.
There are $O(AL)$ cells, with minimum length
$\ell_{\min}:=\min_I|I|\geq c\epsilon/A^2$.
The half-length of each nonterminal cell is at
most a fixed fraction of the complex radius in
Lemma \ref{lm:time-analytic}. A time polynomial of degree
$m=O(L)$ therefore has error $CA4^{-m}\leq C\epsilon$ there.

Direct interpolation at an arbitrary collection of nodes would
not automatically preserve the spatial Lipschitz bound.
Instead, we use the discrete filtered polynomial operator
$\mathcal V_m$ of supplementary Lemma \ref{lm:bounded-time-filter}.
It uses $O(m)$ real nodes, reproduces polynomials through degree $m$,
and has operator norm at most three, independently of $m$.
Applying it to the $H_t$ at its nodes produces $T_I$ on a cell
$I=[a,b]$. The operator norm controls both the approximation error
and every spatial difference.

For $u=(t-a)/(b-a)$, correct the endpoint values by setting
\[
 P_I(t,x)=T_I(t,x)+(1-u)\{H_a(x)-T_I(a,x)\}
                    +u\{H_b(x)-T_I(b,x)\}.
\]
The corrected operator has norm at most seven and matches both endpoint
values. The approximation error is at most $C\epsilon$, the output is
bounded by $CA$, and the spatial Lipschitz constant is at most $C$.
Adjacent cells therefore join continuously.
On the terminal interval use a line segment between $H_{1-\delta}$
and the known clipped identity $H_1(x)=x$ on the cube.
The temporal Lipschitz bound controls its error by
$C\epsilon+CA\delta\leq C\epsilon$.

\subsubsection{Sieve entropy, neural realization, and Lipschitz bounds}
The number of spatial coefficient vectors used at the time nodes is
$O(AL^2)$. Consequently
\[
 Q_{\rm an}\leq C(A/\eta)^d AL^2,\qquad
 \log|\mathcal G_{A,\eta}|
 \leq C Q_{\rm an}\log(CA/\epsilon).
\]
The time transformations are fixed linear maps of the original
rounded coefficients. To bound the time derivative, first subtract the local
constant $H_a(x)$ on a cell $I=[a,b]$ of length $\ell_I$.
The temporal Lipschitz bound and the approximation error at the time
nodes give $\|H_t-H_a\|_\infty\leq C(A\ell_I+\epsilon)$ at every
time node in the cell. The corrected time operator reproduces constants and
has norm at most seven, so
\[
 \|P_I-H_a\|_\infty\leq C(A\ell_I+\epsilon),\qquad
 \|\partial_tP_I\|_\infty
 \leq CL^3(A+\epsilon/\ell_I)\leq CA^2L^3.
\]
The derivative bound follows by a Chebyshev expansion of this
degree-$O(L)$ polynomial and $\ell_I\geq c\epsilon/A^2$.
The terminal linear piece has slope at most $C(A+\epsilon/\delta)$.
Continuity of the pieces therefore gives
$\operatorname{Lip}_t(P)\leq CA^2L^3$ globally after clipping.

To realize $P$ by ReLU networks, first write its continuous
piecewise-polynomial dependence on time in a telescoping form.
If cell $i$ has endpoints $a_i,b_i$ and polynomial $P_i(u,x)$, let
$c_i(t)=\operatorname{clip}_{[0,1]}((t-a_i)/(b_i-a_i))$. Then
\[
 P(t,x)=H_0(x)+\sum_i\{P_i(c_i(t),x)-P_i(0,x)\}.
\]
For a time in cell $i$, the increments from earlier cells sum to
$H_{a_i}-H_0$, while those from later cells vanish. This representation uses continuous
clipped coordinates, not discontinuous indicator gates.
It can therefore be implemented with approximate products and
exact ReLU clipping.

There are $O(Q_{\rm an}L)$ product terms, each of degree $O(L)$
and with a fixed number of distinct input coordinates.
The multiplication construction in Appendix \ref{subsec:analytic-approx-proof}
controls both uniform and Lipschitz approximation errors. We choose the
product accuracy so that the full neural sum $F$ satisfies
\[
 \|F-P\|_\infty\leq C\epsilon,\qquad
 \operatorname{Lip}_x(F-P)\leq C\epsilon/\eta,\qquad
 \operatorname{Lip}_t(F-P)\leq C\epsilon/\ell_{\min}.
\]
Since $\epsilon/\ell_{\min}\leq CA^2$, the neural approximant
retains the temporal Lipschitz bound $CA^2L^3$.
The resulting network sizes are those in
\eqref{eq:gcr-sieve-network-size}.

Finally, enumerate all coefficient arrays formed from the finite
alphabet, with coefficients shared at common time endpoints, and retain
the realized networks satisfying \eqref{eq:gcr-sieve-bounds}, with
sufficiently large fixed constants.
For every target in the class, the approximant constructed above belongs
to this sieve.

\subsubsection{Comparison with first-order time approximation}
If only temporal Lipschitz continuity were used, a uniform time mesh
would need spacing $O(\epsilon/A)$ and hence spatial coefficient vectors
at $O(A/\epsilon)$ time nodes. The corresponding coefficient count would be
$O(A^{d+1}\eta^{-(d+r)})$. The two regression-error balances and their
resulting $W_2$ upper rates, with logarithms and powers of $A$
suppressed, are
\begin{equation}\label{eq:aos-time-comparison}
 \begin{aligned}
 \text{First-order time:}\quad
 &\eta^{2r}+n^{-1}\eta^{-(d+r)}
       &&\Longrightarrow\ n^{-r/(d+3r)},\\
 \text{Multiscale time:}\quad
 &\eta^{2r}+n^{-1}\eta^{-d}
       &&\Longrightarrow\ n^{-r/(d+2r)}.
 \end{aligned}
\end{equation}

\subsection{An oracle inequality for the empirical risk}
By Lemma \ref{lm:gcr-low-precision}, for uniform constants $C,c>0$,
\[
 \sup_{t\in[0,1]}\mathbb P(\|X_t\|_\infty>A)\leq Ce^{-cA^2},
 \qquad
 \mathbb P(\|U\|_\infty>R)\leq Ce^{-cR^2}.
\]
The oracle inequality below uses the response tail.
Appendix \ref{sec:supp-conditional-noise} gives a stronger conditional multiplier bound.

\begin{lemma}[Oracle inequality for the flow matching loss]\label{lm:stoc-error-fm}
Let $\mathbb D_n=((W_i,U_i))_{i=1}^n$ be an i.i.d. sample, write
$P_W$ for the marginal law of $W$, and let $m(W)=\E[U\mid W]$. Suppose $\mathbb P(\|U\|_\infty>s)\leq C_0e^{-c_0s^2}$ for every $s\geq0$. Let $\mathcal F$ be a deterministic finite nonempty class with $\|f\|_\infty\leq B$ for each $f\in\mathcal F$, and let $\hat f$ be an exact empirical squared-loss minimizer with a fixed tie rule. For every analysis threshold $R\geq\max\{1,B\}$,
\begin{align}
 \E_{\mathbb D_n}\|\hat f-m\|_{L^2(P_W)}^2
 &\leq C\inf_{f\in\mathcal F}\|f-m\|_{L^2(P_W)}^2
       +\frac{C R^2\log(2|\mathcal F|)}{n}\notag\\
 &\quad+C n(1+B^2+R^2)e^{-cR^2}.\label{eq:finite-oracle}
\end{align}
Constants depend only on the response tail constants and output dimension.
\end{lemma}

\begin{proof}
For a function $a$ of an observed pair, write
$\mathbb P_n a=n^{-1}\sum_{i=1}^n a(W_i,U_i)$ and
$\mathbb P a=\E a(W,U)$ for its empirical and population averages.
We first use Bernstein's inequality to obtain a risk bound for a finite
class with clipped responses \citep{massart2007concentration}. We then
transfer this bound to the estimator based on the original responses.
Clip each coordinate of $U$ to $[-R,R]$, obtaining $U^R$, and let
$m_R(W)=\E[U^R\mid W]$. Let $\hat f_R$ minimize the clipped empirical squared loss over the same ordered finite class.

For $f\in\mathcal F$ define
\[
 g_f(W,U)
 =\|f(W)-U^R\|_2^2
  -\|m_R(W)-U^R\|_2^2,\qquad
 r_f=\E g_f=\|f-m_R\|_{L^2(P_W)}^2.
\]
Because $R\geq B$ and the conditional expectation is coordinatewise bounded by $R$,
\[
 |g_f|\leq C_dR^2,\qquad
 \operatorname{Var}(g_f)\leq\E g_f^2\leq C_dR^2 r_f.
\]
Indeed, $g_f=(f-m_R)^\top(f+m_R-2U^R)$, and the second factor has Euclidean norm at most $4\sqrt d R$.
Bernstein's inequality, a union bound over $\mathcal F$, and
$\sqrt{ab}\leq a/4+b$ imply that, with probability at least $1-e^{-u}$,
\begin{equation}\label{eq:finite-bernstein}
 |(\mathbb P_n-\mathbb P)g_f|
 \leq \frac12 r_f+
       C_d R^2\{\log(2|\mathcal F|)+u\}/n
 \quad\text{for every }f\in\mathcal F.
\end{equation}
Here one applies the bounded-variable inequality to $g_f-\E g_f$, whose bound is also $C_dR^2$.

Let $a_u$ denote the last term in \eqref{eq:finite-bernstein} and choose a deterministic $f_0$ minimizing $r_f$. The empirical minimization property and the simultaneous event give
\[
 \tfrac12r_{\hat f_R}-a_u
 \leq\mathbb P_n g_{\hat f_R}
 \leq\mathbb P_n g_{f_0}
 \leq\tfrac32r_{f_0}+a_u.
\]
Integrating this tail inequality in $u\geq0$ yields
\begin{equation}\label{eq:clipped-oracle}
 \E_{\mathbb D_n}r_{\hat f_R}
 \leq3\inf_{f\in\mathcal F}r_f+
           C_d R^2\log(2|\mathcal F|)/n.
\end{equation}

Jensen's inequality and the unconditional response tail give
\begin{align*}
 D_R:=\|m_R-m\|_{L^2(P_W)}^2
 &\leq\E\|U^R-U\|_2^2
 \leq C(1+R^2)e^{-cR^2}.
\end{align*}
The last estimate follows from integration of the coordinate tails. Thus
$\inf_f r_f\leq2\inf_f\|f-m\|_{L^2}^2+2D_R$ and
$\|\hat f_R-m\|_{L^2}^2\leq2r_{\hat f_R}+2D_R$.

On the event $\max_i\|U_i\|_\infty\leq R$, the raw and clipped empirical objectives coincide for every candidate, so the fixed tie rule ensures $\hat f=\hat f_R$. Its complement has probability at most $C n e^{-cR^2}$. For any possible candidate,
\[
 \|f-m\|_{L^2}^2\leq2dB^2+2\E\|U\|_2^2
 \leq C(1+B^2).
\]
Consequently, replacing $\hat f_R$ by $\hat f$ in the expected squared regression error costs at most $C n(1+B^2)e^{-cR^2}$. Combining this fact with \eqref{eq:clipped-oracle} and the bounds on $D_R$ proves \eqref{eq:finite-oracle}.
\end{proof}

\begin{corollary}[High-probability oracle inequality]
\label{cor:aos-confidence}
Under the assumptions of Lemma \ref{lm:stoc-error-fm}, let
$0<\alpha<1$ and $R\geq\max\{1,B\}$ satisfy
$Cne^{-cR^2}\leq\alpha/2$, with the tail constants adjusted if necessary.
With probability at least $1-\alpha$,
\begin{align}
 \|\hat f-m\|_{L^2(P_W)}^2
 &\leq C\inf_{f\in\mathcal F}\|f-m\|_{L^2(P_W)}^2
   +\frac{CR^2\log(4|\mathcal F|/\alpha)}{n}\notag\\
 &\quad+C(1+R^2)e^{-cR^2}.
 \label{eq:aos-confidence}
\end{align}
\end{corollary}
\begin{proof}
Use $u=\log(2/\alpha)$ in \eqref{eq:finite-bernstein}.
The displayed empirical comparison in the preceding proof gives
$r_{\hat f_R}\leq3\inf_f r_f+4a_u$.
Intersect this event with the event that no training response is
clipped. Their total failure probability is at most $\alpha$.
The same $L^2$ comparison between $m_R$ and $m$ gives the bound.
\end{proof}

The confidence parameter changes the analysis threshold. For the neural sieve and
$\alpha=n^{-u}$, one can take $R=K_u\sqrt{\log n}$.
The extra term $R^2\log(1/\alpha)/n$ is
$O_u((\log n)^2/n)$, which is smaller than the stated squared
nonparametric rate for fixed $d,r$.

The approximation and spatial tail contributions are
$C\{\eta^{2r}+A^2e^{-cA^2}\}$.
For $R\geq CA$, the preceding lemma and sieve entropy give
\begin{align}
 \E\|\widehat v_n-v^*\|_{L^2(P_W)}^2
 &\leq C\{\eta^{2r}+A^2e^{-cA^2}\}
       +\frac{CR^2A^{d+1}\eta^{-d}L^3}{n}\notag\\
 &\quad+Cn(1+A^2+R^2)e^{-cR^2}.
 \label{eq:gcr-main-oracle-balance}
\end{align}

\subsection{Wasserstein bounds from velocity estimation}
Condition on the learned velocity. Write $\mathtt L_x$ and $\mathtt L_t$
for its spatial and temporal Lipschitz constants in Euclidean norm,
and $\mathtt B$ for its uniform coordinatewise output bound.
Coupling its exact ODE and the true ODE
with the same Gaussian initial point and using the learned spatial
Lipschitz constant gives
\[
 W_2(\widetilde p_1,p)
 \leq e^{\mathtt L_x}\|\widehat v_n-v^*\|_{L^2(P_W)}.
\]
For uniform Euler steps, the deterministic comparison gives
\[
 W_2(\widehat p_n,\widetilde p_1)
 \leq C_d e^{\mathtt L_x}
       (\mathtt L_x\mathtt B+\mathtt L_t)/K_{\mathrm{num}}.
\]
We prove both estimates on $[0,1]$ below and combine them with
\eqref{eq:gcr-main-oracle-balance} in the proof of
Theorem \ref{thm:gcr-rate}.

\textit{ODE stability and Euler error.}
Let $f$ be a fixed realization of the learned velocity with Euclidean spatial
Lipschitz constant $L_x=\mathtt L_x$, uniform Euclidean output bound
$B_2\leq\sqrt d\,\mathtt B$, and temporal Lipschitz constant
$L_t=\mathtt L_t$. Couple the exact true and
learned flows by the same Gaussian initial variable $Z$, writing
$\bar X_t=\Phi_{0,t}(Z)$ and $D_t=\widetilde X_t-\bar X_t$. Their integral equations give
\[
 \|D_t\|_2\leq\int_0^t
 \{L_x\|D_s\|_2+\|f(s,\bar X_s)-v^*(s,\bar X_s)\|_2\}\,ds.
\]
Gronwall followed by Minkowski and Cauchy--Schwarz yields
\[
 \{\E_Z\|D_1\|_2^2\}^{1/2}
 \leq\int_0^1e^{L_x(1-s)}
       \|f(s,\cdot)-v^*(s,\cdot)\|_{L^2(p_s)}\,ds
 \leq e^{L_x}\|f-v^*\|_{L^2(P_W)}.
\]
Only the learned spatial Lipschitz constant occurs in this stability
factor, and the regression error is integrated against the true
interpolation marginals.

For the Euler comparison let $\Delta=1/K_{\mathrm{num}}$.
The exact learned flow moves by at most $B_2\Delta$ on one step.
Its local remainder is therefore at most
$(L_t+L_xB_2)\Delta^2/2$.
If $e_j$ is the pathwise discrepancy at the $j$th mesh point,
\[
 e_{j+1}\leq(1+L_x\Delta)e_j+
                    (L_t+L_xB_2)\Delta^2/2,\qquad e_0=0.
\]
Summing the geometric series gives
$e_{K_{\mathrm{num}}}\leq
e^{L_x}(L_t+L_xB_2)/(2K_{\mathrm{num}})$, also when $L_x=0$
by the same elementary bound. This pathwise estimate bounds $W_2$.
The conversion from the coordinatewise to the Euclidean output bound
accounts for the fixed dimensional constant in the displayed comparison.

\subsection{Proof of the main learning theorem}
\begin{proof}[Proof of Theorem \ref{thm:gcr-rate}]
For each target, the rounded approximant is a deterministic member of
$\mathcal G_{A,\eta}$ with approximation error at most $C\eta^r$ on the
spatial cube. Its squared approximation error there is therefore at most
$C\eta^{2r}$. Outside the cube, its bounded output and the linear growth
of the true velocity bound the integrand by $C(A^2+1+\|x\|_2^2)$.
The uniform fourth-moment bound and the spatial tail bound yield
$CA^2e^{-cA^2}$,
after decreasing $c$ if needed.

Lemma \ref{lm:stoc-error-fm} now gives
\eqref{eq:gcr-main-oracle-balance}. Set
$R=K_R\sqrt{\log n}$, with $K_R$ large enough to dominate the
network output bound and control the response-tail term.
Taking $K_A,K_R$ sufficiently large makes both tail terms
no larger than the displayed regression-error bound. For the prescribed
mesh, $\log(1/\eta_n)\asymp\log n$ and hence
$L_n\asymp\log n$. Writing $a_d=(d+9)/2$, we have
\begin{align*}
 R^2A_n^{d+1}L_n^3&\lesssim(\log n)^{a_d},\\
 n^{-1}(\log n)^{a_d}\eta_n^{-d}
 &=\eta_n^{2r}
 =n^{-2r/(d+2r)}(\log n)^{r(d+9)/(d+2r)}.
\end{align*}
The squared velocity bound follows uniformly over the class.

The preceding deterministic comparisons apply conditionally on
training. By \eqref{eq:gcr-sieve-bounds}, the sieve has
$\mathtt L_x\leq C$, $\mathtt B\leq CA_n$, and
$\mathtt L_t\leq CA_n^2L_n^3$.
The prescribed step count makes their Euler term at most
$C\epsilon_n$. Taking expectations, using Jensen's inequality on
the velocity norm, and applying the triangle inequality proves
the expected $W_2$ bound. On the event from
Corollary \ref{cor:aos-confidence}, the same two deterministic
comparisons give the high-probability assertion, with
$\alpha=n^{-u}$. Every constant is uniform over the fixed class.
\end{proof}

\section{Convergence rates and minimax lower bounds}
\label{sec:gcr-main-limits}

Theorem \ref{thm:gcr-rate} is an upper bound for the specified empirical
least-squares estimator. To make the comparison precise, define
\[
 \mathcal R_n(\mathcal Q_r(\mathbf H))
 =\inf_{\widehat P}\sup_{p\in\mathcal Q_r(\mathbf H)}
       \E_p W_2(\widehat P,p),
\]
where the infimum ranges over measurable estimators based on $n$
independent target observations, including randomized estimators.

\subsection{Convergence of the empirical measure}
The empirical measure is a useful reference because it does not require
velocity estimation. Write $P_n=n^{-1}\sum_{i=1}^n\delta_{Y_i}$.
By Corollary \ref{cor:gcr-geometry}, set
$Z_i^\circ:=F^{-1}(Y_i)$. Then $Z_i^\circ$ are i.i.d. with law $\gamma_d$
and $Y_i=F(Z_i^\circ)$. Let
$\gamma_{d,n}=n^{-1}\sum_i\delta_{Z_i^\circ}$. The push-forward coupling
and its inverse give the pathwise comparison
\[
 L_F^{-1}W_2(\gamma_{d,n},\gamma_d)
 \leq W_2(P_n,p)\leq L_FW_2(\gamma_{d,n},\gamma_d).
\]
The bounds of Gaussian empirical measures
\citep{ledoux2021gaussian},
followed by Jensen's inequality, yield for $n\geq3$
\begin{equation}\label{eq:aos-empirical-benchmark}
 \sup_{p\in\mathcal Q_r(\mathbf H)}\E_pW_2(P_n,p)
 \leq C\mathfrak b_{n,d},\qquad
 \mathfrak b_{n,d}:=
 \begin{cases}
 \{\log\log n/n\}^{1/2},&d=1,\\
 (\log n)/\sqrt n,&d=2,\\
 n^{-1/d},&d\geq3.
 \end{cases}
\end{equation}
The usual moment-based bounds \citep{fournier2015rate} remain valid
but lose this Gaussian geometry in low dimensions.

For $d\geq3$, the flow exponent $r/(d+2r)$ exceeds $1/d$ exactly
when $r(d-2)>d$. For $r=2$ this requires $d>4$. At $d=4$ the
polynomial exponents agree. In dimensions one and two, the upper bound
for the empirical measure decays faster than the bound established for
our flow estimator, for every fixed $r$.
The flow result gives a smoothness-sensitive guarantee for the
specified empirical risk and a learned, discretized transport.

\subsection{Minimax lower bound from Gaussian perturbations}

For sufficiently large fixed constants $\mathbf H$, the smooth Gaussian
perturbation subfamilies of supplementary Theorem
\ref{thm:gcr-minimax-lower} give
\begin{equation}\label{eq:gcr-main-minimax-bracket}
 \begin{aligned}
 c\max\{n^{-1/2},n^{-(r+1)/(d+2r)}\}
 &\leq\mathcal R_n(\mathcal Q_r(\mathbf H))\\
 &\leq C\min\{\mathfrak b_{n,d},\,
 n^{-r/(d+2r)}(\log n)^{r(d+9)/(2(d+2r))}\}.
 \end{aligned}
\end{equation}
The lower bound follows from Assouad's method for the nonparametric
term and Le Cam's two-point method for the parametric term
\citep{assouad1983deux,yu1997assouad}.
The requirement on $\mathbf H$ ensures containment of a nontrivial
perturbation family. Membership of this family in GCR is verified directly
using uniform bounds on the derivatives of the multiplicative perturbation through order
$r$: one derivative transferred to the Gaussian kernel shows
the permitted $h^{-1}$ growth at order $r+1$.
For $d>2$ its nonparametric exponent differs from the flow upper exponent
by $1/(2r+d)$. For $d\leq2$ the displayed parametric lower bound
dominates and the empirical measure matches its polynomial exponent.
The gap for the specified flow estimator can therefore be distinguished
from the minimax gap for unrestricted estimators.
This lower bound comes from a subclass and does not establish the sharp
minimax rate over the full GCR class.
These lower bounds also apply to the observed training triplets,
which preserve the information in the target samples. Appendix
\ref{subsec:sampled-objective-information} explains this equivalence
and distinguishes it from the efficiency of the empirical risk.

\subsection{Distributional error and weighted negative norms}
The population excess risk equals the squared
$L^2(P_W)$ velocity error, but the generated distribution depends on how these errors are
transported and combined over time. A deterministic upper bound
by the full regression norm cannot by itself describe cancellations
in that combination.

To describe the possible gain, consider the true flow map
$\Phi_{t,1}$ and smooth compactly supported test functions $g$.
For a sufficiently regular perturbation $e_t=\widehat v_n(t,\cdot)-v^*(t,\cdot)$,
the linearized endpoint pairing is
\[
 \langle\mathcal T e,g\rangle
 =\int_0^1\int p_t(x)e_t(x)^\top
           \nabla(g\circ\Phi_{t,1})(x)\,dx\,dt.
\]
The weighted negative norm tests this functional against gradients:
\[
 \|\mathcal T e\|_{\dot H^{-1}(p)}
 =\sup_{\substack{g\in C_c^\infty\\
                         \int\|\nabla g\|_2^2p\leq1}}
                       |\langle\mathcal T e,g\rangle|.
\]
Flow stability immediately bounds it by the full regression norm,
but improving the statistical rate requires a sharper bound for
this projected random error, together with a nonlinear remainder
bound. The precise endpoint identities and the additional finiteness
requirements for an informative negative-norm bound are given in
Appendix \ref{sec:endpoint-projection}. These requirements are not shown by
the regression oracle inequality.

A weighted norm can also be too strong in the tails.
The supplementary Gaussian example takes a target
$N(0,\sigma^2)$ with $\sigma^2<1/2$. Any fixed bounded learned
velocity displaces a Gaussian input by only a bounded amount.
Its exact time-one law $\widetilde p_1$ can therefore have tails too heavy
for the full error norm $\|\widetilde p_1-p\|_{\dot H^{-1}(p)}$ to be
finite, even though $W_2(\widetilde p_1,p)$ is finite.
This obstruction is relevant to each finite bounded-output sieve used
here. Convergence of the regression risk to zero does not, by itself,
ensure finiteness of the weighted negative norm.

Using the average of the target density and the density generated by
the exact learned flow at time one avoids this particular finiteness
problem. The argument uses uniform Lipschitz bounds on the time-one
flow maps and their inverses, together with uniform bounds on the images
of the origin under the flow maps. Appendix \ref{subsec:clean-mixture}
obtains these bounds for an optional anchored sieve, defined in
\eqref{eq:gcr-main-anchor}, which retains the same learning rate.

\section{Related work}\label{sec:related-work}
The relevant literature addresses three related questions:
regularity of the population velocity, sampling accuracy
given a velocity approximation, and statistical accuracy
when the velocity is learned from data. We summarize the main
connections below. Appendix \ref{sec:detailed-related-work}
provides detailed comparisons of assumptions, error metrics,
training objectives, and treatment of the terminal time.

Flow matching and stochastic interpolants characterize velocity
fields through conditional expectations
\citep{lipman2023flow,liu2023flow,albergo2023stochastic,gao2024gaussian}.
\citet{benton2024error} study velocity regularity and propagation
of velocity-approximation error. \citet{stephanovitch2026lipschitz}
analyzes regularity and numerical sampling error using tilted
posterior distributions. Our analysis builds on standard Gaussian conditioning
identities and derives the required regularity bounds directly
under GCR. It then connects scale-weighted channel derivatives
and quantitative time analyticity to neural approximation with
uniform control of the spatial Lipschitz constant.

For diffusion models,
\citet{li2024nonasymptotic,li2024sharp,li2024accelerating}
analyze sampling error in terms of numerical discretization and
assumed score-estimation accuracy. Bounds in terms of the number
of sampling steps should be distinguished from statistical rates
in the training sample size. On the statistical side,
\citet{oko2023diffusion} obtain smoothness-dependent total variation (TV) and $W_1$
rates for compactly supported Besov densities, while
\citet{zhang2024minimax} obtain TV rates for sub-Gaussian
Sobolev densities without a density lower bound.
For densities supported on a cube, H\"older smooth on that cube,
and bounded above and away from zero there,
\citet{dou2024optimal} establish sharp minimax rates for score
estimation at each positive noise level. Their distribution estimator,
based on continuous reverse diffusion, attains the minimax TV rate
without additional logarithmic factors or early stopping.
\citet{tang2024adaptivity} and \citet{tang2025conditional} study
adaptation to manifold structure and conditional distribution estimation.
\citet{wibisono2024optimal} establish optimal score estimation rates
for sub-Gaussian densities with Lipschitz scores. Under an additional
uniform Lipschitz condition on scores along the Ornstein--Uhlenbeck
process, they also derive an expected TV guarantee for a discretized
diffusion sampler.
\citet{cai2025minimax} establish near-minimax expected TV rates
for a discretized probability-flow ODE for sub-Gaussian densities
in their H\"older class with $0<\beta\leq2$, without a density
lower bound. Their analysis combines smooth regularized score
estimation using Gaussian kernels with bounds on the score,
Jacobian, and discretization errors.

\citet{stephanovitch2025generalization} obtain nearly minimax $W_1$
rates for all $\beta>0$ and $d\geq3$ under bounded H\"older
perturbations of strongly convex potentials with controlled derivative
growth. They use neural denoising score matching with time and Gaussian
noise integrated out. Their guarantee covers continuous reverse SDEs
and probability-flow ODEs with early stopping.
These results concern different density classes,
error metrics, and treatments of the terminal time, so their
smoothness parameters and rates do not directly translate to GCR.

For flow matching, \citet{fukumizu2025flow} give nearly minimax
$W_p$ rates for compactly supported Besov targets under additional
regularity assumptions on the boundary and the posterior mean.
\citet{he2026theory} study conditional velocity fields with the initial
source point as an additional network input. They analyze gradient
descent for two-layer ReLU networks trained with sampled auxiliary
variables and report high-probability $W_1$ guarantees under bounded
target support and additional coupling and regularity conditions.
\citet{su2025highorder} study high-order flow matching with Transformer
networks, and \citet{su2025kl} derive Kullback--Leibler (KL) error bounds and TV guarantees
under spatially uniform bounds on the score along the target path,
the velocity fields, and their specified derivatives.
\citet{kumar2026manifold} establish Wasserstein guarantees governed by
the intrinsic dimension for targets on compact manifolds without
boundary and of dimension at least three, under additional geometric
and velocity regularity conditions.
\citet{marzouk2023distribution} study distribution estimation
with neural differential equations using a likelihood-based
criterion rather than the flow matching objective.

\citet{kunkel2026lipschitz} combines posterior covariance estimates
with neural approximation of the velocity and its derivatives to obtain
high-probability $W_1$ guarantees on unbounded domains.
\citet{kunkel2026minimax} connect empirical flow matching to kernel
density estimation, separating the kernel estimator's distributional
error from the error in approximating the empirical velocity field
that generates it.

Table \ref{tab:common-flow-comparison} compares these guarantees
with those of \citet{stephanovitch2025generalization} and Theorem \ref{thm:gcr-rate}
on a common class of bounded $C^2$ Gaussian perturbations.
Appendix \ref{subsec:common-flow-comparison} specifies this subclass
and verifies the target conditions.
The three cited procedures integrate the auxiliary randomness and
retain positive terminal noise. Our analysis balances approximation
and estimation errors for the population velocity under the empirical risk.
Then the $W_2$ guarantee at time one is yielded by flow stability and Euler error estimates.

\begin{table}[!ht]
\caption{Guarantees on the common subclass
\eqref{eq:common-flow-subclass} for $d\geq2$. Logarithmic factors are suppressed and all class constants are fixed.
Here $\vartheta>0$ is fixed.}
\label{tab:common-flow-comparison}
\centering
\begin{tabularx}{\linewidth}{@{}p{0.43\linewidth}p{0.21\linewidth}>{\raggedright\arraybackslash}X@{}}
\toprule
Result & Polynomial rate & Guarantee\\
\midrule
\citet{kunkel2026lipschitz}, Theorem 4.3
& $n^{-2/(d+9+\vartheta)}$ & High-probability $W_1$ \\[3pt]
\citet{kunkel2026minimax}, Remark 3.8(3)
& $n^{-2/(d+4)}$ & Expected $W_1$\\[3pt]
\citet{stephanovitch2025generalization},\newline Theorem I ($d\geq3$)
& $n^{-3/(d+4)}$ & Expected $W_1$\\[3pt]
Theorem \ref{thm:gcr-rate}, $r=2$
& $n^{-2/(d+4)}$ & Expected $W_2$ and high-probability $W_2$ \\
\bottomrule
\end{tabularx}
\end{table}

In classical Wasserstein estimation, density smoothness and
density lower bounds play distinct roles \citep{niles2022minimax}.
Convergence rates of the empirical measure provide a further comparison
\citep{weed2019sharp}. Section \ref{sec:gcr-main-limits} makes
these comparisons within the present target class.
Our approximation argument builds on established tools
\citep{yarotsky2017error,guhring2020error,occorsio2021filtered}.
We combine a graded time mesh, spatial approximation with
derivative control, and finite quantization to preserve a uniform
spatial Lipschitz bound for the learned velocity. The complexity of
time approximation grows only polylogarithmically with the inverse accuracy.

\section{Conclusion and discussion}\label{sec:conclusion}
Gaussian-channel regularity provides conditions on the target
distribution under which the population velocity has uniform spatial
regularity of order $r$ and is Lipschitz continuous in time throughout
$[0,1]$. The vanishing coefficient of the score allows derivatives of
the smoothed log density to grow as the noise scale tends to zero,
while posterior concentration yields quantitative local time analyticity
before the terminal time. These properties allow us to construct a
finite neural sieve whose members satisfy a uniform spatial Lipschitz
bound. For an exact empirical risk minimizer over this sieve, we bound the
expected $W_2$ error by
$Cn^{-r/(d+2r)}(\log n)^{r(d+9)/(2(d+2r))}$ and obtain a corresponding
high-probability bound and an explicit bound on the Euler discretization
error.

The upper and lower bounds leave two questions open: whether the
specified empirical risk yields faster distributional rates, and what
the exact minimax rate is over a fixed GCR class.
For the first question, a possible approach is
to bound the projected random velocity error and control the nonlinear
remainder in the terminal distribution. The negative-norm comparisons
also show that the weight needs to accommodate the tails of both the
target and learned distributions.

\clearpage
\section*{Supplementary materials}
\renewcommand{\theassumption}{\thesection.\arabic{assumption}}
\noindent\textbf{Organization and dependencies.}
Appendix \ref{sec:gcr-proofs} completes the construction in
Theorem \ref{thm:gcr-sieve}, using the local time analyticity
proved in Appendix \ref{subsec:supp-time-analytic}. These are the
supplementary inputs to the proof of the statistical upper bound.
Appendix \ref{sec:uqtr-supplement} records an unweighted sufficient
condition and the additional density smoothness it implies, solely for comparison.
Appendix \ref{sec:benchmark-geometry} proves the lower bound in
\eqref{eq:gcr-main-minimax-bracket} directly inside GCR.
Appendices \ref{sec:endpoint-projection} and
\ref{sec:consolidated-results} support the discussion of why a sharper
Wasserstein rate requires further work.
Appendix \ref{sec:concrete-target-proofs} collects the proofs of the
four target-class propositions in Section \ref{subsec:concrete-target-classes},
in the same order. Appendix
\ref{sec:detailed-related-work} gives detailed literature comparisons.
Appendix \ref{sec:supp-conditional-noise} establishes conditional noise
control for the sampled loss. The proof of the statistical upper bound
uses the oracle inequality based on unconditional tail bounds in the
main paper.

Throughout, $Y$ is a target observation, $Z$ is an independent standard
Gaussian, $X_t=(1-t)Z+tY$, and $U=Y-Z$. The notation $p_t$ denotes the
interpolation density and $p_1=p$ the target density.
Let $T\sim\operatorname{Unif}[0,1]$ be independent of $(Y,Z)$, and
write $P_W$ for the law of $W=(T,X_T)$. Thus
\[
 P_W(dt,dx)=p_t(x)\,dt\,dx,\qquad
 \|f\|_{L^2(P_W)}^2
 =\int_0^1\int_{\sR^d}\|f(t,x)\|_2^2p_t(x)\,dx\,dt.
\]
The fixed-time norm $\|\cdot\|_{L^2(p_t)}$ uses the measure $p_t(x)\,dx$.
The true velocity is denoted by $v^*(t,x)$ as in
\eqref{eq:gcr-main-interpolation}. Its dependence on the fixed target
is suppressed. We write $v^*(t,x;q)$ to specify a different target density $q$.
We reserve $\bar X_t=\Phi_{0,t}(Z)$ for the true ODE trajectory,
which has density $p_t$ but need not share the sample paths of $X_t$.
Conditionally on the fitted velocity, $\widetilde p_t$ denotes the exact
learned ODE law, whereas $\widehat p_n$ denotes the Euler output in the main
theorem. The order $r$ and spatial mesh $\eta$ in Appendix
\ref{sec:gcr-proofs} agree with the main paper. The smoothing scale
$h$ is used only for Gaussian channels. The index $s\geq3$ in
uniform quadratic-tilt regularity (UQTR) specifies the highest spatial derivative order controlled
by the unweighted bounds on $\ell_h$ in Appendix
\ref{sec:uqtr-supplement}. In the testing argument,
$\delta$ denotes a perturbation scale and $\varepsilon$ a fixed
amplitude. Neither is an estimator tuning parameter.

\begin{appendix}
\section{Detailed construction of the GCR neural sieve}
\label{sec:gcr-proofs}
\label{sec:analytic-appendix}
The construction uses the velocity bounds
\eqref{eq:physical-envelope}, the endpoint velocity $v^*(1,x)=x$,
and the bounds for analytic continuation in time from Lemma \ref{lm:time-analytic}.
These follow directly from GCR by Theorem \ref{thm:gcr-transfer}.
Constants depend only on $\mathbf H$, $d$, and $r$.

\subsection{Multiplication with derivative control}
\begin{lemma}\label{lm:step2-product}
For every fixed integer $q\geq1$ and $0<\xi<1$, there is a ReLU network $P_{q,\xi}$ on $[-1,1]^q$ such that
\[
 \left\|P_{q,\xi}(z)-\prod_{i=1}^q z_i\right\|_\infty+
 \operatorname{Lip}\left(P_{q,\xi}-\prod_{i=1}^q z_i\right)
 \leq C_q\xi.
\]
It has size and depth at most $C_q(1+\log(1/\xi))$ and width at most $C_q$. The Lipschitz error is on the entire cube, including its boundaries.
\end{lemma}
\begin{proof}
For $u\in[0,1]$, let
\[
 g(u)=2\varrho(u)-4\varrho(u-1/2)+2\varrho(u-1),\qquad
 S_m(u)=u-\sum_{j=1}^m4^{-j}g^{\circ j}(u).
\]
As in \citet[Proposition 2]{yarotsky2017error}, $S_m$ is the linear interpolant of $u^2$ on the grid with spacing $2^{-m}$. This follows by induction: the next tent term corrects the midpoint interpolation error on every previous cell. On a cell $[a,a+2^{-m}]$, the interpolant has slope $2a+2^{-m}$, whence
\[
 \|S_m-u^2\|_\infty\leq 2^{-2m-2},\qquad
 \operatorname{Lip}(S_m-u^2)\leq2^{-m}.
\]
The latter follows from the derivative bound on each cell and continuity across grid points. Carrying the input, current tent value, and partial sum through the layers implements $S_m$ using $O(m)$ parameters and depth. Signed quantities can be carried as their positive and negative parts.

Affinely rescale this construction to a square approximant $q_m$ on $[-4,4]$:
\[
 q_m(z)=64S_m((z+4)/8)-64(z+4)/8+16.
\]
The square error and its Lipschitz constant are at most $C2^{-m}$. Thus, on $[-2,2]^2$,
\[
 M_m(u,v)=\{q_m(u+v)-q_m(u-v)\}/4
\]
approximates $uv$ with both uniform error and Lipschitz error at most $C2^{-m}$.

For a fixed number of factors, iterate $M_m$. To justify derivative control under composition, write
\begin{align*}
 M_m(T(z),z_q)-P(z)z_q
 &=E_m(T(z),z_q)+(T(z)-P(z))z_q,\\
 E_m(u,v)&=M_m(u,v)-uv.
\end{align*}
If the previous product approximant $T$ has uniformly bounded Lipschitz constant and approximates the exact product $P$ in value and Lipschitz seminorm, the right side has both errors at most a fixed multiple of $2^{-m}$ plus the previous errors. Induction gives $C_q2^{-m}$. For $m$ above a fixed threshold depending only on $q$, all intermediate products have absolute value at most $2$, so the multiplication domain is valid. Taking $m$ of order $1+\log(1/\xi)$ proves the claim.
\end{proof}

\subsection{Quantized spatial approximation at fixed time}
\label{sec:physical-approx-proof}
Fix a real time $t$, $A\geq1$, $0<\eta\leq1/2$, and
$\epsilon=\eta^r$. The construction below is applied independently
at the real time nodes of the time approximation.

\begin{lemma}[Stable quantized spatial approximation]\label{lm:clean-spatial}
There is a spatial approximation $H_t$, described by at most
$C(A/\eta)^d$ scalar coefficients from a deterministic alphabet of at
most $CA/\epsilon$ values, such that
\[
 \|H_t-v^*(t,\cdot)\|_{\infty,[-A,A]^d}\leq C\epsilon,\qquad
 \|H_t\|_\infty\leq CA,\qquad
 \operatorname{Lip}(H_t)\leq C.
\]
The last two bounds hold globally after spatial input clipping.
The coefficients may depend on $v^*(t,\cdot)$, but the grid, alphabet, and
basis functions depend only on the fixed class and $(A,\eta)$.
The products in this representation are evaluated exactly before network realization.
\end{lemma}
\begin{proof}
Set $J=\lceil2A/\eta\rceil$ and $\ell=2A/J$, so that
$\eta/2\leq\ell\leq\eta$. On the tensor grid of spacing $\ell$ in
$[-A,A]^d$, let $w_a$ be the tensor products of the one-dimensional
nodal hats. They form a nonnegative partition of unity, have bounded
overlap, and satisfy $\sum_a\nabla w_a=0$ almost everywhere.
For each grid node $a$ define
\[
 T_{t,a}(x)=\sum_{|\alpha|\leq r-1}
    \frac{\partial_x^\alpha v^*(t,a)}{\alpha!}(x-a)^\alpha,
 \qquad G_t(x)=\sum_a w_a(x)T_{t,a}(x).
\]
On the support of $w_a$, Taylor's theorem and
\eqref{eq:physical-envelope} give
\[
 \|T_{t,a}(x)-v^*(t,x)\|_2\leq C\eta^r,\qquad
 \|\nabla T_{t,a}(x)-\nabla_x v^*(t,x)\|_{\mathrm{op}}\leq C\eta^{r-1}.
\]
Using the partition identities before taking norms gives
\[
 \nabla G_t-\nabla v^*(t,\cdot)
 =\sum_a(\nabla w_a)(T_{t,a}-v^*(t,\cdot))
   +\sum_a w_a(\nabla T_{t,a}-\nabla v^*(t,\cdot)).
\]
Thus $\|G_t-v^*(t,\cdot)\|_\infty\leq C\epsilon$ and
$\operatorname{Lip}(G_t)\leq C$. Continuity across the cell
boundaries extends the almost-everywhere derivative bound to the
whole cube.

Let
\[
 z_{a,i}(x)=\operatorname{clip}_{[-1,1]}((x_i-a_i)/\ell),\qquad
 \theta_{t,a,\alpha}=
       \frac{\ell^{|\alpha|}}{\alpha!}\partial_x^\alpha v^*(t,a).
\]
Each scalar coefficient is bounded by $C_HA$. Round it to the nearest
multiple of $\epsilon$ in the fixed alphabet
$\epsilon\mathbb Z\cap[-B_0,B_0]$, where
$B_0=\epsilon\lceil(C_HA+\epsilon)/\epsilon\rceil$.
Denote the rounded coefficients by $\widetilde\theta_{t,a,\alpha}$ and set
\[
 H_t(x)=\sum_{a,|\alpha|\leq r-1}
              \widetilde\theta_{t,a,\alpha}w_a(x)z_a(x)^\alpha.
\]
Here and below the input is first clipped coordinatewise to the
spatial cube. On a hat support, the scaled coordinate is unchanged.
Off that support its polynomial factor is immaterial.
Bounded overlap gives
\[
 \|H_t-G_t\|_\infty\leq C\epsilon,\qquad
 \operatorname{Lip}(H_t-G_t)\leq C\epsilon/\eta\leq C.
\]
The approximation and output bounds follow from these inequalities
and the linear growth in \eqref{eq:physical-envelope}. Input clipping
is nonexpansive, so the global bounds follow as well.
There are $(J+1)^d$ nodes and a fixed number of coefficients per
node. This proves the coefficient count and the alphabet bound.
\end{proof}

\subsection{A bounded discrete time filter}
\label{subsec:time-filter}
The following construction uses a classical de la Vall\'ee-Poussin
filter. See \citet{occorsio2021filtered} for related filtered
polynomial interpolation. Our discrete operator need not interpolate
its node values, and we prove its precise norm and reproduction
bounds below.
\begin{lemma}\label{lm:bounded-time-filter}
For every integer $m\geq1$, there is a linear operator $\mathcal V_m$, using $8m$ real nodes in $[-1,1]$, with polynomial output of degree at most $2m-1$, such that
\[
 \|\mathcal V_m f\|_\infty\leq3\|f\|_\infty,\qquad
 \mathcal V_m P=P\quad(\deg P\leq m).
\]
It applies to vector-valued data and to functions with values in a normed space. In particular, applying it in time multiplies a common spatial Lipschitz constant by at most $3$.
\end{lemma}
\begin{proof}
Let $N=8m$, $\theta_j=2\pi j/N$, and
\[
 F_M(\theta)=\frac1M\left|\sum_{r=0}^{M-1}e^{ir\theta}\right|^2
   =\sum_{|k|<M}(1-|k|/M)e^{ik\theta}.
\]
This kernel is nonnegative. Its discrete average over $\theta-\theta_j$ is $1$ whenever $M<N$, because the only multiple of $N$ in its Fourier support is zero. Set $V_m=2F_{2m}-F_m$. Its Fourier coefficient is $1$ for $|k|\leq m$, $2-|k|/m$ for $m<|k|<2m$, and zero otherwise. If these coefficients are denoted by $v_k$, define
\begin{align}
 \lambda_j(z)&=\frac1N\left\{1+
           2\sum_{k=1}^{2m-1}v_k\cos(k\theta_j)T_k(z)\right\},\notag\\
 (\mathcal V_m f)(z)&=\sum_{j=0}^{N-1}\lambda_j(z)f(\cos\theta_j),
 \label{eq:discrete-filter}
\end{align}
where $T_k(\cos\theta)=\cos(k\theta)$.
For $z=\cos\theta$, $\lambda_j(z)$ is the average of
$V_m(\theta-\theta_j)/N$ and $V_m(\theta+\theta_j)/N$.
Positivity and the discrete means of the two Fej\'er kernels yield
\[
 \sum_j|\lambda_j(z)|\leq3,\qquad \sum_j\lambda_j(z)=1.
\]
Discrete Fourier orthogonality shows reproduction of $T_l$ for $0\leq l\leq m$: frequencies $k\pm l$ have absolute value less than $N$, so no nonzero alias survives, and $v_l=1$. These polynomials span the required polynomial space. The norm and Lipschitz assertions follow from the sum of absolute weights. In particular,
\begin{equation}\label{eq:filter-near-best}
 \|\mathcal V_m f-f\|_\infty
 \leq4\inf_{\deg P\leq m}\|f-P\|_\infty.
\end{equation}
Endpoint interpolation will be imposed separately.
\end{proof}

\subsection{Multiscale time approximation and neural realization}
\label{subsec:analytic-approx-proof}
\begin{proof}[Proof of Theorem \ref{thm:gcr-sieve}]
Write $\epsilon=\eta^r$ and $L=\log(2+A/\epsilon)$. Choose an integer $J_*$ so that
\[
 \delta=2^{-J_*}\leq\epsilon/A<2\delta.
\]
Thus $J_*=O(L)$ and $\delta\asymp\epsilon/A$.
Partition the coarse band $[0,1/2]$ and each band
$[1-2^{-j},1-2^{-(j+1)}]$, $1\leq j<J_*$, into
$N_A=\lceil C_0 A\rceil$ equal cells, where $C_0$ is a sufficiently large fixed constant. Append the terminal cell $[1-\delta,1]$.
There are $B_{\rm cell}\leq CAL$ cells. Let
$\ell_{\min}:=\min_I|I|$ be their minimum length. Then
\begin{equation}\label{eq:analytic-cell-lengths}
 \ell_{\min}\geq c\delta/A\geq c\epsilon/A^2.
\end{equation}
Let $r_A(t)=c(1-t)/A$ be the complex-time radius supplied by
Lemma \ref{lm:time-analytic}, with its fixed constant $c>0$.
For every nonterminal cell $I=[a,b]$, its midpoint $c_I$ and half-length $r_I$ satisfy
$r_I/(1-c_I)\leq1/(2N_A)$, since the distance of each band
from time one is at least its length. Taking
$C_0\geq\max\{1,2/c\}$ therefore gives
$r_I/r_A(c_I)\leq A/(2cN_A)\leq1/4$.
The analytic bound and Cauchy's coefficient estimate show that the degree-$m$ Taylor polynomial at $c_I$ approximates $v^*(t,x)$ throughout $I$ with error
\[
 CA\sum_{k>m}(r_I/r_A(c_I))^k\leq CA4^{-m},
\]
uniformly for $x$ in the spatial cube. Choose
$m=\lceil C_1 L\rceil$ with fixed $C_1$ large enough that this error is at most $C\epsilon$.

\textit{Spatial approximation at real time nodes.}
At each cell endpoint and each real node
$c_I+r_I\cos\theta_j$ used by \eqref{eq:discrete-filter}, form the local spatial Taylor approximation and round its scaled coefficients to multiples of $\epsilon$, as in Lemma \ref{lm:clean-spatial}.
Evaluate its products exactly at this stage and denote the resulting spatial function by $H_t(x)$. Use the same coefficient vector for each occurrence of a cell endpoint, so adjacent cells share their endpoint data. The number of spatial coefficient vectors used at the time nodes, counting separate interior occurrences if desired, is at most $C A L^2$.
After spatial clipping, that lemma gives
\begin{equation}\label{eq:analytic-spatial-data}
 \|H_t-v^*(t,\cdot)\|_{\infty,[-A,A]^d}\leq C\epsilon,\quad
 \|H_t\|_\infty\leq CA,\quad \operatorname{Lip}(H_t)\leq C.
\end{equation}
Each scalar coefficient belongs to a fixed alphabet with at most $CA/\epsilon$ choices. Use the known function $H_1(x)=\operatorname{clip}_{[-A,A]^d}(x)$ at the last endpoint. Its approximation error there is zero.

\textit{Time approximation and endpoint correction.}
On each nonterminal cell apply $\mathcal V_m$, rescaled from $[-1,1]$ to $I$, to its spatial data and denote the result by $T_I(t,x)$. Equations \eqref{eq:filter-near-best} and \eqref{eq:analytic-spatial-data} imply
\[
 \|T_I-v^*\|_{\infty,I\times[-A,A]^d}\leq C\epsilon,\quad
 \|T_I\|_\infty\leq CA,\quad \operatorname{Lip}_x(T_I)\leq C.
\]
The first bound follows by applying the operator first to the exact velocity values, using the degree-$m$ Taylor comparison, and then bounding the perturbation of those values by its norm $3$.
For $u=(t-a)/(b-a)$ define
\begin{equation}\label{eq:analytic-endpoint-correction}
 P_I(t,x)=T_I(t,x)+(1-u)\{H_a(x)-T_I(a,x)\}
                       +u\{H_b(x)-T_I(b,x)\}.
\end{equation}
It has degree at most $2m$ in $u$, exactly matches $H_a,H_b$, and still has uniform error $C\epsilon$, output bound $CA$, and spatial Lipschitz bound $C$.
Indeed the corrected operator has norm at most
$3+(1-u)(1+3)+u(1+3)=7$ on the collection of node and endpoint data. This bound, independent of $m$, applies equally to spatial differences.
On the terminal cell use the linear interpolant between $H_{1-\delta}$ and $H_1$. The temporal Lipschitz bound for the true velocity gives error at most $C\epsilon+CA\delta\leq C\epsilon$ there.

The pieces form a continuous function $P(t,x)$ on the full domain,
because adjacent pieces use identical endpoint data. Its spatial
Lipschitz constant is uniformly bounded. For temporal stability,
it is important to bound the local oscillation rather than the full
output amplitude. Fix a nonterminal cell $I=[a,b]$ of length
$\ell_I=b-a$. For every node $t$ used by its corrected operator,
including both endpoints, \eqref{eq:physical-envelope} and
\eqref{eq:analytic-spatial-data} imply
\begin{align*}
 \|H_t-H_a\|_\infty
 &\leq C\epsilon+
       \|v^*(t,\cdot)-v^*(a,\cdot)\|_{\infty,[-A,A]^d}\\
 &\leq C(\epsilon+A\ell_I),
\end{align*}
where all comparisons are first on the spatial cube and then extended
by input clipping. The corrected operator reproduces constant data:
if every node value equals $H_a$, its output is identically $H_a$.
Its norm bound of seven therefore gives
\begin{equation}\label{eq:analytic-local-oscillation}
 \|P_I-H_a\|_\infty\leq C(\epsilon+A\ell_I).
\end{equation}
No time derivative of a rounded coefficient is needed.

Any degree-$D$ scalar polynomial bounded by $M$ on $[-1,1]$
has Chebyshev coefficients bounded by $2M$ (by cosine orthogonality), and
$\|T_k'\|_\infty\leq k^2$. The latter follows from
$T_k'(\cos\theta)=k\sin(k\theta)/\sin\theta$ and
$|\sin(k\theta)|\leq k|\sin\theta|$, with continuity at the endpoints.
Thus its derivative is bounded by $CM(D+1)^3$. Apply this to
$P_I-H_a$, coordinatewise after rescaling time to $[-1,1]$.
Since $D\leq2m$ and $H_a$ is constant in time,
\[
 \|\partial_tP_I\|_\infty
 \leq Cm^3(A+\epsilon/\ell_I)\leq CA^2L^3.
\]
The last inequality uses \eqref{eq:analytic-cell-lengths} and $A\geq1$.
On the terminal cell the slope is bounded by
\[
 \frac{\|H_1-H_{1-\delta}\|_\infty}{\delta}
 \leq CA+C\epsilon/\delta\leq CA,
\]
because $H_1$ is the exact endpoint velocity on the cube and
$\delta\asymp\epsilon/A$. Continuity across all time cells, followed
by nonexpansive time and spatial clipping, proves
\begin{equation}\label{eq:analytic-nominal-time-bound}
 \operatorname{Lip}_t(P)\leq CA^2L^3.
\end{equation}

\textit{Coefficient count.}
The spatial mesh still has $J=\lceil2A/\eta\rceil$ intervals per coordinate and a fixed number $\binom{d+r-1}{d}$ of polynomial coefficients per node. The total number of scalar quantized coefficients is therefore
\[
 Q_{\rm an}\leq C(J+1)^d\,A L^2
            \leq C A^{d+1}\eta^{-d}L^2.
\]
Endpoint sharing only reduces this count. All temporal transformations are fixed linear maps on these coefficients. The transformed coefficients are not independently quantized.

\textit{Globally continuous representation.}
Index all cells by $i=0,\ldots,B_{\rm cell}-1$, with endpoints $a_i,b_i$, and let $P_i(u,x)$ be their polynomials in the coordinate $u\in[0,1]$. Write
$c_i(t)=\operatorname{clip}_{[0,1]}((t-a_i)/(b_i-a_i))$.
The continuous piecewise polynomial has the global representation
\begin{equation}\label{eq:analytic-telescoping}
 P(t,x)=H_0(x)+\sum_i\{P_i(c_i(t),x)-P_i(0,x)\}.
\end{equation}
For a time in cell $i$, earlier increments telescope to $H_{a_i}-H_0$, and later increments vanish. This proves the identity, including endpoints, after clipping time. It avoids discontinuous indicator gates in the neural implementation.

\textit{Product approximation with a growing number of factors.}
The proof of Lemma \ref{lm:step2-product} also gives a quantitative bound when the number of factors grows. First regard the $q$ factors as independent coordinates on $[-1,1]^q$. Iterating the multiplication construction gives a network whose uniform and Lipschitz errors sum to at most $\xi$ and whose depth is at most
\[
C q\{q+\log(2/\xi)\},\qquad\text{with width at most }Cq.
\]
For completeness, if each multiplication has both errors at most $\zeta$, the decomposition used in that proof bounds the accumulated error $E_q$ by
$E_q\leq 2E_{q-1}+Cq\zeta$, as long as previous products have absolute value at most $2$ and Lipschitz constant at most $q+1$.
Choosing $\zeta\leq\xi/(C4^q)$ ensures these conditions by induction and gives $E_q\leq\xi$. The square construction needs $O(q+\log(2/\xi))$ layers per multiplication. If all $q$ factors depended on distinct inputs, additional parameters would be needed to retain these inputs through the layers. In the products below, only a fixed number of spatial hats, scaled spatial coordinates, and one time coordinate need to be retained. The remaining factors repeat these quantities. Retaining only the distinct quantities through the layers gives width bounded by a class-dependent constant and size $Cq\{q+\log(2/\xi)\}$ per product. The extra Lipschitz factor caused by repetition is included in the accuracy choice below.

Expand the time polynomials in \eqref{eq:analytic-telescoping} into powers of $c_i(t)$. Their scalar weights, as linear combinations of the original spatial coefficients, have absolute value at most $CAe^{Cm}$. To see this explicitly, the recurrence
$T_{k+1}(z)=2zT_k(z)-T_{k-1}(z)$ bounds the sum of absolute monomial coefficients of $T_k(2u-1)$ by $C^k$. The filter weights, endpoint corrections, and subtraction at $u=0$ add at most polynomial factors in $m$, absorbed into $e^{Cm}$.
The total number of product terms is at most
$P_*\leq C Q_{\rm an}(m+1)$. This includes each coordinate of the prescribed terminal function $H_1$, represented as $A$ times a bounded scalar. Each product term has at most $C(d+r+m)$ factors, but only a fixed number of distinct input coordinates. Its factors are spatial hats, clipped scaled spatial coordinates, and the exact clipped time coordinate $c_i(t)$.

Set
\[
 \xi=\frac{c\epsilon}{P_* A e^{Cm}(m+1)}.
\]
The extra factor $m+1$ accounts for the Lipschitz constant introduced by repeating factor coordinates. Implement every product with this error, before forming the fixed linear combinations of quantized coefficients.
Let $F$ be the resulting neural function. Summing over all product terms gives
\begin{align*}
 \|F-P\|_\infty&\leq C\epsilon,\\
 \operatorname{Lip}_x(F-P)&\leq C\epsilon/\eta\leq C,\\
 \operatorname{Lip}_t(F-P)&\leq C\epsilon/\ell_{\min}
                                      \leq CA^2.
\end{align*}
These bounds hold globally after input clipping. For the constructed
rounded approximant, \eqref{eq:analytic-nominal-time-bound} and the
last error bound imply
$\operatorname{Lip}_t(F)\leq CA^2L^3$.
Thus it satisfies the tighter temporal constraint in
\eqref{eq:gcr-sieve-bounds}, without changing the alphabet or the
product accuracy. These bounds do not require exact support preservation
or exact endpoint matching by an individual approximate product. The
full neural sum is continuous, and its error in Lipschitz seminorm is
controlled directly. Since
$\log P_*\leq CL$, $m\leq CL$, and $\log(2/\xi)\leq CL$,
each required product costs at most $CL^2$ parameters and layers. Parallel realization gives size $C Q_{\rm an}L^3$ and depth $CL^2$. Using fixed-dimensional identity channels within each product gives width $C Q_{\rm an}L$.

\textit{Finite-sieve construction.}
Take all quantized coefficient arrays from the fixed alphabet, with shared cell endpoint coefficients, realize their networks as above, and retain only those satisfying \eqref{eq:gcr-sieve-bounds} for a fixed sufficiently large class-dependent constant. Include the zero network as well. The rounded approximant just constructed satisfies the constraints and \eqref{eq:gcr-sieve-approx}, and therefore belongs to the restricted sieve. These constraints are on the realized functions, and are independent of the data and unknown target. We do not give an efficient procedure for checking these constraints. There are at most $1+(CA/\epsilon)^{Q_{\rm an}}$ candidates before restriction, giving \eqref{eq:gcr-sieve-entropy}. The construction also proves \eqref{eq:gcr-sieve-network-size}.
\end{proof}

The regression oracle inequality, tail balance, confidence statement,
and deterministic ODE/Euler comparison are proved in Section
\ref{sec:gcr-main-proof} of the main paper.

\subsection{Alternative Gaussian-channel regularity conditions}
The preceding construction uses exactly the derivative orders and
time bounds supplied by GCR. The following observation describes alternative derivative scale bounds.
\begin{remark}[Alternative derivative scale bounds]\label{rem:gcr-scale-profile}
More generally, retain the score-growth and low-precision conditions
but replace \eqref{eq:gcr-weighted} by
$\|D^m\ell_h\|_\infty\leq C_mh^{-\alpha_m}$.
On $t\geq1/2$, direct differentiation gives
\begin{align*}
 \|D_x v^*(t,\cdot)\|_\infty&\lesssim1+h^{1-\alpha_2},\\
 \|D_x^k v^*(t,\cdot)\|_\infty&\lesssim h^{1-\alpha_{k+1}},\quad k\geq2,\\
 \|\partial_t v^*(t,x)\|_2
 &\lesssim(1+\|x\|_2)
       \{1+h^{1-\alpha_2}+h^{2-\alpha_3}\}.
\end{align*}
Thus $\alpha_2\leq1$ and $\alpha_3\leq2$ suffice for uniform first-order
space-time Lipschitz bounds. The main learning theorem uses the
more specific region $\alpha_m\leq1$ for
$2\leq m\leq r+1$.
\end{remark}

\section{An unweighted sufficient condition for GCR}
\label{sec:uqtr-supplement}
This appendix contrasts the scale-weighted GCR conditions with a
stronger, unweighted condition. We identify the additional density smoothness imposed by this condition.

\begin{assumption}[Uniform quadratic-tilt regularity]
\label{assump:uqtr}
Fix an integer $s\geq3$. A centered density satisfies UQTR$_s$ if
$K(0,b)<\infty$ for every real $b$ and
\begin{align}
 \sup_{\substack{0\leq\lambda\leq1\\b\in\sR^d}}
 \|\operatorname{Cov}_{\Pi_{\lambda,b}}(Y)\|_{\mathrm{op}}
 &\leq C_{\mathrm{low}},\label{eq:uqtr-simple-low}\\
 \sup_{0<h\leq1}\|D^m\ell_h\|_\infty
 &\leq L_m,\qquad 2\leq m\leq s.\label{eq:uqtr-simple-high}
\end{align}
Write $\mathcal U_s(\mathbf L)$ for this class with fixed constants
$\mathbf L=(C_{\mathrm{low}},L_2,\ldots,L_s)$.
\end{assumption}

\begin{proposition}[GCR inclusion and log-density regularity]
\label{prop:uqtr-inclusion}
\label{prop:uqtr-density-order}
Every fixed class $\mathcal U_s(\mathbf L)$ is contained in a
GCR$(s-1)$ class with fixed constants. Each member has a strictly positive
representative with $\log p\in C^{s-1,1}_{\mathrm{loc}}(\sR^d)$ and
\begin{align}
 \|\nabla\log p(x)\|_2&\leq C(1+\|x\|_2),\notag\\
 \|D^j\log p\|_{L^\infty}&\leq C,\quad 2\leq j\leq s,
 \label{eq:uqtr-density-order}\\
 |D^\alpha p(x)|&\leq C_\alpha(1+\|x\|_2)^{|\alpha|}p(x),
 \quad |\alpha|\leq s.\notag
\end{align}
Top-order derivatives are weak derivatives and their bounds hold
almost everywhere. Constants are uniform over the fixed class.
\end{proposition}
\begin{proof}
Lemma \ref{lm:gcr-low-precision} supplies uniform target moments.
For $0<h\leq1$, $\nabla q_h$ is integrable and $\int\nabla q_h=0$.
The uniform Hessian bound therefore gives
\[
 \|\nabla\ell_h(0)\|_2
 \leq\int\|\nabla\ell_h(0)-\nabla\ell_h(x)\|_2q_h(x)\,dx
 \leq L_2\E\|Y+hZ\|_2\leq C.
\]
Hence $\|\nabla\ell_h(x)\|_2\leq C(1+\|x\|_2)$.
Multiplying the unweighted derivative bounds by $h\leq1$
verifies GCR$(s-1)$.

To identify the unsmoothed density, uniform second moments give a
fixed ball $B_R$ with $q_h$-mass at least $1/2$.
The score bound gives $|\ell_h(x)-\ell_h(0)|\leq C_R$ there, so
\[
 \frac12\leq\int_{B_R}q_h\leq |B_R|e^{C_R}q_h(0),
 \qquad
 1\geq\int_{B_1}q_h\geq |B_1|e^{-C_1}q_h(0).
\]
Thus $\ell_h(0)$ is uniformly bounded above and below.
On each fixed ball, the functions $\ell_h$ are bounded in
$W^{s,\infty}$. Arzela--Ascoli and a diagonal subsequence give
convergence in $C^{s-1}$ on compact sets to a function $\ell$ whose
order-$s$ weak derivatives retain the uniform bounds.
Approximate-identity convergence $q_h\to p$ in $L^1$ identifies
$p=e^\ell$ almost everywhere. Hence
$\ell\in C^{s-1,1}_{\mathrm{loc}}$.
The weak chain rule expresses $D^\alpha p/p$ as a sum of products
of derivatives of $\ell$ of total order $|\alpha|$.
Only its first derivatives have linear growth, which proves the
density bounds.
\end{proof}

The distinction is strict already at order two.
The strongly convex potential in \eqref{eq:gcr-strict-example}
satisfies GCR$(2)$ by Proposition \ref{prop:gcr-c2}, whereas its
unbounded third log-density derivative excludes UQTR$_3$ by the
preceding proposition. Thus the density-regularity conclusion above
cannot be imposed on all GCR targets. Learning guarantees for this
stronger subclass follow from the main theorem at $r=s-1$.

\section{Minimax lower bound from Gaussian perturbations}
\label{sec:benchmark-geometry}
\label{sec:benchmark-proofs}
This appendix proves the lower bound in
\eqref{eq:gcr-main-minimax-bracket} directly under GCR.
All constants defining the class are fixed independently of the sample size and
perturbation scale. Membership of the testing family is verified
without assuming uniform unweighted bounds on the spatial derivatives
of $\ell_h=\log(p*\varphi_{h^2})$.

\subsection{Observation model and minimax lower bound}
Define the minimax risk, distinct from the risk of a particular training procedure, by
\begin{equation}\label{eq:gcr-minimax-risk}
 \mathcal R_n(\mathcal P)
 =\inf_{\widehat P}\sup_{p\in\mathcal P}
       \E_p W_2(\widehat P,p).
\end{equation}
The infimum is over measurable estimators taking values in the space of
probability measures and based on $n$ independent target observations,
including randomized estimators.

\begin{theorem}[Minimax lower bound under Gaussian-channel regularity]
\label{thm:gcr-minimax-lower}
For each fixed $d$ and integer $r\geq2$, there are finite GCR constants
$\mathbf H_*$ such that, for every componentwise larger
$\mathbf H\geq\mathbf H_*$ and all sufficiently large $n$,
\begin{equation}\label{eq:gcr-minimax-lower}
 \mathcal R_n(\mathcal Q_r(\mathbf H))
 \geq c\max\{n^{-1/2},n^{-(r+1)/(2r+d)}\}.
\end{equation}
The constant $c>0$ does not depend on $n$. The lower bound already
holds on strictly positive densities with full support obtained by
adding smooth compactly supported perturbations to a Gaussian density.
The corresponding multiplicative perturbations relative to the Gaussian
density have uniformly bounded derivatives
through order $r$, independently of the perturbation scale.
\end{theorem}

The requirement $\mathbf H\geq\mathbf H_*$ ensures that the class contains a nontrivial
local family. Proposition
\ref{prop:gcr-finite-perturbation}, proved in Appendix
\ref{subsec:supp-finite-perturbation}, verifies GCR membership.
Assouad's method \citep{assouad1983deux}
then gives the nonparametric term, and Le Cam's two-point method gives
the parametric term. See \citet{yu1997assouad} for an exposition of
both testing reductions. Their application here is detailed below.
The density smoothness, distributional loss, and tail assumptions matter in minimax
comparisons. See \citet{niles2022minimax} for the distinct compact-domain
setting. This lower bound is not a matching minimax characterization
and does not settle possible logarithmic factors in critical dimensions.

\subsection{Testing argument}
\begin{proof}[Proof of Theorem \ref{thm:gcr-minimax-lower}]
Fix the regularity order $r\geq2$ and a nonzero
$u\in C_c^\infty((0,1)^d)$ for which
$\psi=\partial_1^2u$ is nonzero. Integration by parts gives
\[
 \int\psi(z)\,dz=0,\qquad
 \int z_k\psi(z)\,dz=0,\quad 1\leq k\leq d.
\]
Let $\delta=1/k$ for an integer $k\geq1$ and partition $[0,1]^d$ into
$M=k^d=\delta^{-d}$ cubes $a_j+[0,\delta]^d$. Define
\begin{equation}\label{eq:gaussian-assouad-family}
 p_\theta(x)=\varphi_1(x)+
     \varepsilon \delta^r\sum_{j=1}^M\theta_j
                  \psi\left(\frac{x-a_j}{\delta}\right),
 \qquad \theta\in\{-1,1\}^M.
\end{equation}
The constant $\varepsilon>0$ will be fixed sufficiently small, independently
of $\delta$. The perturbations have disjoint supports, vanish in
neighborhoods of their cell boundaries, and have zero mass and first
moment, even after translation. Thus $p_\theta$ is normalized and
centered. Since $\varphi_1$ is bounded below on $[0,1]^d$, small $\varepsilon$
gives
$\varphi_1/2\leq p_\theta\leq3\varphi_1/2$ globally.
For
$g_\theta=(p_\theta-\varphi_1)/\varphi_1$, differentiation on that
fixed cube gives
\[
 \|g_\theta\|_\infty\leq C\varepsilon,\qquad
 \|D^\alpha g_\theta\|_\infty\leq C_\alpha\varepsilon,\quad
 1\leq|\alpha|\leq r,
\]
because $\delta^{r-j}\leq1$ for $j\leq r$. There is no factor $M$:
at any point at most one bump contributes. Proposition
\ref{prop:gcr-finite-perturbation} gives common GCR$(r)$ constants
$\mathbf H_*$, independent of $\delta,\theta$.
These targets have full support and belong to the class with fixed
constants in the theorem statement.

Here is a separation argument using only Lipschitz test functions.
Choose a fixed $L_\psi>0$ such that
$\chi=\psi/L_\psi$ has $\|\nabla\chi\|_\infty\leq1$.
For $\theta,\theta'$, sum the functions
\[
 \operatorname{sign}(\theta_j-\theta_j')\,
           \delta\chi((x-a_j)/\delta)
\]
over the differing coordinates, and extend by zero outside their
supports. The resulting smooth compactly supported function has
gradient norm at most one globally. Its integral against
$p_\theta-p_{\theta'}$ equals
$2\varepsilon \delta^{r+d+1}\int\psi^2/L_\psi$ times the Hamming distance
$H(\theta,\theta')$. For any coupling of $p_\theta$ and $p_{\theta'}$,
the absolute difference between the expectations of this test function
is at most the expected distance between the coupled variables. Hence
\begin{equation}\label{eq:bump-w1-separation}
 W_1(p_\theta,p_{\theta'})
 \geq c_\psi\varepsilon \delta^{r+d+1}H(\theta,\theta').
\end{equation}
For neighboring $p_\theta$ and $p_{\theta'}$, the lower bound for
$p_{\theta'}$ on the cube and the Kullback--Leibler (KL) divergence bound
$\mathrm{KL}(p,q)\leq\int(p-q)^2/q$ give
\begin{equation}\label{eq:bump-kl}
 \mathrm{KL}(p_\theta,p_{\theta'})
 \leq C\varepsilon^2 \delta^{2r+d}.
\end{equation}

For completeness, take any estimator $\widehat P$ with finite risk
at all vertices, and decode a nearest vertex in $W_1$, using a fixed
tie rule. The triangle inequality gives
$W_1(p_{\widehat\theta},p_\theta)
 \leq2W_1(\widehat P,p_\theta)\leq2W_2(\widehat P,p_\theta)$.
Under the uniform prior on the vertices, pair the experiments with
opposite values of each bit. The sum of the two testing errors is
at least one minus their total variation distance. By Pinsker's
inequality and \eqref{eq:bump-kl}, this distance is at most
$\sqrt{Cn\varepsilon^2\delta^{2r+d}/2}$. Additional independent randomization
has no effect on this bound. Choose
$k=\lceil n^{1/(2r+d)}\rceil$ and reduce the fixed $\varepsilon$ so the last
bound is at most $1/2$. The prior-averaged Hamming loss is then at
least $M/4$. Combining with \eqref{eq:bump-w1-separation} proves
\[
 \sup_\theta\E_\theta W_2(\widehat P,p_\theta)
 \geq c \delta^{r+d+1}M
 =c \delta^{r+1}\geq c'n^{-(r+1)/(2r+d)}.
\]
Estimators with infinite risk already satisfy this lower bound.

To obtain the parametric bound in all dimensions, use the same
fixed unit-cube bump and the two densities
$p_\pm=\varphi_1\pm\delta_n\psi$, where
$\delta_n=\varepsilon_0n^{-1/2}$ and $\varepsilon_0$ is sufficiently small.
Proposition \ref{prop:gcr-finite-perturbation} places these densities in the same
fixed GCR class, enlarging $\mathbf H_*$ once if necessary,
independently of $n$.
Their $W_1$ separation is at least $c\delta_n$, and the total variation
distance between the corresponding product laws for $n$ independent
observations is at most
$\sqrt{Cn\delta_n^2/2}\leq1/2$. The same two-point decoding argument
gives $c'n^{-1/2}$. Taking the smaller of the two fixed lower-bound
constants proves \eqref{eq:gcr-minimax-lower}.
\end{proof}
\subsection{Interpretation of the lower bound}
Theorem \ref{thm:gcr-minimax-lower}, together with the upper bound in
Theorem \ref{thm:gcr-rate} and the bound for the empirical measure in
\eqref{eq:aos-empirical-benchmark}, proves
\eqref{eq:gcr-main-minimax-bracket}. The testing family uses
perturbation amplitude $\delta^r$: the Gaussian kernel supplies
one additional derivative at cost $h^{-1}$, exactly as permitted
by the scale-weighted GCR condition.

For $d>2$, the nonparametric term in
\eqref{eq:gcr-minimax-lower} dominates $n^{-1/2}$. Its exponent is
$(r+1)/(2r+d)$, whereas the main upper exponent is $r/(2r+d)$.
For $d\leq2$, the displayed parametric term dominates.

\section{Wasserstein bounds via weighted negative norms}
\label{sec:endpoint-projection}
\label{sec:endpoint-proofs}
The main theorem controls distribution error by the full regression
norm. This appendix studies deterministic bounds on the error in the
terminal distribution using weighted negative norms, conditional on a
fitted velocity.

\subsection{Adjoint decomposition of distributional error}
Condition on the training sample. Let $\tilde p_t$ denote the law of
the exact learned ODE. Write $e_t=\widehat v_n(t,\cdot)-v^*(t,\cdot)$ and let $\Phi_{t,u}$ be the true flow from $t$ to $u$. Set
\[
 \Lambda=\int_0^1\operatorname{Lip}_x(v^*(t,\cdot))\,dt<\infty,
 \qquad \psi_t^g=g\circ\Phi_{t,1}.
\]
For a zero-mass signed measure or distribution $f$, define
\begin{equation}\label{eq:weighted-negative-norm}
 \|f\|_{\dot H^{-1}(p_1)}
 =\sup_{\substack{g\in C_c^\infty(\sR^d)\\
                  \int\|\nabla g\|_2^2p_1\leq1}}
       |\langle f,g\rangle|.
\end{equation}
This norm may be infinite.

Define the linear terminal operator and its adjoint-projected norm by
\begin{align}
 \langle\mathcal T e,g\rangle
 &=\int_0^1\int p_t(x)e_t(x)\cdot\nabla\psi_t^g(x)\,dx\,dt,\notag\\
 \mathfrak A(e)&=\|\mathcal T e\|_{\dot H^{-1}(p_1)}.
 \label{eq:adjoint-projected-norm}
\end{align}
For $t<1$, $p_t$ is positive, and the learned Lipschitz flow has a density. Let
\begin{equation}\label{eq:endpoint-remainder-size}
 \mathfrak R(e)=\int_0^1
   \left\|e_t\left(\frac{\tilde p_t}{p_t}-1\right)\right\|_{L^2(p_t)}\,dt,
\end{equation}
allowing the value $+\infty$.

\begin{theorem}[Adjoint decomposition with a nonlinear remainder]
\label{thm:endpoint-projection}
Suppose the target satisfies the main GCR assumptions and the learned
velocity is measurable in time and, for almost every $t$, is globally
spatially Lipschitz. More precisely, suppose there are nonnegative
functions $L_{\mathrm{fit}},a_{\mathrm{fit}}\in L^1(0,1)$ such that
\[
 \operatorname{Lip}_x(\widehat v_n(t,\cdot))\leq L_{\mathrm{fit}}(t),
 \qquad \|\widehat v_n(t,0)\|_2\leq a_{\mathrm{fit}}(t)
 \quad\text{for almost every }t.
\]
Then, for each $g\in C_c^\infty$,
\begin{align}
 \langle\tilde p_1-p_1,g\rangle
 &=\int_0^1\int \tilde p_t e_t\cdot\nabla\psi_t^g\,dx\,dt\notag\\
 &=\langle\mathcal T e,g\rangle+
   \int_0^1\int(\tilde p_t-p_t)e_t\cdot\nabla\psi_t^g\,dx\,dt.
 \label{eq:exact-endpoint-decomposition}
\end{align}
Moreover,
\begin{align}
 \mathfrak A(e)&\leq e^\Lambda\int_0^1\|e_t\|_{L^2(p_t)}\,dt<\infty,\notag\\
 \mathfrak A(e)&\leq e^\Lambda\|e\|_{L^2(P_W)},\notag\\
 \|\tilde p_1-p_1-\mathcal T e\|_{\dot H^{-1}(p_1)}
 &\leq e^\Lambda\mathfrak R(e),\label{eq:endpoint-remainder-bound}\\
 W_2(\tilde p_1,p_1)
 &\leq2\{\mathfrak A(e)+e^\Lambda\mathfrak R(e)\}.
 \label{eq:endpoint-projected-w2}
\end{align}
The $L^2(P_W)$ norm may be infinite under the stated time-integrability
assumptions, but the linear term $\mathfrak A(e)$ is always finite.
The last bound is informative when $\mathfrak R(e)<\infty$.
This additional finiteness is not supplied by the existing regression
oracle inequality.
\end{theorem}

The factor $2$ follows from the weighted comparison proved below,
also given by \citet[Theorem 1]{peyre2018comparison}. A uniform positive lower bound on the density ratio, $\tilde p_1/p_1\geq c>0$, is not necessary for this one-sided comparison.

There is also a geometric form of the linear term. Let $\mathcal P_{p_1}$ be the orthogonal projection in $L^2(p_1;\sR^d)$ onto the closure of gradients of compactly supported smooth functions, and define
\[
 B_e(y)=\int_0^1
 D\Phi_{t,1}(\Phi_{1,t}(y))\,e_t(\Phi_{1,t}(y))\,dt.
\]
Spatial derivatives can be interpreted almost everywhere. Under the
hypotheses of Theorem \ref{thm:endpoint-projection}, this time integral
defines an element of $L^2(p_1;\sR^d)$, as verified below. Thus its
orthogonal projection is well defined, and
\begin{equation}\label{eq:terminal-gradient-projection}
 \mathcal T e=-\operatorname{div}(p_1B_e),\qquad
 \mathfrak A(e)=\|\mathcal P_{p_1}B_e\|_{L^2(p_1)}.
\end{equation}
To preserve possible cancellations, we retain the projection and the time integral before applying norms.

\subsubsection{Weighted transport comparison}
For probability densities $p,q$ with finite second moments, \citet[Theorem 1]{peyre2018comparison} gives
\begin{equation}\label{eq:weighted-transport-comparison}
 W_2(p,q)\leq2\|q-p\|_{\dot H^{-1}(p)}.
\end{equation}
Here is a flux proof on $\sR^d$. If the norm on the right is finite, Riesz representation on the $L^2(p)$ closure of gradients gives a vector field $G$ with
$-\operatorname{div}(pG)=q-p$ and
$\|G\|_{L^2(p)}=\|q-p\|_{\dot H^{-1}(p)}$.
Use the density path and flux
\[
 \rho_u=(1-u)^2p+(2u-u^2)q,\qquad j_u=2(1-u)pG.
\]
They satisfy $\partial_u\rho_u+\operatorname{div}j_u=0$, and, for $u<1$,
\[
 \int\frac{\|j_u\|_2^2}{\rho_u}\,dx
 \leq4\int\|G\|_2^2p\,dx.
\]
The dynamic transport bound gives \eqref{eq:weighted-transport-comparison}. Densities with zeros cause no issue: the flux vanishes where $p=0$, and its energy is interpreted in the usual absolutely continuous sense. If the negative norm is infinite, the inequality is vacuous.

\begin{proof}[Proof of Theorem \ref{thm:endpoint-projection}]
Condition on the data. The assumptions imply the integrable growth
bound $\|\widehat v_n(t,x)\|_2\leq
a_{\mathrm{fit}}(t)+L_{\mathrm{fit}}(t)\|x\|_2$.
Carath\'{e}odory existence and uniqueness and forward and backward
Gronwall estimates give a global bi-Lipschitz learned flow, with
Lipschitz constants bounded by $\exp(\int_0^1L_{\mathrm{fit}})$.
Its trajectories have uniformly bounded second moments, using
$\int_0^1a_{\mathrm{fit}}<\infty$ and the Gaussian initial law.
The bi-Lipschitz property also implies that the learned flow preserves
absolute continuity. The true flow has the
same properties by Theorem \ref{thm:gcr-transfer}. For $t<1$,
$p_t$ is positive, so $\tilde p_t/p_t$ is defined almost everywhere.

The true interpolation has uniformly bounded second moments. Let
\[
 M_X=\sup_{0\leq t\leq1}
       \left(\int\|x\|_2^2p_t(x)\,dx\right)^{1/2}<\infty.
\]
The learned growth bound and the true bound
$\|v^*(t,x)\|_2\leq C(1+\|x\|_2)$ imply, for almost every $t$,
\[
 \|e_t\|_{L^2(p_t)}
 \leq a_{\mathrm{fit}}(t)+M_XL_{\mathrm{fit}}(t)+C(1+M_X).
\]
Consequently,
\begin{equation}\label{eq:adjoint-error-integrability}
 \int_0^1\|e_t\|_{L^2(p_t)}\,dt
 \leq\|a_{\mathrm{fit}}\|_{L^1(0,1)}
      +M_X\|L_{\mathrm{fit}}\|_{L^1(0,1)}+C(1+M_X)<\infty.
\end{equation}
This is an $L^1$-in-time, $L^2$-in-space estimate. No square
integrability in time of $a_{\mathrm{fit}}$ or $L_{\mathrm{fit}}$ is assumed.

For $g\in C_c^\infty$, the function $\psi_t^g=g\circ\Phi_{t,1}$ solves
$\partial_t\psi_t^g+v^*(t,\cdot)\cdot\nabla\psi_t^g=0$, with terminal value $g$.
The chain rule along the learned trajectories gives
\[
 \frac{d}{dt}\psi_t^g(\widetilde X_t)
   =e_t(\widetilde X_t)\cdot\nabla\psi_t^g(\widetilde X_t)
 \quad\text{for almost every }t.
\]
This identity remains valid using almost-everywhere spatial derivatives of the flows. The flow Lipschitz bounds give bounded gradients for this test function, and linear growth and finite moments provide integrable domination. Integrate the identity over compact subintervals of $(0,1)$ and let their endpoints tend to $0$ and $1$, using this domination. Taking expectations then gives the integrated identity. The initial term is
$\E g(\Phi_{0,1}(Z))=\int gp_1$, and the terminal term is $\int g\tilde p_1$. This proves the first line of \eqref{eq:exact-endpoint-decomposition}. Adding and subtracting $p_t$ proves the second.

Since $(\Phi_{t,1})_\#p_t=p_1$ and
$\|D\Phi_{t,1}\|_{\rm op}\leq e^\Lambda$ almost everywhere,
\begin{equation}\label{eq:adjoint-gradient-energy}
 \|\nabla\psi_t^g\|_{L^2(p_t)}
 \leq e^\Lambda\|\nabla g\|_{L^2(p_1)}.
\end{equation}
Cauchy--Schwarz in space, followed by time integration and
\eqref{eq:adjoint-error-integrability}, proves
$\mathfrak A(e)\leq e^\Lambda\int_0^1\|e_t\|_{L^2(p_t)}\,dt<\infty$.
When $e\in L^2(P_W)$, Cauchy--Schwarz in time further gives
$\mathfrak A(e)\leq e^\Lambda\|e\|_{L^2(P_W)}$.
Otherwise this latter inequality is understood with an infinite
right-hand side.
Applying the same estimate to $(\tilde p_t/p_t-1)e_t$ bounds the remainder by
$e^\Lambda\mathfrak R(e)$. Taking the supremum over normalized tests gives
\eqref{eq:endpoint-remainder-bound}. The triangle inequality for the dual norm and \eqref{eq:weighted-transport-comparison} prove \eqref{eq:endpoint-projected-w2}.

For \eqref{eq:terminal-gradient-projection}, first note that the
push-forward identity and the flow derivative bound give
\[
 \left\|D\Phi_{t,1}(\Phi_{1,t}(\cdot))
            e_t(\Phi_{1,t}(\cdot))\right\|_{L^2(p_1)}
 \leq e^\Lambda\|e_t\|_{L^2(p_t)}.
\]
Together with \eqref{eq:adjoint-error-integrability}, this proves
that the integral defining $B_e$ exists in $L^2(p_1;\sR^d)$.
Minkowski's inequality yields
\[
 \|B_e\|_{L^2(p_1)}
 \leq e^\Lambda\int_0^1\|e_t\|_{L^2(p_t)}\,dt<\infty.
\]
The same integrable bound justifies changing the order of integration
in the pairing with $\nabla g$. Substituting $y=\Phi_{t,1}(x)$ gives
\[
 \int p_1(y)B_e(y)\cdot\nabla g(y)\,dy.
\]
The dual norm of this functional on the gradient closure is exactly the norm of the orthogonal projection of $B_e$ onto that closure. This proves both assertions.
\end{proof}
\begin{remark}[Conditions for sharper rates]
Theorem \ref{thm:gcr-rate} bounds $\E\|e\|_{L^2(P_W)}^2$. It does not show a sharper bound for $\E\mathfrak A(e)$ or control \eqref{eq:endpoint-remainder-size}. For example, a uniform error bound $\sup_{t,x}\|e_t(x)\|_2\leq a$ would give
$\mathfrak R(e)\leq a\int_0^1\|\tilde p_t/p_t-1\|_{L^2(p_t)}\,dt$,
but neither ingredient in this sufficient condition follows from the present oracle inequality. A sharper statistical rate for the terminal distribution requires control of both terms in \eqref{eq:endpoint-projected-w2} and an Euler step size small enough that the discretization error is of the same or smaller order.
\end{remark}

\subsection{Divergence of the target-weighted negative norm}
\begin{proposition}[Divergence of the target-weighted norm under bounded velocities]
\label{prop:weighted-tail-obstruction}
Let $d=1$, the source be $N(0,1)$, and the target be
$p_1=N(0,\sigma^2)$ with $0<\sigma^2<1/2$. This target satisfies
GCR$(r)$ for every fixed integer $r\geq2$. If the learned velocity has a
finite uniform output bound, then its exact endpoint law satisfies
\begin{equation}\label{eq:infinite-weighted-error}
 W_2(\tilde p_1,p_1)<\infty,\qquad
 \|\tilde p_1-p_1\|_{\dot H^{-1}(p_1)}=\infty.
\end{equation}
This applies to every member of the bounded-output sieves in this paper, for each finite $n$, irrespective of its training objective.
\end{proposition}

Consequently, the target-weighted comparison in Theorem
\ref{thm:endpoint-projection} is vacuous for these targets and
bounded-output velocities.

A valid alternative comparison uses the endpoint mixture
$\mu_1=(p_1+\tilde p_1)/2$:
\begin{equation}\label{eq:mixture-weight-transport}
 W_2(\tilde p_1,p_1)
 \leq2\|\tilde p_1-p_1\|_{\dot H^{-1}(\mu_1)}.
\end{equation}
Indeed, applying the weighted comparison with reference measure $\mu_1$ and using the triangle inequality gives
\[
 \begin{aligned}
 W_2(\tilde p_1,p_1)&\leq W_2(\tilde p_1,\mu_1)+W_2(\mu_1,p_1)\\
 &\leq2\|\tilde p_1-\mu_1\|_{\dot H^{-1}(\mu_1)}
      +2\|p_1-\mu_1\|_{\dot H^{-1}(\mu_1)}
 =2\|\tilde p_1-p_1\|_{\dot H^{-1}(\mu_1)}.
 \end{aligned}
\]
This comparison is understood as vacuous if its right side is infinite.
It suggests studying error bounds with a reference weight that depends on the learned law, or with weights that vary along a path joining the target and learned laws. For each finite $n$, Section \ref{subsec:clean-mixture} bounds the negative norm with mixture weight uniformly over the fixed target class and over sieve members whose velocities satisfy a fixed bound at the origin. The loss and sampling scheme remain unchanged. Another possibility is to use velocity fields that are affine outside a bounded region, which would require a separate approximation and oracle analysis.

\begin{proof}[Proof of Proposition \ref{prop:weighted-tail-obstruction}]
For the Gaussian target, the tilted covariance is
$(\lambda+\sigma^{-2})^{-1}$, uniformly bounded at low precision,
and all real moment generating functions are finite. Moreover,
$q_h$ is the density of $N(0,\sigma^2+h^2)$, so
\[
 \ell_h'(x)=-\frac{x}{\sigma^2+h^2},\qquad
 \ell_h''(x)=-\frac1{\sigma^2+h^2},\qquad
 \ell_h^{(m)}=0\quad(m\geq3).
\]
These formulas verify the score-growth and weighted derivative
conditions of GCR directly, with constants depending on the fixed
$\sigma$.

Let $B=\int_0^1\|\widehat v_n(t,\cdot)\|_\infty\,dt<\infty$.
Let $\widehat F$ be the learned time-one flow map. Then
$|\widehat F(Z)-Z|\leq B$ pathwise, and therefore
\[
 \tilde p_1((x,\infty))\geq\mathbb P(Z>x+B)=\bar\Phi(x+B).
\]
Write $F_q,F_p$ for the distribution functions of $q=\tilde p_1$ and
$p=N(0,\sigma^2)$. Since $\sigma^2<1$, the target tail is negligible compared with $\bar\Phi(x+B)$ as $x\to\infty$. The Gaussian tail lower bound gives, for all sufficiently large $x$,
\[
 |F_q(x)-F_p(x)|\geq \tfrac12\bar\Phi(x+B)
 \geq c\,\frac{\exp(-(x+B)^2/2)}{x+B}.
\]
For example, the last inequality follows by integrating the standard Gaussian density on $[x+B,x+B+(x+B)^{-1}]$.

If $\|q-p\|_{\dot H^{-1}(p)}$ were finite, the Riesz argument above would give a finite-energy flux $J=pG$ satisfying $-J'=q-p$. In one dimension this forces
$J=F_p-F_q+C$. As both distribution functions tend to $1$, any $C\ne0$ makes $\int J^2/p$ infinite. Thus $C=0$ and finiteness would require
$\int(F_p-F_q)^2/p<\infty$. But its integrand is bounded below in the right tail by a positive constant times
\[
 \frac{\exp\{(1/(2\sigma^2)-1)x^2-2Bx-B^2\}}{(x+B)^2},
\]
which has divergent integral because $1/(2\sigma^2)-1>0$.
This contradiction proves the infinite negative norm. In contrast,
$|\widehat F(Z)|\leq|Z|+B$ gives finite second moments, hence finite $W_2$.
\end{proof}

\subsection{No uniform derivative gain for arbitrary velocity errors}
\begin{proposition}[A high-frequency obstruction for unrestricted errors]
\label{prop:endpoint-no-gain}
Take $d=1$, source and target both standard Gaussian with density
$\phi$, and let
$r_t=(t^2+(1-t)^2)^{1/2}$, so that $v^*(t,x)=r_t'x/r_t$.
For $k\geq1$ and $\varepsilon>0$, define
\[
 e_t(x)=\varepsilon r_t\sin(kx/r_t).
\]
Then
\begin{equation}\label{eq:no-gain-ratio}
 \mathfrak A(e)^2=\frac{\varepsilon^2}{2}(1-e^{-2k^2}),\qquad
 \|e\|_{L^2(P_W)}^2=\frac23\mathfrak A(e)^2.
\end{equation}
The exact learned endpoint obeys
\[
 W_2(\tilde p_1,p_1)\geq
 \varepsilon\left(\frac{1-e^{-2k^2}}2\right)^{1/2}
 -\frac{\varepsilon^2k}{2}.
\]
For $\varepsilon=c k^{-r}$ with a sufficiently small fixed $c>0$ and integer $r\geq1$, this gives no additional factor $k^{-1}$.
The velocity error nevertheless has spatial derivatives through order $r$
that are uniformly bounded in $k$, and its temporal Lipschitz bound
is at most $C(1+|x|)$ with $C$ independent of $k$.
\end{proposition}

This is an obstruction to deducing a derivative gain solely from regularity and gradient projection. It is not a lower bound for the particular finite-sieve ERM, and the displayed perturbed velocity has linear growth rather than a global output bound.

\begin{proof}[Proof of Proposition \ref{prop:endpoint-no-gain}]
The true ODE trajectory is $\bar X_t=\Phi_{0,t}(Z)=r_t Z$, with $\Phi_{t,1}(x)=x/r_t$. Thus
\[
 B_e(y)=\int_0^1r_t^{-1}\varepsilon r_t\sin(ky)\,dt=\varepsilon\sin(ky).
\]
This is the derivative of $-\varepsilon\cos(ky)/k$. Multiplying the potential by smooth expanding cutoffs approximates its derivative in $L^2(\phi)$, since Gaussian tails control the cutoff derivative. Hence $B_e$ is already in the gradient closure and its projection does not change it. The Gaussian characteristic function yields
$\E\sin^2(kZ)=(1-e^{-2k^2})/2$.
Moreover,
\[
 \int_0^1r_t^2\,dt=\int_0^1\{t^2+(1-t)^2\}\,dt=\frac23.
\]
These identities prove \eqref{eq:no-gain-ratio}.

For the exact perturbed ODE, write $\widetilde X_t=r_t Y_t$. The equation becomes
\[
 \dot Y_t=\varepsilon\sin(kY_t),\qquad Y_0=Z.
\]
Its flow is increasing in the initial state. The induced monotone coupling in one dimension is optimal for squared transport cost, and $r_1=1$. Also
\[
 |Y_t-Z|\leq\varepsilon t,\qquad
 |Y_1-Z-\varepsilon\sin(kZ)|
 \leq\varepsilon k\int_0^1|Y_t-Z|\,dt\leq\varepsilon^2k/2.
\]
The reverse triangle inequality in $L^2$ gives the stated Wasserstein lower bound.

For $\varepsilon=ck^{-r}$, derivatives of $e_t$ of spatial order $j\leq r$ are bounded by $C_r\varepsilon k^j$, uniformly in $t$ because $r_t\geq1/\sqrt2$. These bounds are uniform in $k$ for $j\leq r$.
Its temporal derivative has magnitude at most $C\varepsilon(1+k|x|)$, and hence at most $C'(1+|x|)$ uniformly in $k$. The temporal derivative of the true velocity also satisfies a bound of the latter form. For sufficiently small $c$, the lower bound is a positive constant times $\varepsilon$. This shows why spatial frequency alone gives no universal gain for the adjoint map on arbitrary errors.
\end{proof}

\subsection{Finiteness of the mixture-weighted negative norm}
\label{subsec:clean-mixture}
For this subsection, choose a sufficiently large constant $M_0$
depending only on the fixed target class, before observing the data,
and restrict the finite sieve to
\begin{equation}\label{eq:gcr-main-anchor}
 \mathcal G^0_{A,\eta}
 =\{g\in\mathcal G_{A,\eta}:
        \sup_{0\leq t\leq1}\|g(t,0)\|_2\leq M_0\}.
\end{equation}
The uniform approximation and origin bound in
\eqref{eq:physical-envelope} show that an approximant exists.
Thus the approximation error and architecture bounds are unchanged,
and the entropy cannot increase. Theorem \ref{thm:gcr-rate} also
holds for empirical minimization over this subclass, but the main
estimator uses $\mathcal G_{A,\eta}$ without the additional restriction.

The origin restriction supplies a common bound on the image of the
origin under each learned flow. Together with spatial stability,
this ensures that the distribution error has finite mixture-weighted
negative norm.
Condition on the training data and write $p=p_1$, $q=\tilde p_1$, and
$\mu=(p+q)/2$, where $p$ is any target in the fixed GCR class and $q$ is the exact
ODE endpoint of any member of this anchored sieve, not its Euler
approximation. The weighted negative
norm with reference $\mu$ is defined as in
\eqref{eq:weighted-negative-norm}, replacing $p_1$ by $\mu$.

\begin{theorem}[Uniform Poincar\'{e} inequality for the mixture distribution]
\label{thm:anchored-mixture}
There are constants $L_F\geq1$ and $M_F<\infty$, independent of $A,\eta,n$
and of the selected sieve member, such that the true and learned endpoint
maps $F$ and $\widehat F$ satisfy
\[
 \operatorname{Lip}(F),\operatorname{Lip}(F^{-1}),
 \operatorname{Lip}(\widehat F),\operatorname{Lip}(\widehat F^{-1})
 \leq L_F,\qquad
 \|F(0)\|_2,\|\widehat F(0)\|_2\leq M_F.
\]
Let $V_d$ be the volume of the unit ball and set
\[
 m_*=(2\pi)^{-d/2}L_F^{-d}
       \exp\{-L_F^2(1+M_F)^2/2\},\quad
 D_*=(V_dm_*)^{-1},\quad C_{\mathrm{mix}}=(1+D_*)L_F^2.
\]
Then the mixture has the uniform Poincar\'{e} inequality
\begin{equation}\label{eq:mixture-poincare}
 \operatorname{Var}_{\mu}(g)
 \leq C_{\mathrm{mix}}\int\|\nabla g\|_2^2\,d\mu,
 \qquad g\in C_c^\infty(\sR^d),
\end{equation}
and
\begin{align}
 \|q-p\|_{\dot H^{-1}(\mu)}^2
 &\leq 2C_{\mathrm{mix}}\int\frac{(q-p)^2}{p+q}\,dx
 \leq4C_{\mathrm{mix}},\label{eq:finite-mixture-norm}\\
 W_2(q,p)&\leq2\|q-p\|_{\dot H^{-1}(\mu)}.\notag
\end{align}
\end{theorem}
\begin{proof}[Proof of Theorem \ref{thm:anchored-mixture}]
The true velocity and all candidate velocities in the anchored sieve have a
common spatial Lipschitz bound $L_v$ and an origin bound $M_v$.
In particular $\|v(t,x)\|_2\leq M_v+L_v\|x\|_2$.
Global ODE existence and the Gronwall inequality, applied also
backwards in time, give bijective endpoint maps with
\[
 \operatorname{Lip}(F),\operatorname{Lip}(F^{-1})\leq e^{L_v},
 \qquad \|F(0)\|_2\leq M_ve^{L_v},
\]
and identical inequalities for $\widehat F$. Take these upper
bounds as $L_F$ and $M_F$. The true endpoint map pushes the standard
Gaussian to $p$ by the established flow representation.

For a bi-Lipschitz map $F$, the change-of-variables formula and
almost-everywhere differentiability yield
\[
 p(x)=\varphi_1(F^{-1}(x))\,|\det DF^{-1}(x)|
 \quad\hbox{for almost every }x.
\]
Every singular value of $DF^{-1}$ is at least $L_F^{-1}$
almost everywhere. Also
$\|F^{-1}(x)\|_2\leq L_F\|x-F(0)\|_2$.
It follows that $p\geq m_*$ almost everywhere on the unit ball.
The same is true of $q$. Let $\nu_0$ be the uniform probability
measure on that ball. Then
\begin{equation}\label{eq:common-overlap-measure}
 \nu_0\leq D_*p,\qquad \nu_0\leq D_*q.
\end{equation}
In particular $D_*\geq1$.

Both $p$ and $q$ have Poincar\'{e} constants at most $L_F^2$.
Indeed, applying the Gaussian Poincar\'{e} inequality to $g\circ F$
and the almost-everywhere chain rule gives
\[
 \operatorname{Var}_p(g)
 =\operatorname{Var}_{\gamma_d}(g\circ F)
 \leq L_F^2\int\|\nabla g\|_2^2p.
\]
The Gaussian Poincar\'{e} inequality extends from smooth to Lipschitz
functions by Sobolev approximation. For a smooth function $a$, its
standard proof uses the Ornstein--Uhlenbeck semigroup
\citep{bakry2014analysis}:
\[
 \operatorname{Var}_{\gamma_d}(a)
 =2\int_0^\infty\int\|\nabla P_ta\|_2^2\,d\gamma_d\,dt
 \leq 2\int_0^\infty e^{-2t}\,dt
               \int\|\nabla a\|_2^2\,d\gamma_d.
\]
The displayed identity and inequality follow from Gaussian invariance,
integration by parts, and $\nabla P_ta=e^{-t}P_t\nabla a$.
Approximation justifies the formula for the
bounded Lipschitz compositions in use.

By \eqref{eq:common-overlap-measure},
\[
 |\E_pg-\E_{\nu_0}g|^2
 \leq \int|g-\E_pg|^2\,d\nu_0
 \leq D_*\operatorname{Var}_p(g),
\]
and similarly for $q$. Consequently
$|\E_pg-\E_qg|^2
 \leq2D_*\{\operatorname{Var}_p(g)+\operatorname{Var}_q(g)\}$.
The exact mixture variance identity now gives
\begin{align*}
 \operatorname{Var}_{\mu}(g)
 &=\tfrac12\operatorname{Var}_p(g)
   +\tfrac12\operatorname{Var}_q(g)
   +\tfrac14|\E_pg-\E_qg|^2\\
 &\leq\tfrac{1+D_*}{2}
       \{\operatorname{Var}_p(g)+\operatorname{Var}_q(g)\}\\
 &\leq (1+D_*)L_F^2\int\|\nabla g\|_2^2\,d\mu.
\end{align*}
This proves \eqref{eq:mixture-poincare}, with a constant uniform
over the target class and the selected network.

Write $f=q-p$ and $a=f/\mu$. Then $\int a\,d\mu=0$ and
$|a|\leq2$. For every admissible test function,
\[
 |\langle f,g\rangle|
 =\left|\int a(g-\E_\mu g)\,d\mu\right|
 \leq \left(C_{\mathrm{mix}}\int a^2\,d\mu\right)^{1/2}
           \left(\int\|\nabla g\|_2^2\,d\mu\right)^{1/2}.
\]
Since
$\int a^2d\mu=2\int(q-p)^2/(p+q)\leq4$,
this proves \eqref{eq:finite-mixture-norm}.
The transport comparison is \eqref{eq:mixture-weight-transport},
already proved using the weighted comparison through the midpoint
measure. All these measures have finite second moments by the
linear-growth flow bounds.
\end{proof}

For every finite $n$, the negative norm with mixture weight is finite
uniformly over the fixed target class and the anchored sieve members,
including for the Gaussian example with small variance. The proof uses a common
positive overlap on one fixed ball and stability of flows with a
Gaussian initial law.
It requires neither a uniform positive lower bound on the density ratio between the two endpoint laws
nor a global lower density bound.

Retain the linear operator $\mathcal T$ of
\eqref{eq:adjoint-projected-norm} and set
\[
 \mathfrak A_\mu(e)=\|\mathcal T e\|_{\dot H^{-1}(\mu)},\qquad
 \mathfrak N_\mu(e)=\|q-p-\mathcal T e\|_{\dot H^{-1}(\mu)}.
\]
Theorem \ref{thm:anchored-mixture} and $\mu\geq p/2$ imply
\begin{align}
 \mathfrak A_\mu(e)&\leq\sqrt{2}\,\mathfrak A(e)
      \leq\sqrt{2}e^\Lambda\|e\|_{L^2(P_W)},\notag\\
 \mathfrak N_\mu(e)&<\infty,\qquad
 W_2(q,p)\leq2\{\mathfrak A_\mu(e)+\mathfrak N_\mu(e)\}.
 \label{eq:finite-mixture-projection}
\end{align}
To verify these bounds, take a test function satisfying
$\int\|\nabla g\|_2^2\,d\mu\leq1$. Then
$\int\|\nabla g\|_2^2p\leq2$. Rescaling the test function proves
$\mathfrak A_\mu(e)\leq\sqrt2\,\mathfrak A(e)$.
The bound on $\mathfrak A(e)$ in Theorem \ref{thm:endpoint-projection}
is finite because $e$ has at most linear growth and the true
interpolation has uniformly finite second moments. The triangle inequality with
\eqref{eq:finite-mixture-norm} then makes
$\mathfrak N_\mu(e)$ finite and proves
\eqref{eq:finite-mixture-projection}.

\section{Local asymptotic efficiency on parametric submodels}
\label{sec:consolidated-results}
\label{sec:consolidated-proofs}
This appendix distinguishes information in the observations from
the efficiency of the estimator obtained by minimizing the chosen
empirical risk. We illustrate this distinction on a one-parameter
family of strictly positive densities on $\sR$ at a
fixed frequency. The empirical minimization is over exact
target-induced velocities, rather than over the neural sieve used
in the main theorem.

\subsection{Information equivalence of the observation schemes}
\label{subsec:sampled-objective-information}
Both the target sample and the source sample can be recovered from
each training triplet through
\[
 Y=X_T+(1-T)U,\qquad Z=X_T-TU.
\]
Conversely, the triplets can be generated from the target data using
independent auxiliary times and source samples. Consequently the
observed triplets and the target data augmented with independent
auxiliary randomness encode the same information. An estimator based
on the triplets can be regarded as a randomized estimator based on the
target observations, and an estimator based on target observations can
be applied to the recovered $Y_i$.
Thus lower bounds of distribution estimation that allow randomized
estimators also apply to the triplet observations.

For example, conditioning the squared loss on a target observation
and integrating over the auxiliary time and source would define a
different empirical criterion from the loss studied here, which uses
one independent draw of the auxiliary time and source per target sample.
Reusing several independent auxiliary draws per target
would instead yield clustered observations, not a larger sample
of independent target observations. Establishing a statistical improvement
from such modifications requires accounting for the resulting covariance
structure. It cannot be obtained by replacing $n$ by the total
number of simulated pairs in the present oracle inequality.

The following fixed-frequency submodel quantifies the effect of the empirical risk on asymptotic variance.

\subsection{Empirical minimization over Gaussian perturbation submodels}
Fix an integer $r\geq2$. Let $d=1$,
and let $\phi$ be the standard Gaussian
density. For integers $k\geq1$, define
\[
 a_k(y)=\cos(ky)-e^{-k^2/2},\qquad
 q_{\theta,k}(y)=\phi(y)\{1+\theta a_k(y)\},\qquad
 |\theta|\leq\eta k^{-r},
\]
where $\eta>0$ is a sufficiently small fixed constant.
Let $\widehat\theta_{n,k}$ minimize the original empirical
flow matching loss over this interval, with one training triplet per
target sample, using the exact velocities of
$q_{\theta,k}$ and a fixed measurable tie rule. The true parameter is zero, so the unqualified $v^*$ in this
appendix is the velocity of the reference target $\phi$.
Explicitly, the criterion is
$n^{-1}\sum_{i=1}^n|U_i-v^*(W_i;q_{\theta,k})|^2$, where
$W_i=(T_i,X_{T_i})$, $U_i=Y_i-Z_i$, and
$X_{T_i}=(1-T_i)Z_i+T_iY_i$.
At the true parameter, $Y_i,Z_i$ are independent standard Gaussians
and $T_i$ is independent and uniform on $[0,1]$. The triplets are
independent across $i$.

\begin{proposition}[Asymptotic variance comparison on Gaussian perturbation submodels]
\label{prop:positive-path-noise}
All the displayed densities are centered, strictly positive, and belong
to a common GCR$(r)$ class with constants independent of $k$.
Write $r_t^2=t^2+(1-t)^2$ and
\[
 b_k(t,x)
 =-\frac{k(1-t)}{r_t^2}
       \exp\left\{-\frac{k^2(1-t)^2}{2r_t^2}\right\}
       \sin\left(\frac{ktx}{r_t^2}\right).
\]
Set $J_k=\int_0^1\E_{p_t}b_k^2\,dt$,
$V_k=\int_0^1r_t^{-2}\E_{p_t}b_k^2\,dt$, and
$E_k=\|\phi a_k\|_{\dot H^{-1}(\phi)}^2$.
For each fixed $k$, as $n\to\infty$,
\begin{align}
 \sqrt n\,\widehat\theta_{n,k}
 &\ \Longrightarrow\ N(0,V_k/J_k^2),\notag\\
 nW_2^2(q_{\widehat\theta_{n,k},k},\phi)
 &\ \Longrightarrow\ E_k(V_k/J_k^2)\,\chi_1^2.
 \label{eq:positive-path-clt}
\end{align}
As $k\to\infty$,
\begin{equation}\label{eq:positive-path-asymptotics}
 kJ_k\longrightarrow\frac{\sqrt\pi}{8},\qquad
 \frac{V_k}{J_k}\longrightarrow1,\qquad
 k^2E_k\longrightarrow\frac12.
\end{equation}
In contrast, the estimator obtained by clipping
\[
 \frac{1}{nI_k}\sum_{i=1}^n a_k(Y_i),\qquad
 I_k=\E_\phi a_k(Y)^2=\frac{(1-e^{-k^2})^2}{2},
\]
to the same parameter interval has fixed-$k$ asymptotic parameter
variance $1/I_k$. Thus the ratio of the asymptotic parameter variances
of the flow matching estimator and the moment estimator is
\begin{equation}\label{eq:positive-path-efficiency}
 I_k V_k/J_k^2\sim\frac{4}{\sqrt\pi}\,k.
\end{equation}
\end{proposition}

The order of limits is explicit: first $n\to\infty$ at fixed $k$, then
$k\to\infty$. No uniform growing-$k$ asymptotic statement or convergence
of the corresponding scaled risks is asserted. Even with exact
target-induced velocities, the estimator defined by the original
empirical least-squares loss need not be uniformly efficient across
smooth perturbation directions.

\subsection{Proof of the asymptotic variance comparison}
\begin{proof}[Proof of Proposition \ref{prop:positive-path-noise}]
Write $a=a_k$. Gaussian symmetry gives $\int a\phi=0$ and
$\int ya(y)\phi(y)\,dy=0$. Also $\|a\|_\infty\leq2$.
For $g=\theta a_k$ and $1\leq j\leq r$,
$\|g^{(j)}\|_\infty\leq\eta k^{j-r}\leq\eta$ and
$\|g\|_\infty\leq2\eta$. Choosing $\eta\leq1/8$ ensures positivity
and verifies the conditions of Proposition
\ref{prop:gcr-finite-perturbation} uniformly in $k,\theta$.

At the Gaussian reference, the conditional mean and variance of $Y$
given $X_t=x$ are $tx/r_t^2$ and $(1-t)^2/r_t^2$. Write
\begin{align*}
 a_t(x)&=\exp\{-k^2(1-t)^2/(2r_t^2)\}
             \cos(ktx/r_t^2)-e^{-k^2/2},\\
 b_k(t,x)&=\E[(U-v^*(t,x))a_k(Y)\mid X_t=x].
\end{align*}
Gaussian integration by parts yields the expression for $b_k$ stated
in the proposition.
The conditional density under the perturbed target is the reference
conditional density multiplied by $(1+\theta a_k(Y))/(1+\theta a_t(x))$.
Taking the conditional expectation of $U$ and subtracting its reference
mean therefore gives, for all parameters in the specified interval,
\[
 v^*(t,x;q_{\theta,k})=v^*(t,x)+
                \frac{\theta b_k(t,x)}{1+\theta a_t(x)}.
\]
Consequently,
\begin{align*}
 \partial_\theta v^*(t,x;q_{\theta,k})
 &=\frac{b_k(t,x)}{(1+\theta a_t(x))^2},\\
 \partial_\theta^2 v^*(t,x;q_{\theta,k})
 &=-\frac{2a_t(x)b_k(t,x)}{(1+\theta a_t(x))^3}.
\end{align*}
The parameter derivatives through order three are bounded uniformly over
this compact interval for each fixed $k$.

We next use the classical smooth M-estimation argument
\citep[Chapter 5]{vandervaart1998asymptotic}, verifying its conditions
for each fixed $k$.
Let $\ell_\theta(W,U)=|U-v^*(W;q_{\theta,k})|^2$.
Its population excess risk at the true Gaussian law is exactly
\[
 \E(\ell_\theta-\ell_0)
 =\theta^2\int_0^1\E_{p_t}
            \frac{b_k(t,X_t)^2}{(1+\theta a_t(X_t))^2}\,dt.
\]
It is positive for $\theta\ne0$, since $J_k>0$, and has its unique
minimum at the interior point zero. For fixed $k$, the loss and its
first three parameter derivatives have integrable envelopes, using the
Gaussian moments and the uniformly positive denominator. Continuity
of the empirical risk guarantees a minimizer on the compact interval.
A finite parameter grid and the derivative envelope give a uniform law
of large numbers on this interval. It follows that every empirical
minimizer is consistent and lies in the interior with probability
tending to one. Applying the same uniform argument to the second
derivative of the loss justifies a Taylor expansion, at zero, of
the first-order condition for the empirical minimizer.

With $\zeta=U-v^*(W)$, the score at zero is $-2\zeta b_k(W)$
and the population second derivative is $2J_k$, since
$\E[\zeta\mid W]=0$.
Conditionally on $t$, $\zeta$ is independent of $X_t$ and has variance
$r_t^{-2}$. Hence $\E[\zeta^2 b_k(W)^2]=V_k$.
The scalar central limit theorem and the preceding Taylor expansion give
\[
 \sqrt n\,\widehat\theta_{n,k}
   =\frac{1}{J_k\sqrt n}\sum_{i=1}^n\zeta_i b_k(W_i)+o_{\mathbb P}(1)
   \ \Longrightarrow\ N(0,V_k/J_k^2)
\]
for each fixed $k$.

We justify the local Wasserstein expansion needed for the second limit.
For $f=\phi a_k$, let $N_f=\|f\|_{\dot H^{-1}(\phi)}$.
The Gaussian Poincar\'{e} inequality and boundedness of $a_k$ make $N_f$
finite. In one dimension the flux
$G(x)=\int_{-\infty}^x f(y)\,dy$ vanishes at both infinities and has
energy $\int G^2/\phi=N_f^2$. The path $\phi+u\theta f$, $0\leq u\leq1$,
with flux $-\theta G$, has transport action at most
$\theta^2N_f^2/(1-|\theta|\|a_k\|_\infty)$. Thus
\[
 \limsup_{\theta\to0}\frac{W_2(\phi+\theta f,\phi)}{|\theta|}
 \leq N_f.
\]
For the reverse bound, use any $g\in C_c^\infty$ and an optimal coupling
$(X,Y)$ with first marginal $\phi$ and second marginal $\phi+\theta f$.
Taylor's theorem and Cauchy--Schwarz give
\[
 |\theta\langle f,g\rangle|
 \leq\|g'\|_{L^2(\phi)}W_2(\phi+\theta f,\phi)
       +\tfrac12\|g''\|_\infty W_2^2(\phi+\theta f,\phi).
\]
The upper bound already proved makes the second term $O(\theta^2)$.
Divide by $|\theta|$, take a lower limit, and then the supremum over
test functions satisfying $\|g'\|_{L^2(\phi)}\leq1$ to obtain the
matching lower bound.
Hence $W_2(\phi+\theta f,\phi)=|\theta|\sqrt{E_k}+o(|\theta|)$.
Consistency, the parameter limit, and this deterministic expansion prove
\eqref{eq:positive-path-clt}.

It remains to verify the frequency asymptotics. Gaussian integration gives
\begin{equation}\label{eq:positive-path-integrand}
 \E_{p_t}b_k^2
 =\frac{k^2(1-t)^2}{2r_t^4}
       e^{-k^2(1-t)^2/r_t^2}
       \{1-e^{-2k^2t^2/r_t^2}\}.
\end{equation}
Set $z=k(1-t)/t$. In the transformed integral,
\[
 kJ_k=\frac12\int_0^\infty
 \frac{k^4z^2}{(k^2+z^2)^2}
 e^{-k^2z^2/(k^2+z^2)}
 \{1-e^{-2k^4/(k^2+z^2)}\}\,dz.
\]
On $z\leq k$, the integrand is bounded by
$z^2e^{-z^2/2}/2$ and tends to $z^2e^{-z^2}/2$.
On $z\geq k$, its integral is at most $Ck^3e^{-k^2/2}$.
Dominated convergence gives $kJ_k\to\sqrt\pi/8$.
The extra factor in $V_k$ is
$r_t^{-2}=(k+z)^2/(k^2+z^2)\leq2$, tending to one at fixed $z$.
The same domination proves $V_k/J_k\to1$.

For the negative norm, let $H_m=\operatorname{He}_m/\sqrt{m!}$ be the
normalized probabilists' Hermite polynomials, orthonormal in $L^2(\phi)$.
The normalized Hermite expansion of $a_k$ has
coefficient
$e^{-k^2/2}\operatorname{Re}((ik)^m)/\sqrt{m!}$ for $m\geq1$.
The constant coefficient vanishes. The Gaussian energy identity
$\int(H_m')^2\phi=m$ and orthogonality give
\begin{equation}\label{eq:positive-path-negative-series}
 E_k=e^{-k^2}\sum_{\substack{m\geq2\\m\ {\rm even}}}
          \frac{k^{2m}}{m!\,m}.
\end{equation}
For completeness, take the finite Hermite sums whose coefficients are
those of $a_k$ divided by $m$. Their gradients converge in $L^2(\phi)$,
since the series in \eqref{eq:positive-path-negative-series} is finite.
Applying smooth cutoffs to each polynomial potential shows that its gradient
lies in the $L^2(\phi)$ closure used in the negative norm.
Gaussian integration by parts and
passage to the energy limit show that the negative distributional
derivative of the limiting flux equals $\phi a_k$.
This proves the series identity.

If $M$ is Poisson with mean $\lambda=k^2$, the series is
$\E[\mathbf1_{\{M\geq2,\ M\ {\rm even}\}}/M]$.
On $M\geq\lambda/2$, $\lambda/M\leq2$ and it converges to one in
probability as $\lambda\to\infty$. The contribution of
$2\leq M<\lambda/2$ to $\lambda E_k$ is bounded by
$\lambda\mathbb P(M<\lambda/2)$, which tends to zero exponentially
by the exponential Markov inequality. Finally
$\mathbb P(M\ {\rm even})=(1+e^{-2\lambda})/2$.
These observations prove $k^2E_k\to1/2$.

The value $I_k=(1-e^{-k^2})^2/2$ follows by integrating $\cos^2(kY)$.
At the reference, the centered sample mean of $a_k(Y)$ has variance
$I_k/n$. Its clipped rescaling has the same fixed-$k$ asymptotic law,
since zero is interior to the fixed parameter interval and the clipping
event has probability tending to zero. The parameter variance is
$1/I_k$. Dividing the two variances and using the proven frequency
limits yields \eqref{eq:positive-path-efficiency}.
\end{proof}

\section{Quantitative local analyticity in time}
\label{sec:technical-regularity}
\label{subsec:supp-time-analytic}
This appendix proves Lemma \ref{lm:time-analytic}, which supplies the
local time analyticity used in the neural sieve construction.
The target-to-velocity transfer and its endpoint argument remain in
Section \ref{sec:gcr-main-regularity} of the main paper.

\begin{proof}[Proof of Lemma \ref{lm:time-analytic}]
Fix a real time $0\leq t<1$, and let $s=1-t$, $m=\nabla_bK(\lambda_t,b_{t,x})$, and $\xi=Y-m$ under $\Pi_{\lambda_t,b_{t,x}}$. The covariance premise of the lemma, uniformly over all real tilts $b$, gives
\[
 \nabla_b^2 K(\lambda_t,b)\preceq C\sigma_t^2 I_d,\qquad
 \sigma_t=\begin{cases}1,&0\leq t\leq1/2,\\s,&1/2<t<1.\end{cases}
\]
For $t>1/2$, this follows from
$(1+\lambda_t)^{-1}\asymp s^2$. For $0\leq t\leq1/2$,
Lemma \ref{lm:gcr-low-precision} also gives the mean bound used below. Integrating the Hessian along the segment from $b_{t,x}$ to $b_{t,x}+u$ proves
\begin{equation}\label{eq:analytic-centered-mgf}
 \E e^{u^\top\xi}
 =\exp\{K(\lambda_t,b_{t,x}+u)-K(\lambda_t,b_{t,x})-u^\top m\}
 \leq e^{C\sigma_t^2\|u\|_2^2/2}
 \quad(u\in\sR^d).
\end{equation}
The required real partition functions are finite by the assumed finiteness of the real moment generating function. The low-precision covariance bound also implies the centered target tail bound. Chernoff bounds for the coordinates and a union bound imply
\[
 \mathbb P(\|\xi\|_2/\sigma_t>r)\leq C_de^{-c_dr^2},
 \qquad
 \E\exp\{c_*\|\xi\|_2^2/\sigma_t^2\}\leq C_*,
\]
for fixed sufficiently small $c_*>0$. In particular, for every fixed
$j\geq0$ and $0<c<c_*$, the random variables
$(\|\xi\|_2/\sigma_t)^j\exp\{c\|\xi\|_2^2/\sigma_t^2\}$
are uniformly integrable.

For complex $w$ near $t$, define
\[
 \lambda(w)=\frac{w^2}{(1-w)^2},\quad
 b(w)=\frac{wx}{(1-w)^2},\quad
 q=\lambda(w)-\lambda_t,\quad
 a=b(w)-b_{t,x}-qm.
\]
After factoring out $\exp\{(b(w)-b_{t,x})^\top m-q\|m\|_2^2/2\}$, the new partition integral relative to the real posterior has normalized denominator
\begin{equation}\label{eq:analytic-denominator}
 D(w)=\E e^{a^\top\xi-q\|\xi\|_2^2/2}.
\end{equation}
For $t\leq1/2$ and $|w-t|\leq c_0$, differentiating the rational
functions $\lambda(w)$ and $b(w)$ on a fixed neighborhood bounded
away from $1$ gives
\[
 |q|\leq C|w-t|,\quad
 \|b(w)-b_{t,x}\|_2\leq CA|w-t|,\quad
 \|m\|_2\leq CA.
\]
Thus $\sigma_t\|a\|_2+\sigma_t^2|q|\leq CA|w-t|$.

For $t>1/2$ and $|w-t|\leq s/4$, write
$m=x/t+r$, where the posterior-mean premise gives $\|r\|_2\leq Cs^2A$.
The exact cancellation
\[
 b(w)-b_{t,x}-(\lambda(w)-\lambda_t)x/t
 =-\frac{w(w-t)}{t(1-w)^2}x
\]
and the rational derivative bound $|q|\leq C|w-t|/s^3$ show that
\begin{align*}
 s\|a\|_2
 &\leq CA|w-t|/s+Cs|q|\|r\|_2
 \leq CA|w-t|/s,\\
 s^2|q|&\leq C|w-t|/s.
\end{align*}
Consequently, in both regimes, if $|w-t|<c s/A$ for sufficiently small fixed $c$,
\[
 \sigma_t\|a\|_2+\sigma_t^2|q|\leq Cc.
\]
With $Z=\|\xi\|_2/\sigma_t$, the absolute value of the exponent in \eqref{eq:analytic-denominator} is at most $Cc(Z+Z^2)$. The bound $|e^z-1|\leq |z|e^{|z|}$ and the exponential-square moment above imply
\[
 |D(w)-1|
 \leq Cc\,\E[(Z+Z^2)e^{Cc(Z+Z^2)}]\leq C'c\leq1/2.
\]
The expectation is uniformly bounded when $c$ is small enough. This also provides integrable domination on every compact subset of the disk. The integrands are holomorphic there. Dominated integration (or Morera's theorem followed by Fubini) proves that the denominator and each centered numerator are holomorphic.

The continued posterior mean is therefore
\[
 m_w=m+\frac{\E[\xi e^{a^\top\xi-q\|\xi\|_2^2/2}]}{D(w)},
 \qquad \|m_w-m\|_2\leq C\sigma_t.
\]
For real $w$ in the disk with $0\leq w<1$, this ratio equals the original conditional mean by the change-of-tilt identity. Hence
$v^*(w,x)=(m_w-x)/(1-w)$ is its holomorphic continuation.
For $0\leq t\leq1/2$, its numerator is bounded by $CA$ and $|1-w|$ is bounded below. For $1/2<t<1$,
\[
 \|m-x\|_2=\|(s/t)x+r\|_2\leq CsA,\quad
 \|m_w-m\|_2\leq Cs,\quad |1-w|\geq3s/4.
\]
Both give $\|v^*(w,x)\|_2\leq CA$. Cauchy's formula on smaller concentric circles gives \eqref{eq:analytic-cauchy}, after adjusting the fixed constant. This includes a disk at $t=0$.
\end{proof}

\section{Verification of Gaussian-channel regularity}
\label{sec:concrete-target-proofs}
This appendix proves the four propositions in Section
\ref{subsec:concrete-target-classes} in their order of appearance.
The verification of the bounded perturbations of a strongly convex
potential is included with Proposition \ref{prop:gcr-ordinary}.
The last proof also supplies the GCR membership result used in the
lower-bound and local-efficiency arguments.

\subsection{Strongly log-concave targets with \texorpdfstring{$C^2$}{C2} potentials}
\label{subsec:supp-c2}
\begin{proof}[Proof of Proposition \ref{prop:gcr-c2}]
We give the moment argument to avoid differentiating $U$ a third time.
If a potential $V$ has
$cI_d\preceq D^2V\preceq CI_d$ and minimizer $y_*$, the quadratic
estimates about $y_*$ bound the normalizing integral from below by
the corresponding Gaussian integral of precision $C$, and the moment
integrals from above by those of precision $c$. Write $\E_V$ for expectation under
the probability density proportional to $e^{-V}$. Consequently, for every fixed $j$,
\[
 \E_V\|Y-y_*\|_2^j
 \leq (C/c)^{d/2}c^{-j/2}\E\|G\|_2^j.
\]
The same order holds for moments centered at $\E_VY$.

For the low-precision posterior the potential is
$U(y)+\lambda\|y\|_2^2/2-b^\top y$, whose Hessian lies between
$\kappa+\lambda$ and $\Lambda+\lambda$. Its centered covariance is
uniformly bounded, independently of $b$ and $\lambda\in[0,1]$.
The strong convexity also makes all real target moment generating
functions finite. For the target itself, the assumed mean bound and the
moment estimate about its mode yield uniform bounds on the norm of its
mode and on $\E\|Y\|_2^2$, with constants
depending only on $d,\kappa,\Lambda,M_\mu$.

For the terminal observation $z=Y+hG$, let $\pi_{h,z}$ be the posterior.
Its Hessian bounds are $\kappa+h^{-2}$ and $\Lambda+h^{-2}$, with
uniformly bounded ratio. Hence for all fixed $j$,
\begin{equation}\label{eq:gcr-c2-moments}
 \E_{\pi_{h,z}}\|Y-\E_{\pi_{h,z}}Y\|_2^j\leq C_jh^j.
\end{equation}
Let $B=\nabla U(Y)$ and $A=D^2U(Y)$. The Lipschitz bound on $\nabla U$
also gives
$\E_{\pi_{h,z}}\|B-\E_{\pi_{h,z}}B\|_2^j\leq C_jh^j$.
All these bounds are uniform in $z$.

Differentiation of the translated convolution, using only the first
two derivatives of $U$, gives
\begin{equation}\label{eq:gcr-c2-hessian}
 \nabla\ell_h(z)=-\E_{\pi_{h,z}}B,\qquad
 D^2\ell_h(z)=-\E_{\pi_{h,z}}A
                     +\operatorname{Cov}_{\pi_{h,z}}(B).
\end{equation}
Its Hessian norm is at most $\Lambda+Ch^2$.
For each $F(Y)\in\{A_{ij},B_i,B_j,B_iB_j\}$,
\[
 \partial_{z_k}\E_{\pi_{h,z}}F(Y)
       =h^{-2}\operatorname{Cov}_{\pi_{h,z}}(F(Y),Y_k).
\]
These functions are independent of $z$ and have at most quadratic growth,
so the Gaussian tail bounds for the posterior justify differentiation
locally under the integral. Applying this identity
to \eqref{eq:gcr-c2-hessian},
rather than differentiating $A$ as a function of $y$, gives
\begin{equation}\label{eq:gcr-c2-third}
 \partial_{ijk}\ell_h(z)
 =h^{-2}\left\{
 -\operatorname{Cov}_{\pi_{h,z}}(A_{ij},Y_k)
 +\operatorname{Cum}_{\pi_{h,z}}(B_i,B_j,Y_k)\right\}.
\end{equation}
The first covariance is at most $Ch$ since $A$ is bounded, and the
third cumulant is at most $Ch^3$ by
\eqref{eq:gcr-c2-moments} and H\"{o}lder's inequality.
Thus $\|D^3\ell_h\|_{\mathrm{op}}\leq C(h^{-1}+h)\leq C/h$.
This argument applies to $C^2$ potentials. No distributional or
classical third derivative of $U$ has been assumed.

The uniform Hessian bound, $\int\nabla q_h=0$, and the uniform second
moment of $Y+hG$ give
\[
 \|\nabla\ell_h(0)\|_2
 \leq \|D^2\ell_h\|_\infty\,\E\|Y+hG\|_2\leq C.
\]
Hence
$\|\nabla\ell_h(x)\|_2\leq C(1+\|x\|_2)$.
Together these facts verify every part of GCR$(2)$.

\end{proof}

\subsection{Log-density smoothness and posterior concentration}
\label{subsec:supp-ordinary}
\begin{proof}[Proof of Proposition \ref{prop:gcr-ordinary}]
Write $g=\log p$ and let $\nu_{h,z}$ be the posterior law of $Y$
given $Y+hZ=z$. It equals $\Pi_{h^{-2},z/h^2}$.
For any fixed $\lambda\geq0$, the Hessian of $K(\lambda,b)$ in $b$
is the posterior covariance. Integrating
\eqref{eq:gcr-ordinary-concentration} along a real tilt segment gives
\[
 \log\E_{\Pi_{\lambda,b}}
 e^{u^\top(Y-\E_{\Pi_{\lambda,b}}Y)}
 \leq \frac{C_{\mathrm{post}}\|u\|_2^2}{2(1+\lambda)}.
\]
The finite real moment generating function justifies these derivatives
as in Lemma \ref{lm:gcr-low-precision}. Consequently, for every fixed
$q\geq1$, uniformly in $z$ and $0<h\leq1$,
\begin{equation}\label{eq:gcr-ordinary-moments}
 \E_{\nu_{h,z}}\|Y-\E_{\nu_{h,z}}Y\|_2^q\leq C_qh^q.
\end{equation}
At $(\lambda,b)=(0,0)$ the same bound and $\|\E Y\|_2\leq M_\mu$ give
uniform target moments, in particular a uniform second moment.

Let $e_1,\ldots,e_j$ be unit directions. For a nonempty block
$B\subseteq\{1,\ldots,j\}$, write $D_{e_B}g$ for one directional
derivative in each direction indexed by $B$. Differentiation of the
translated convolution gives
\begin{equation}\label{eq:gcr-ordinary-partitions}
 D_{e_1}\cdots D_{e_j}\ell_h(z)
 =\sum_{\mathcal A\in\operatorname{Part}(j)}
   \operatorname{Cum}_{\nu_{h,z}}
       \bigl(D_{e_B}g(Y):B\in\mathcal A\bigr),
 \qquad 1\leq j\leq r.
\end{equation}
Here $\operatorname{Part}(j)$ is the set of partitions of
$\{1,\ldots,j\}$, and a one-argument cumulant is its expectation.
For example, the identity follows by differentiating the logarithm of
\[
 \frac{q_h(z+v)}{q_h(z)}
 =\E_{\nu_{h,z}}\exp\{g(Y+v)-g(Y)\}
\]
at $v=0$. The bounded Hessian of $g=\log p$ makes the score grow
at most linearly for each density. Taylor's bound for $g(Y+v)-g(Y)$,
the finite exponential moments, and the derivative bounds through
order $r$ justify these derivatives locally in $v$.

For $2\leq j\leq r$, the one-block term in
\eqref{eq:gcr-ordinary-partitions} is bounded by $L_j$.
Every other cumulant has at least two arguments and is invariant
under adding a constant to any argument. A score component is
$L_2$-Lipschitz, so \eqref{eq:gcr-ordinary-moments} and Jensen's
inequality give, for every fixed $q$,
\[
 \E_{\nu_{h,z}}
 |D_eg(Y)-\E_{\nu_{h,z}}D_eg(Y)|^q\leq C_qh^q.
\]
Components of $D^ag$ for $2\leq a\leq r$ are uniformly bounded.
The finite moment-cumulant formula and H\"older's inequality therefore
bound every term in \eqref{eq:gcr-ordinary-partitions} uniformly.
This proves $\|D^j\ell_h\|_\infty\leq C$ for $2\leq j\leq r$.

For the last derivative, keep the functions of $Y$ fixed and
differentiate the posterior Gaussian likelihood. For functions
$f_1,\ldots,f_q$ independent of $z$ and with the required moments,
\begin{equation}\label{eq:gcr-ordinary-cumulant-derivative}
 D_e\operatorname{Cum}_{\nu_{h,z}}(f_1(Y),\ldots,f_q(Y))
 =h^{-2}\operatorname{Cum}_{\nu_{h,z}}
       (f_1(Y),\ldots,f_q(Y),e^\top Y).
\end{equation}
For $q=1$ this is the posterior covariance differentiation identity.
For general $q$ it follows by differentiating the finite
moment-cumulant formula. Local domination follows from polynomial
growth of the score and the posterior exponential moments, so it
does not require differentiating the $f_i$.
Apply \eqref{eq:gcr-ordinary-cumulant-derivative} to
\eqref{eq:gcr-ordinary-partitions} at $j=r$.
After centering all arguments, the newly inserted argument
$e^\top(Y-\E_{\nu_{h,z}}Y)$ has $L^q(\nu_{h,z})$ norm $O(h)$
for every fixed $q\geq1$. The other arguments have uniformly bounded
$L^q(\nu_{h,z})$ norms, so each new cumulant is $O(h)$. Multiplication by
$h^{-2}$ yields
\[
 \|D^{r+1}\ell_h\|_\infty\leq C/h.
\]
Only derivatives of $g$ through order $r$ have been used.

It remains to check score growth with a constant uniform over the class.
The already established Hessian bound makes $\nabla\ell_h$
uniformly Lipschitz. Since $q_h$ is a Gaussian convolution,
$\int\nabla q_h=0$. Its gradient is integrable for each $h>0$.
The uniform target moment bound then gives
\begin{align*}
 \|\nabla\ell_h(0)\|_2
 &=\left\|\int\{\nabla\ell_h(0)-\nabla\ell_h(z)\}q_h(z)\,dz\right\|_2\\
 &\leq C\E\|Y+hZ\|_2\leq C.
\end{align*}
Thus $\|\nabla\ell_h(z)\|_2\leq C(1+\|z\|_2)$.
The concentration premise supplies \eqref{eq:gcr-low} directly,
and multiplication of the unweighted derivative bounds by $h\leq1$
supplies the remaining weighted GCR bounds. This proves the proposition.
\end{proof}

\noindent\textit{Verification of the bounded-perturbation family
\eqref{eq:gcr-ordinary-family}.}
For fixed $(\lambda,b)$, the quadratic tilt of the density proportional
to $e^{-U}$ has potential
\[
 V_{\lambda,b}(y)=U(y)+\lambda\|y\|_2^2/2-b^\top y,
 \qquad
 (\kappa+\lambda)I_d\preceq D^2V_{\lambda,b}
 \preceq(\Lambda+\lambda)I_d.
\]
Here $\E_{\mathrm{core}}$ denotes expectation under the density proportional
to $e^{-V_{\lambda,b}}$. Let $y_*$ be the unique minimizer of this potential.
The quadratic comparison used in
Appendix \ref{subsec:supp-c2} gives, for every fixed $q\geq1$,
\[
 \E_{\mathrm{core}}\|Y-y_*\|_2^q
 \leq
 \left(\frac{\Lambda+\lambda}{\kappa+\lambda}\right)^{d/2}
 (\kappa+\lambda)^{-q/2}\E\|G\|_2^q.
\]
Reweighting by $a$ increases the expectation of any nonnegative
function by at most $a_+/a_-$. Variance in any unit direction is
bounded by the second moment about $y_*$, so the perturbed posterior
satisfies \eqref{eq:gcr-ordinary-concentration}, with a constant
depending only on the specified bounds and dimension.
At $(\lambda,b)=(0,0)$, strong convexity gives
$\|y_*\|_2\leq M_0/\kappa$. The same moment bound supplies a uniform
bound on the mean. The quadratic lower bound on $U$, together with
$a\leq a_+$, also makes every real exponential moment finite.
Finally, $a\geq a_->0$ and its derivative bounds control the
derivatives of $\log a$ through order $r$. Hence
$\log p=-U+\log a-\log\mathcal Z$ satisfies
\eqref{eq:gcr-ordinary-smoothness}. This verifies every premise
of Proposition \ref{prop:gcr-ordinary} for the stated family.

\subsection{Gaussian convolution targets}
\label{subsec:supp-gaussian-latent}
\begin{proof}[Proof of Proposition \ref{prop:gcr-gaussian-latent}]
Write $Y=Q+\sigma G$ as in the proposition. Its density is positive,
its mean has norm at most $R$, and its real moment generating function is finite
everywhere because $Q$ is bounded.
Under a tilt $(\lambda,b)$, completing the square in the conditional
Gaussian density gives
\[
 Y=a_\lambda\widetilde Q+c_\lambda b+s_\lambda G',
 \quad
 a_\lambda=(1+\lambda\sigma^2)^{-1},\quad
 c_\lambda=\frac{\sigma^2}{1+\lambda\sigma^2},\quad
 s_\lambda^2=\frac{\sigma^2}{1+\lambda\sigma^2}.
\]
Here $\widetilde Q$ has the tilted mixing law, still supported on the
ball of radius $R$, and $G'\sim\gamma_d$ is independent of it.
Consequently
\[
 \operatorname{Cov}_{\Pi_{\lambda,b}}(Y)
 =a_\lambda^2\operatorname{Cov}(\widetilde Q)+s_\lambda^2I_d
 \preceq(R^2+\sigma^2)I_d,\qquad 0\leq\lambda\leq1.
\]
This verifies \eqref{eq:gcr-low} uniformly in the real tilt.

For $0<h\leq1$, let $v=\sigma^2+h^2$. Up to an additive constant,
\[
 \ell_h(x)=-\frac{\|x\|_2^2}{2v}
 +\log\E\exp\{x^\top Q/v-\|Q\|_2^2/(2v)\}.
\]
Let $\E_x$, $\operatorname{Cov}_x$, and $\operatorname{Cum}_x$ refer
to the exponentially tilted law of $Q$ in the preceding display.
Differentiation, justified by bounded
support, gives
\[
 \nabla\ell_h=-x/v+\E_xQ/v,\qquad
 D^2\ell_h=-I_d/v+\operatorname{Cov}_x(Q)/v^2,
\]
and, for $m\geq3$,
$D^m\ell_h=v^{-m}\operatorname{Cum}_x(Q,\ldots,Q)$,
where the cumulant with $m$ arguments is the $m$th cumulant tensor of $Q$.
Every fixed-order cumulant is uniformly bounded by the finite
moment-cumulant formula and $\|Q\|_2\leq R$. Since $v\geq\sigma^2$,
the score has uniform linear growth and all derivatives of order at
least two, through any fixed order, are bounded uniformly in $x,h$.
Multiplication by $h\leq1$ proves \eqref{eq:gcr-weighted}.
The constants depend only on the quantities stated in the proposition.
\end{proof}

\subsection{Finite-order Gaussian perturbations}
\label{subsec:supp-finite-perturbation}
\begin{proof}[Proof of Proposition \ref{prop:gcr-finite-perturbation}]
The bounds $\varphi_1/2\leq p\leq3\varphi_1/2$ imply finiteness
of all real moment generating functions. Under a quadratic tilt,
the reference Gaussian law is
\[
 \nu_{\lambda,b}=N\bigl(b/(1+\lambda),(1+\lambda)^{-1}I_d\bigr).
\]
The density of $\Pi_{\lambda,b}$ relative to this law is
$(1+g)/\E_{\nu_{\lambda,b}}(1+g)$, and is at most three.
Writing $m_\nu=b/(1+\lambda)$, for every unit vector $u$,
\[
 \operatorname{Var}_{\Pi_{\lambda,b}}(u^\top Y)
 \leq\E_{\Pi_{\lambda,b}}\{u^\top(Y-m_\nu)\}^2
 \leq\frac{3}{1+\lambda}.
\]
In particular this proves the low-precision GCR condition uniformly
in $b$.

For $0<h\leq1$, completing the square in the convolution gives
\begin{align*}
 q_h(x)&=\varphi_{1+h^2}(x)B_h(x),
 & B_h(x)&=1+\E g(a_hx+\tau_h Z),\\
 a_h&=(1+h^2)^{-1},
 & \tau_h&=\frac{h}{\sqrt{1+h^2}},
\end{align*}
where $Z\sim\gamma_d$ and $1/2\leq B_h\leq3/2$.
For $1\leq j\leq r$, differentiation under the expectation yields
$D^jB_h=a_h^j\E D^jg(a_hx+\tau_h Z)$, with norm at most $G$.
For unit directions $u_1,\ldots,u_{r+1}$, one additional derivative
can instead be applied to the Gaussian kernel:
\begin{align*}
 D^{r+1}B_h(x)[u_1,\ldots,u_{r+1}]
 =\frac{a_h^{r+1}}{\tau_h}\E\bigl[
 D^rg(a_hx+\tau_h Z)[u_1,\ldots,u_r]
 \langle Z,u_{r+1}\rangle\bigr].
\end{align*}
This identity follows by differentiating the convolution of the
bounded continuous tensor $D^rg$. The differentiated Gaussian kernel
is integrable. Hence $\|D^{r+1}B_h\|_\infty\leq CG/h$ without
differentiating $g$ an additional time.

The finite-order chain rule for $\log B_h$ is a sum of products
of derivatives of $B_h$, divided by powers of $B_h$.
All derivatives through order $r$ are uniformly bounded. At order
$r+1$, a product either contains the single factor $D^{r+1}B_h$
or only lower-order factors. The denominator is bounded away from
zero, so
\[
 \|D^j\log B_h\|_\infty\leq C\quad(1\leq j\leq r),\qquad
 \|D^{r+1}\log B_h\|_\infty\leq C/h.
\]
Since $\ell_h=\log\varphi_{1+h^2}+\log B_h$, its score is bounded
by $C(1+\|x\|_2)$, its derivatives of orders two through $r$ are
uniformly bounded, and $h\|D^{r+1}\ell_h\|_\infty\leq C$.
These are precisely the remaining GCR requirements.
All constants above are independent of the particular $g$.
\end{proof}

\section{Detailed comparison with related work}
\label{sec:detailed-related-work}
This appendix expands Section \ref{sec:related-work} of the main
paper. It records the assumptions and theorem-level distinctions
underlying the concise comparison there.
We distinguish three kinds of guarantees: regularity of a
population velocity, stability of a sampler given a velocity approximation error,
and statistical error for an empirical estimator.

\subsection{Regularity of velocities}
Flow matching and stochastic interpolants construct velocities
through a regression formulation
\citep{lipman2023flow,liu2023flow,albergo2023building,albergo2023stochastic,huang2023conditional}.
The Gaussian interpolation framework of \citet{gao2024gaussian}
provides a natural setting for these conditional formulas.
Our use of posterior means, covariances, and smoothed scores belongs
to this line of analysis.

\citet{benton2024error} study velocity regularity and propagation
of velocity approximation error. Their Theorem~1 bounds the $W_2$
error of the exact learned ODE using an integrated $L^2$ velocity
error and Lipschitz control of the learned velocity. More recently,
\citet{stephanovitch2026lipschitz} studies spatial and temporal
regularity, sampling error, and functional inequalities using
tilted posterior laws. That work motivates using posterior
concentration bounds in the regularity analysis. Here we verify the
required estimates directly under GCR.

For statistical approximation, we use uniform bounds on
$hD^m\ell_h$ through order $r+1$ and quantitative local analyticity
in time obtained from the posterior representation.
A first-order Lipschitz estimate alone would not yield the
coefficient count used in the main theorem.

\subsection{Statistical estimation and propagation of velocity error}
A bound expressed in terms of an assumed velocity error does not
yet specify how that error depends on the sample size.
We derive this dependence using an explicit finite sieve and the original
empirical squared loss. The stability argument uses the regularity
of the learned velocity, not only that of the population velocity.
For comparison, likelihood-based analyses of continuous normalizing
flows, including \citet{marzouk2023distribution}, concern a different
criterion and empirical process.

Diffusion sampling theory also separates score estimation from
error propagation. \citet{li2024nonasymptotic} obtain discrete-time
guarantees for probability-flow ODE and DDPM-type samplers, with
leading step-dependent terms of order $T^{-1}$ and $T^{-1/2}$,
respectively, for fixed dimension and up to logarithmic factors.
Here $T$ is the number of sampling steps, not the sample size $n$.
For the ODE sampler, \citet[Theorem 1]{li2024sharp} sharpen the
dimension dependence to a term of order $d/T$, up to logarithms,
when $T\gtrsim d^2\log^5T$. Their theorem assumes absolute continuity
and a support-radius bound for the target, together with average
score and Jacobian estimation-error bounds, and compares the output
with a slightly noised target.
\citet{li2024accelerating} further develop higher-order samplers
with step-dependent terms of order $T^{-2}$ and $T^{-1}$ for their
deterministic and stochastic constructions, respectively.

\subsection{Statistical guarantees for diffusion models}
Probability-flow ODEs and diffusion models give related score-based
constructions \citep{song2021scorebased,song2021denoising}.
The following comparisons suppress logarithmic factors and keep
dimension and class constants fixed.
\citet{oko2023diffusion} consider Besov densities of smoothness $s$
supported on a cube and bounded above and below there, with additional
smoothness near the boundary. Under the Besov parameter restrictions
in their Section~2 and Assumptions~2.4--2.6, their Theorems~5.1 and~5.4
give expected total variation (TV) error of order $n^{-s/(2s+d)}$ and,
for $d\geq3$, expected $W_1$ error of order
$n^{-(s+1-\delta)/(2s+d)}$ for arbitrarily small fixed $\delta>0$.
These theorem statements use an empirical criterion that integrates
over the auxiliary randomness, and the $W_1$ result uses a
time-partitioned construction.
Both results stop the reverse diffusion at positive forward noise.
Their Section~4.1 and Theorem~C.6 also analyze finite Monte Carlo
approximations of the empirical loss.

For densities supported on $[-1,1]^d$, H\"older smooth of order
$\alpha>0$ on that cube, and bounded above and away from zero there,
\citet[Theorems 8 and 12]{dou2024optimal} obtain the minimax expected
TV rate $n^{-\alpha/(2\alpha+d)}$ without additional logarithmic factors.
They choose score estimators according to the noise level and use them
in a continuous reverse SDE, followed by truncation to the support.
Their sharp score estimation bounds at small noise levels allow the
reverse process to reach zero forward noise without early stopping.

Compact support and a density lower bound are not necessary for
all such guarantees. \citet[Theorem 3.8]{zhang2024minimax} obtain
expected TV error of order $n^{-\beta/(2\beta+d)}$ for sub-Gaussian
densities in their Sobolev class of order $0<\beta\leq2$, without
a density lower bound. They combine their kernel-based score estimator
with continuous reverse diffusion stopped at the noise variance
$t_0=n^{-2/(2\beta+d)}$.
The Sobolev assumption controls the resulting smoothing bias.
Sub-Gaussian tails alone suffice for their time-dependent score
estimation bound, but not for this rate to the unsmoothed target.

For deterministic sampling, \citet[Theorem 1]{cai2025minimax}
obtain expected TV error of order $n^{-\beta/(2\beta+d)}$
for sub-Gaussian densities in their H\"older class with
$0<\beta\leq2$, without a density lower bound.
Their score estimator uses Gaussian kernels with smooth
regularization, which allows control of both the $L^2$ score
error and the mean Jacobian error.
The guarantee includes discretization of the probability-flow
ODE, with a suitable polynomial number of sampling steps.
The construction uses Gaussian smoothing at variance
$n^{-2/(2\beta+d)}$, stops at positive forward noise, and rescales
the output.

For neural score estimation, \citet[Theorem I]{stephanovitch2025generalization}
obtain expected $W_1$ error of order $n^{-(\beta+1)/(2\beta+d)}$
for $\beta>0$ and $d\geq3$. Their Assumption 1 requires
$p=\exp(-u+a)$, with $u$ strongly convex, $a$ in a bounded
H\"older class of order $\beta$, and local polynomial bounds on
derivatives of $u$ that control density decay.
They minimize a denoising score matching loss that integrates both
time and Gaussian noise, using networks with tanh activation on
separate time intervals and a spatial one-sided Lipschitz constraint.
Their theorem applies to continuous reverse SDEs and probability-flow
ODEs with early stopping and does not include numerical integration
error.

A distinct comparison concerns estimating the score itself.
For sub-Gaussian densities with Lipschitz scores,
\citet{wibisono2024optimal} establish the minimax squared
$L^2(p)$ score risk $n^{-2/(d+4)}$, up to logarithmic factors.
Under the additional uniform Lipschitz condition on scores along
the Ornstein--Uhlenbeck process, their Corollary~6 yields a discretized
diffusion sampler with expected TV error of order $n^{-1/(d+4)}$.
Other improvements exploit structure beyond smoothness:
\citet{chen2023score} study targets on an unknown low-dimensional
linear subspace, while \citet{cole2024score} exploit controlled
neural approximation of a Gaussian log-density perturbation,
together with tail conditions. Neither low-dimensional support nor
such approximation complexity is a consequence of GCR alone.

For nonlinear manifold structure,
\citet[Theorem 2]{tang2024adaptivity} establish a forward--backward
diffusion estimator with expected $W_1$ error of order
$n^{-1/2}\vee n^{-(\alpha+1)/(2\alpha+k)}$, up to logarithmic
factors. Here the target is supported on a compact, boundaryless,
$k$-dimensional manifold with positive reach and smoothness
$\beta\geq2$. Its density relative to manifold volume is
$\alpha$-smooth, bounded away from zero, and
$0\leq\alpha\leq\beta-1$. Their construction uses time-partitioned
score networks, a loss obtained by integrating over the auxiliary
randomness, early stopping, and output truncation. The rate is nearly
minimax in this setting, with $k$ the intrinsic rather than ambient dimension.
\citet{tang2025conditional} extend the statistical analysis to
conditional distribution estimation and manifold adaptation.
These results concern a different target geometry from our
full-support GCR class.

Our bound $n^{-r/(2r+d)}$ in Theorem~\ref{thm:gcr-rate} instead
concerns expected $W_2$ error, including Euler discretization, for
an exact empirical minimizer of the flow matching loss over a finite
ReLU sieve, using one Gaussian source draw and one time draw per
target observation.
GCR imposes uniform Gaussian-channel derivative and tilted-covariance
control, not merely a tail or density-smoothness condition.
GCR targets have full support, and the associated velocities remain
regular through the terminal time, so the construction has no early
stopping bias.
The density smoothness indices above cannot be identified with
the velocity regularity index $r$ without a separate transfer
argument. Moreover, $W_1$ is weaker than $W_2$, and conversion from
TV to $W_2$ requires tail or moment control and generally changes
the exponent.

\subsection{Statistical guarantees for flow matching}
\citet[Theorem 9]{fukumizu2025flow} report nearly minimax
$W_p$ rates, $1\leq p\leq2$ and $d\geq2$, for compactly supported Besov targets
bounded above and below on a cube, with additional boundary
smoothness and a posterior-mean derivative assumption.
Its bound is $n^{-(s+(2\kappa)^{-1}-\delta)/(2s+d)}$ for arbitrarily
small $\delta>0$, where $\kappa$ controls terminal noise decay.
$\kappa=1/2$ gives the almost-minimax exponent.
The theorem uses time-partitioned networks, an empirical risk obtained
by integrating over the auxiliary randomness, and stopping before the
target endpoint. Its fitted class explicitly imposes a time-dependent
spatial Lipschitz bound.
The statistical guarantee concerns the exact ODE law.
Our learned velocity satisfies a spatial Lipschitz bound that is
independent of the truncation radius and holds throughout $[0,1]$.

\citet{he2026theory} study a sampled empirical loss for a conditional
velocity field whose network input includes the initial state.
Their Algorithm~1 draws the source state conditionally on the target
observation, and Theorems~3.2--3.3 analyze gradient descent and
generalization for overparameterized two-layer ReLU networks.
Theorem~3.5 reports a high-probability $W_1$ bound under bounded target
support, inverse-moment conditions on the target, bounded interpolation
schedules and derivatives, and the stated network and optimization
conditions. Although Algorithm~2 uses Euler integration, the proof in
Appendix~F compares continuous ODE solutions and does not quantify
discretization error. Our estimator is an exact empirical minimizer
over a finite sieve and learns the marginal velocity under independent
Gaussian augmentation.

\citet{kunkel2026lipschitz} studies posterior covariance control,
Lipschitz guarantees, and distribution estimation on unbounded
domains. Theorem 4.3 of its author-posted version
(arXiv:2509.02337v2) gives a
high-probability $W_1$ rate
$\operatorname{polylog}(n)n^{-(1+\alpha)/(d+4\alpha+5+\eta)}$,
with $\alpha=1$ applicable to its bounded $C^2$ perturbations of a
Gaussian potential and any fixed $\eta>0$.
The empirical criterion integrates auxiliary time and Gaussian
randomness conditional on each target sample. The construction uses
an exponential noise schedule ending at positive noise, with a
separate smoothing-bias term. Its network analysis includes
derivative approximation for Lipschitz control. Our linear schedule
reaches zero noise at time one.

\citet{su2025highorder} incorporate derivatives of the interpolation
trajectory into a Transformer regression objective that integrates
over auxiliary randomness. Their higher-order construction concerns
trajectory derivatives, while the order $r$ in GCR controls spatial
regularity of the standard flow matching velocity.
\citet[Assumptions~3.2--3.3 and Theorem~3.1]{su2025kl} relate squared
velocity error to KL divergence using $C^2$ velocity fields and spatially
uniform bounds on the score along the target probability path, both
velocity fields, and the derivatives specified in their assumptions.
Their statistical application concerns TV error for Transformer estimators.

\citet[Theorem~2]{kumar2026manifold} consider compact smooth manifolds
without boundary, with positive reach and intrinsic dimension at least
three, and densities that are smooth and bounded below relative to
manifold volume. With an additional one-sided Lipschitz condition
on the velocity, their early-stopped estimator admits a $W_2$ bound
governed by intrinsic dimension and both density and manifold smoothness.

\subsection{Kernel density estimation and Wasserstein optimality}
\citet{kunkel2026minimax} connect empirical flow matching with
kernel density estimation and prove near-minimax $W_1$ guarantees
under their assumptions. Their empirical criterion integrates auxiliary
randomness conditional on the target sample, whereas our empirical
loss uses one auxiliary draw per target sample. Their proof separates
the statistical error of a kernel density estimator from the error in
approximating the empirical velocity field that generates it. Our rate
follows by balancing approximation and estimation errors for the
population velocity over a finite sieve and then applying flow
stability. Whether sharper bounds on the projected velocity error and
the nonlinear remainder in the terminal error decomposition can improve
our rate remains open.

\subsubsection{Comparison on a common target subclass}
\label{subsec:common-flow-comparison}
Fix $d\geq2$ and $G<\infty$. Consider the centered probability
densities
\begin{equation}\label{eq:common-flow-subclass}
 p=\varphi_1(1+g),\qquad g\in C^2(\sR^d),\qquad
 \|g\|_\infty\leq\tfrac12,\qquad
 \max_{j=1,2}\|D^jg\|_\infty\leq G.
\end{equation}
Proposition \ref{prop:gcr-finite-perturbation} puts this family in
a fixed GCR$(2)$ class. It also has a bounded $C^2$ potential
perturbation $a=-\log(1+g)$, since
$p\propto\exp(-\|x\|_2^2/2-a(x))$. Moreover,
\[
 \nabla p=\varphi_1\{\nabla g-x(1+g)\},\qquad
 \|\nabla p\|_{L^1}\leq C(d,G),\qquad
 \int\|x\|_2^{d+4}p(x)\,dx\leq C(d).
\]
The translation bound
$\|p(\cdot+z)-p\|_{L^1}\leq\|z\|_2\|\nabla p\|_{L^1}$
verifies the $B^1_{1,\infty}$ condition. Thus this family meets
the target conditions of \citet{kunkel2026lipschitz} and
\citet{kunkel2026minimax}, with their respective schedules and
estimator definitions.

For $d\geq3$, the same family satisfies Assumption 1 of
\citet{stephanovitch2025generalization} with $\beta=2$.
Using $a=-\log(1+g)$ as above, write
\[
 p=e^{-u-a},\qquad
 u(x)=\frac{\|x\|_2^2}{2}+\frac d2\log(2\pi).
\]
Their H\"older perturbation is $-a$. Since $1+g\geq1/2$,
\[
 \|a\|_\infty\leq\log2,\qquad
 \|Da\|_\infty\leq2G,\qquad
 \|D^2a\|_\infty\leq2G+4G^2.
\]
These bounds also control the oscillation of each second derivative,
so $-a$ lies in their H\"older class of order two with a constant
depending only on $d,G$. Moreover, $u\in C^\infty$, $D^2u=I_d$,
and its minimizer is zero. For $\|x-y\|_2\leq1$,
\[
 \max_{0\leq j\leq3}\|D^ju(y)\|
 \leq C(d)\{1+u(x)\}.
\]
As $u(x)\geq1$ and the neighborhoods in their Assumption 1.4 have
radius at most one, this proves the required local growth bound for
a sufficiently large fixed $K=K(d,G)\geq1$. All parts of their
Assumption 1 therefore hold uniformly over
\eqref{eq:common-flow-subclass}.

Their Theorem I yields expected $W_1$ error of order $n^{-3/(d+4)}$,
up to logarithmic factors, for the integrated denoising score matching
loss and continuous sampler described above.
Theorem \ref{thm:gcr-rate} yields expected $W_2$ error of order
$n^{-2/(d+4)}$, up to logarithmic factors, for the sampled empirical
flow matching loss, including Euler integration to time one.
The former $W_1$ upper bound is faster than the $W_1$ bound obtained
from our theorem through $W_1\leq W_2$.

Table \ref{tab:common-flow-comparison} of the main paper summarizes
these comparisons, using Theorem 4.3 of \citet{kunkel2026lipschitz}
(arXiv:\allowbreak2509.02337v2), Remark 3.8(3) of
\citet{kunkel2026minimax} (arXiv:\allowbreak2504.13336v2), and Theorem I of
\citet{stephanovitch2025generalization} (arXiv:\allowbreak2507.04794v1).
The extension in Remark 3.8(3) of \citet{kunkel2026minimax} uses
the KDE bound under moment assumptions in their Theorem A.1.
Their Theorem 3.7 uses an exponential-size network construction.
Each network in our sieve has polynomially many
parameters, while \eqref{eq:gcr-sieve-entropy} controls the
log-cardinality of the candidate class. Equation \eqref{eq:aos-time-comparison}
compares the resulting rates with those obtained from a first-order time
approximation.

\subsubsection{Classical rates for distribution estimation}
For classical density estimation, \citet{niles2022minimax} show
that smoothness and density lower bounds play distinct roles in
Wasserstein minimax rates. This distinction is particularly
important on an unbounded domain, where a uniform positive
density lower bound is impossible. Our direct subclass lower
bound is constructed inside the specified target family.
Empirical-measure bounds
\citep{fournier2015rate,weed2019sharp} provide a separate comparison,
as discussed in \eqref{eq:aos-empirical-benchmark}.

The sieve construction uses local polynomial approximation and ReLU product
approximations with control of the Lipschitz error
\citep{yarotsky2017error,guhring2020error}, together with bounded filtered
polynomial operators \citep{occorsio2021filtered}.

\section{Conditional noise for the sampled loss}
\label{sec:supp-conditional-noise}
For the independent training pairs $(W_i,U_i)_{i=1}^n$ defined in
Section \ref{sec:prelim}, set
$\zeta_i=U_i-v^*(W_i)$. For a function $a$ of an observed pair, write
$\mathbb P_n a=n^{-1}\sum_{i=1}^n a(W_i,U_i)$ for its empirical average.

\begin{lemma}[Conditional noise and the multiplier term]
\label{lm:gcr-conditional-noise}
Let $\zeta=U-v^*(T,X_T)$. Under GCR, for every $a\in\sR^d$ and
$0<t<1$, a version of the conditional law satisfies
\begin{equation}\label{eq:gcr-conditional-noise}
 \log\E[e^{a^\top\zeta}\mid T=t,X_t=x]
 \leq\frac{C\|a\|_2^2}{2\{t^2+(1-t)^2\}}.
\end{equation}
Thus the residual has a uniform conditional sub-Gaussian proxy
$\sigma_\zeta^2$. For any deterministic finite nonempty class
$\mathcal F$ of measurable vector fields,
$0<\alpha<1$, and $\rho>0$, with probability at least $1-\alpha$,
simultaneously for $f,g\in\mathcal F$,
\begin{align}
 2|\mathbb P_n\langle\zeta,f-g\rangle|
 &\leq\rho\mathbb P_n\|f-g\|_2^2
 +\frac{2\sigma_\zeta^2\log(2|\mathcal F|^2/\alpha)}{\rho n}.
 \label{eq:gcr-multiplier}
\end{align}
\end{lemma}
\begin{proof}
Theorem \ref{thm:gcr-transfer} bounds the posterior covariance by
$C(1+\lambda)^{-1}I_d$ at every real tilt. Integrating this Hessian
bound along a tilt segment bounds the centered posterior log moment
generating function by $C\|a\|_2^2/(2(1+\lambda))$.
Conditionally on $X_t=x$,
$\zeta=(Y-\E[Y\mid X_t=x])/(1-t)$.
Substitution of $\lambda=t^2/(1-t)^2$ proves
\eqref{eq:gcr-conditional-noise}. At time zero the residual is
$Y-\mu_p$, and at time one it is $-Z$. The same bound holds at both endpoints.

Condition on all designs $W_1,\ldots,W_n$. The residuals are
conditionally independent and centered. For each pair $f,g$,
$\mathbb P_n\langle\zeta,f-g\rangle$ is conditionally sub-Gaussian
with proxy $\sigma_\zeta^2\mathbb P_n\|f-g\|_2^2/n$.
A two-sided tail bound and a union bound over pairs give magnitude
at most $\{2\sigma_\zeta^2\mathbb P_n\|f-g\|_2^2
\log(2|\mathcal F|^2/\alpha)/n\}^{1/2}$.
Young's inequality proves \eqref{eq:gcr-multiplier}.
\end{proof}

The multiplier bound needs no response truncation and its coefficient
does not grow with the spatial radius $A$. It does not alone remove
the logarithmic factors in the final rate. For example, if
$\|f\|_\infty,\|g\|_\infty\leq B$, the random-design quantity
$Q=\|f(W)-g(W)\|_2^2$ satisfies
$\operatorname{Var}(Q)\leq4dB^2\E Q$.
The resulting Bernstein comparison still has a $B^2$ complexity
term, with $B\asymp A$ here. Approximation and spatial tails need
also be controlled. The stated rate follows from the oracle inequality
based on unconditional tail bounds in Lemma \ref{lm:stoc-error-fm}, while
\eqref{eq:gcr-multiplier} provides a starting point for sharper
localized analyses.

\end{appendix}
\bigskip

\bibliographystyle{imsart-nameyear}
\bibliography{ref_GCR}
\end{document}